\documentclass[twoside]{article}
\usepackage[accepted]{aistats2023_preprint}

\usepackage[utf8]{inputenc} 
\usepackage[T1]{fontenc}    
\usepackage{url}            
\usepackage{booktabs}       
\usepackage{amsfonts}       
\usepackage{nicefrac}       
\usepackage{microtype}      
\usepackage{algolyx}        
\usepackage[english]{babel}
\usepackage{graphicx}
\usepackage{wrapfig}

\usepackage[round]{natbib}

\usepackage[xindy,acronym,toc,smallcaps,nomain]{glossaries}
\usepackage[normalem]{ulem}
\usepackage[colorlinks=true,allcolors=blue,backref=page]{hyperref}

\usepackage{color}
\usepackage{array}
\usepackage{multirow}
\usepackage{amsmath}
\usepackage{amsthm}
\usepackage{amssymb}

\makeatletter

\newcommand{\stkout}[1]{\ifmmode\text{\sout{\ensuremath{#1}}}\else\sout{#1}\fi}

\newacronym{BO}{bo}{Bayesian optimization}
\newacronym{GP}{gp}{Gaussian process}
\newacronym{ERM}{erm}{expected regret minimization}
\newacronym{CBM}{cbm}{confidence bound minimization}
\newacronym{MAB}{mab}{Multi-Armed Bandit}
\newacronym{BES}{bes}{bounded entropy search}
\newacronym{RF}{rf}{Random Forest}
\newacronym{RMSE}{rmse}{root mean square error}
\newacronym{SRGP}{srgp}{square-root transformed Gaussian process}
\newacronym{RS}{rs}{rejection-sampling}
\newacronym{EI}{ei}{expected improvement}
\newacronym{TS}{ts}{Thompson sampling}

\providecommand{\tabularnewline}{\\}
\floatstyle{ruled}
\newfloat{algorithm}{tbp}{loa}
\providecommand{\algorithmname}{Algorithm}
\floatname{algorithm}{\protect\algorithmname}

\usepackage{algolyx}
\theoremstyle{plain}
\newtheorem{lem}{\protect\lemmaname}

\usepackage{mathptmx}
\usepackage{microtype}
\usepackage{flushend}
\usepackage{pdfoverlay}
\usepackage{tcolorbox}

\usepackage{nicefrac}       

\usepackage{subfig}
\date{}

\@ifundefined{showcaptionsetup}{}{%
 \PassOptionsToPackage{caption=false}{subfig}}
\usepackage{subfig}
\makeatother

\providecommand{\lemmaname}{Lemma}

\usepackage{amsmath}               
  {
      \theoremstyle{plain}
      \newtheorem{definition}{Definition}
  }

\begin{document}

\twocolumn[

\aistatstitle{Gaussian Process Sampling and Optimization with Approximate Upper and Lower Bounds}

\aistatsauthor{ Vu Nguyen \And Marc Peter Deisenroth \And  Michael A. Osborne }

\aistatsaddress{ Amazon \And  University College London \And University of Oxford} ]












\global\long\def\se{\hat{\text{se}}}%

\global\long\def\interior{\text{int}}%

\global\long\def\boundary{\text{bd}}%

\global\long\def\new{\text{*}}%

\global\long\def\stir{\text{Stirl}}%

\global\long\def\dist{d}%

\global\long\def\HX{\entro\left(X\right)}%
 
\global\long\def\entropyX{\HX}%

\global\long\def\HY{\entro\left(Y\right)}%
 
\global\long\def\entropyY{\HY}%

\global\long\def\HXY{\entro\left(X,Y\right)}%
 
\global\long\def\entropyXY{\HXY}%

\global\long\def\mutualXY{\mutual\left(X;Y\right)}%
 
\global\long\def\mutinfoXY{\mutualXY}%

\global\long\def\xnew{y}%

\global\long\def\bm{\mathbf{m}}%

\global\long\def\bx{\mathbf{x}}%

\global\long\def\bw{\mathbf{w}}%

\global\long\def\bz{\mathbf{z}}%

\global\long\def\bu{\mathbf{u}}%

\global\long\def\bs{\boldsymbol{s}}%

\global\long\def\bk{\mathbf{k}}%

\global\long\def\bX{\mathbf{X}}%

\global\long\def\tbx{\tilde{\bx}}%

\global\long\def\by{\mathbf{y}}%

\global\long\def\bY{\mathbf{Y}}%

\global\long\def\bZ{\boldsymbol{Z}}%

\global\long\def\bU{\boldsymbol{U}}%

\global\long\def\bv{\boldsymbol{v}}%

\global\long\def\bn{\boldsymbol{n}}%

\global\long\def\bh{\boldsymbol{h}}%

\global\long\def\bV{\boldsymbol{V}}%

\global\long\def\bK{\boldsymbol{K}}%

\global\long\def\bbeta{\gvt{\beta}}%

\global\long\def\bmu{\gvt{\mu}}%

\global\long\def\btheta{\boldsymbol{\theta}}%

\global\long\def\blambda{\boldsymbol{\lambda}}%

\global\long\def\bgamma{\boldsymbol{\gamma}}%

\global\long\def\bpsi{\boldsymbol{\psi}}%

\global\long\def\bphi{\boldsymbol{\phi}}%

\global\long\def\bpi{\boldsymbol{\pi}}%

\global\long\def\eeta{\boldsymbol{\eta}}%

\global\long\def\bomega{\boldsymbol{\omega}}%

\global\long\def\bepsilon{\boldsymbol{\epsilon}}%

\global\long\def\btau{\boldsymbol{\tau}}%

\global\long\def\bSigma{\gvt{\Sigma}}%

\global\long\def\realset{\mathbb{R}}%

\global\long\def\realn{\realset^{n}}%

\global\long\def\integerset{\mathbb{Z}}%

\global\long\def\natset{\integerset}%

\global\long\def\integer{\integerset}%

\global\long\def\natn{\natset^{n}}%

\global\long\def\rational{\mathbb{Q}}%

\global\long\def\rationaln{\rational^{n}}%

\global\long\def\complexset{\mathbb{C}}%

\global\long\def\comp{\complexset}%

\global\long\def\compl#1{#1^{\text{c}}}%

\global\long\def\and{\cap}%

\global\long\def\compn{\comp^{n}}%

\global\long\def\comb#1#2{\left({#1\atop #2}\right) }%

\global\long\def\nchoosek#1#2{\left({#1\atop #2}\right)}%

\global\long\def\param{\vt w}%

\global\long\def\Param{\Theta}%

\global\long\def\meanparam{\gvt{\mu}}%

\global\long\def\Meanparam{\mathcal{M}}%

\global\long\def\meanmap{\mathbf{m}}%

\global\long\def\logpart{A}%

\global\long\def\simplex{\Delta}%

\global\long\def\simplexn{\simplex^{n}}%

\global\long\def\dirproc{\text{DP}}%

\global\long\def\ggproc{\text{GG}}%

\global\long\def\DP{\text{DP}}%

\global\long\def\ndp{\text{nDP}}%

\global\long\def\hdp{\text{HDP}}%

\global\long\def\gempdf{\text{GEM}}%

\global\long\def\ei{\text{EI}}%

\global\long\def\rfs{\text{RFS}}%

\global\long\def\bernrfs{\text{BernoulliRFS}}%

\global\long\def\poissrfs{\text{PoissonRFS}}%

\global\long\def\grad{\gradient}%
 
\global\long\def\gradient{\nabla}%

\global\long\def\cpr#1#2{\Pr\left(#1\ |\ #2\right)}%

\global\long\def\var{\text{Var}}%

\global\long\def\Var#1{\text{Var}\left[#1\right]}%

\global\long\def\cov{\text{Cov}}%

\global\long\def\Cov#1{\cov\left[ #1 \right]}%

\global\long\def\COV#1#2{\underset{#2}{\cov}\left[ #1 \right]}%

\global\long\def\corr{\text{Corr}}%

\global\long\def\sst{\text{T}}%

\global\long\def\SST{\sst}%

\global\long\def\ess{\mathbb{E}}%

\global\long\def\Ess#1{\ess\left[#1\right]}%

\global\long\def\fisher{\mathcal{F}}%

\global\long\def\bfield{\mathcal{B}}%
 
\global\long\def\borel{\mathcal{B}}%

\global\long\def\bernpdf{\text{Bernoulli}}%

\global\long\def\betapdf{\text{Beta}}%

\global\long\def\dirpdf{\text{Dir}}%

\global\long\def\gammapdf{\text{Gamma}}%

\global\long\def\gaussden#1#2{\text{Normal}\left(#1, #2 \right) }%

\global\long\def\gauss{\mathbf{N}}%

\global\long\def\gausspdf#1#2#3{\text{Normal}\left( #1 \lcabra{#2, #3}\right) }%

\global\long\def\multpdf{\text{Mult}}%

\global\long\def\poiss{\text{Pois}}%

\global\long\def\poissonpdf{\text{Poisson}}%

\global\long\def\pgpdf{\text{PG}}%

\global\long\def\wshpdf{\text{Wish}}%

\global\long\def\iwshpdf{\text{InvWish}}%

\global\long\def\nwpdf{\text{NW}}%

\global\long\def\niwpdf{\text{NIW}}%

\global\long\def\studentpdf{\text{Student}}%

\global\long\def\unipdf{\text{Uni}}%

\global\long\def\transp#1{\transpose{#1}}%
 
\global\long\def\transpose#1{#1^{\mathsf{T}}}%

\global\long\def\mgt{\succ}%

\global\long\def\mge{\succeq}%

\global\long\def\idenmat{\mathbf{I}}%

\global\long\def\trace{\mathrm{tr}}%

\global\long\def\argmax#1{\underset{_{#1}}{\text{argmax}} }%

\global\long\def\argmin#1{\underset{_{#1}}{\text{argmin}\ } }%

\global\long\def\diag{\text{diag}}%

\global\long\def\norm{}%

\global\long\def\spn{\text{span}}%

\global\long\def\vtspace{\mathcal{V}}%

\global\long\def\field{\mathcal{F}}%
 
\global\long\def\ffield{\mathcal{F}}%

\global\long\def\inner#1#2{\left\langle #1,#2\right\rangle }%
 
\global\long\def\iprod#1#2{\inner{#1}{#2}}%

\global\long\def\dprod#1#2{#1 \cdot#2}%

\global\long\def\norm#1{\left\Vert #1\right\Vert }%

\global\long\def\entro{\mathbb{H}}%

\global\long\def\entropy{\mathbb{H}}%

\global\long\def\Entro#1{\entro\left[#1\right]}%

\global\long\def\Entropy#1{\Entro{#1}}%

\global\long\def\mutinfo{\mathbb{I}}%

\global\long\def\relH{\mathit{D}}%

\global\long\def\reldiv#1#2{\relH\left(#1||#2\right)}%

\global\long\def\KL{KL}%

\global\long\def\KLdiv#1#2{\KL\left(#1\parallel#2\right)}%
 
\global\long\def\KLdivergence#1#2{\KL\left(#1\ \parallel\ #2\right)}%

\global\long\def\crossH{\mathcal{C}}%
 
\global\long\def\crossentropy{\mathcal{C}}%

\global\long\def\crossHxy#1#2{\crossentropy\left(#1\parallel#2\right)}%

\global\long\def\breg{\text{BD}}%

\global\long\def\lcabra#1{\left|#1\right.}%

\global\long\def\lbra#1{\lcabra{#1}}%

\global\long\def\rcabra#1{\left.#1\right|}%

\global\long\def\rbra#1{\rcabra{#1}}%


\vspace{-15pt}
\begin{abstract}

Many functions have approximately-known upper and/or lower bounds, potentially aiding the modeling of such functions.
In this paper, we introduce Gaussian process models for functions where such bounds are (approximately) known.
More specifically,  
we propose the \emph{first} use of such bounds to improve \gls{GP} posterior sampling and \gls{BO}.
That is, we transform a \gls{GP} model satisfying the given bounds,  and then sample and weight functions from its posterior. 
To further exploit these bounds in \gls{BO} settings, we present \gls{BES} to select the point gaining the most information about the underlying function, estimated by the \gls{GP} samples, while satisfying the output constraints. 
We characterize the sample variance bounds and show that the decision made by \gls{BES} is explainable. 
Our proposed approach is conceptually straightforward and can be used as a plugin extension to existing methods for \gls{GP} posterior sampling and Bayesian optimization.

\end{abstract}

\vspace{-15pt}

\section{Introduction}



Gaussian processes (\gls{GP}s) provide a powerful probabilistic learning framework that has led to great success in many machine learning settings, such as black-box optimization \citep{Brochu_2010Tutorial}, reinforcement learning \citep{deisenroth2011pilco}, hyperparameter tuning \citep{perrone2019learning,PB2}, and battery forecasting \citep{richardson2017gaussian}. \gls{GP}s have been especially impactful  for Bayesian optimization (\gls{BO}) \citep{Shahriari_2016Taking} to optimize a black-box function $f(\cdot)$, where careful uncertainty representation is crucial. A fundamental challenge for Gaussian process modeling and optimization is that data are often limited. In \gls{BO}, a datum is usually an expensive evaluation of another model. Particularly, we evaluate and optimize the underlying black-box  without knowing the gradient or the analytical form of $f(\cdot)$. In the small data regime, it is always useful to utilize external knowledge about $f(\cdot)$ as has been successfully demonstrated in recent work, such as monotonicity trends \citep{riihimaki2010gaussian}, experimenter's hunches \citep{li2018accelerating_ICDM}, and known optimum values \citep{nguyen2020knowing}.



We propose to use another form of prior knowledge for improving both the \gls{GP} posterior sampling and Bayesian optimization. We consider the \textit{maximum and minimum values} of $f(\cdot)$. These optimal values can be found in machine learning applications, including inverse problems
\citep{mosegaard1995monte,aster2018parameter,bertero2020introduction},
where the goal is to retrieve an input $\bx^{*}$ resulting in the given target from a black-box function. Another example is tuning hyperparameters for classification algorithms \citep{krizhevsky2012imagenet}; users may have vague knowledge about $f^{+}:=\max f(\cdot)$ and $f^-:=\min f(\cdot)$, either from recent state-of-the-art results or directly from the definition of the accuracy metric $f^{+}=100\%$, $f^{-}=0\%$. Similar information can be found in tuning regression problems where the best \gls{RMSE} score is known to be $0$. Other examples are discussed by \cite{nguyen2020knowing}, such as known maximum rewards for some reinforcement learning environments. Despite  the usefulness of such prior knowledge for the low-data regime, the setting we consider in this paper, with \textit{approximately known} maximum and minimum values, is new to the best of our knowledge. 




In this paper, we propose to exploit the knowledge of maximum and minimum values of the black-box for improving the performance of \gls{GP} posterior sampling and making better decisions in \gls{BO}, especially when the number of observations is limited. Rather than using a regular \gls{GP} as a surrogate model to infer the underlying function $f(\cdot)$, we transform the surrogate given the  bound information. We assign different credits to samples based on the bounds and thus discard those falling outside. Moreover, we utilize these bounds for \gls{BO} by proposing bounded entropy search (\gls{BES}) to select the next point as that yielding the most information about the true underlying function. Our key contributions are: 
\begin{itemize}
\item the first use of exploiting  knowledge about \textit{both} the maximum $f^+$ and/or minimum $f^-$ values of a black-box function, which can be \textit{approximately} defined,
\item an efficient posterior sampling procedure for $f(\cdot)$ that explicitly accounts for $f^+$ and $f^-$,
\item a bounded entropy search (\gls{BES}) acquisition function for \gls{BO} given $f^+$ and $f^-$.
\end{itemize}



\section{Preliminaries}


A black-box $f(\cdot)$ is defined over some bounded domain $\mathcal{X}\subset\mathbb{R}^{d}$ where $d$ is the input dimension. The function $f$ can only be accessed through noisy queries of the form $y_{i}\sim\mathcal{N}\bigl(f(\bx_{i}),\sigma_f^{2}\bigr)$
where the input $\bx_i \in \mathbb{R}^d$, the output $y_i \in \mathbb{R}$, and $\sigma_f^{2}$ is the output noise variance. 
We denote $\bX=[\bx_{1},...,\bx_{N}]^{T}\in\mathbb{R}^{N\times d}$, $\by=[y_{1},...,y_{N}]^{T}\in\mathbb{R}^{N \times 1}$, and the training set at iteration $t$ by $D_t=\{\bx_{i}, y_{i} \}_{i=1}^t$.



\subsection{Gaussian processes}
\label{subsec:Sampling-from-GP}

We assume $f$ to be drawn from a \gls{GP} \citep{Rasmussen_2006gaussian} which is a random function $f: \mathcal{X} \rightarrow \mathbb{R}$, such
that every finite collection of those random variables $f(\bX_*) \mid \bX_* \in \mathcal{X}$ follows a multivariate Gaussian distribution. Formally,
we denote a \gls{GP} as $f(\bx)\sim \gls{GP}\bigl(m\left(\bx\right),k\left(\bx,\bx'\right)\bigr)$, where $m(\bx)=\mathbb{E}\left[f\left(\bx\right)\right]$ and $k(\bx,\bx')=\mathbb{E}\left[ \bigl(f\left(\bx\right)-m\left(\bx\right) \bigr)^{T} \bigl(f\left(\bx'\right)-m\left(\bx'\right) \bigr)\right]$
are the mean and covariance functions.
Given the observed input $\bX$ and output $\by$, the \gls{GP} posterior at a new input $\bx_*$ is defined as $ f_{*}\mid\bX,\by,\bx_{*} \sim\mathcal{N}\left(\mu\left(\bx_{*}\right),\sigma^{2}\left(\bx_{*}\right)\right)$.
The posterior predictive mean and variance are:

\begin{minipage}[t]{.48\textwidth}
\vspace{-15pt}
\begin{align}
\mu\left(\bx_{*}\right)=  \; & \mathbf{k}^T_{*}\left[\mathbf{K}+\sigma^2_{f}\idenmat\right]^{-1} \left( \mathbf{y} - \bm \right) + m(\bx_*)
\end{align}
\end{minipage}
\begin{minipage}[t]{.47\textwidth}
\vspace{-15pt}
\begin{align}
\sigma^{2}\left(\bx_{*}\right)= \; &  k_{**}-\mathbf{k}^T_{*}\left[\mathbf{K}+\sigma^2_{f}\idenmat\right]^{-1}\mathbf{k}_{*}
\end{align}
\end{minipage}
where $k_{**}=k\left(\bx_{*},\bx_{*}\right)$, $\bk_{*}=[k\left(\bx_{*},\bx_{i}\right)]^T_{\forall i} \in \mathbb{R}^{N \times 1}$, $\bm=[m\left(\bx_{i}\right)]^T_{\forall i} \in \mathbb{R}^{N \times 1}$ is the prior mean at the observed locations
and $\mathbf{K}=\left[k\left(\bx_{i},\bx_{j}\right)\right]_{\forall i,j} \in \mathbb{R}^{N \times N}$.


Samples $g(\cdot) \sim \gls{GP} \bigl(\mu(\cdot),\sigma^2(\cdot)\bigr)$ from a \gls{GP} posterior can be used for several applications, e.g., \gls{TS} \citep{Thompson_1933Likelihood,russo2018tutorial}, or an information-theoretic decision in \gls{BO} \citep{Hernandez_2014Predictive,Hernandez_2017Parallel,Kandasamy_2018Parallelised}. We follow the decoupled approach \citep{wilson2020efficiently,wilson2021pathwise} to
draw samples $g(\cdot)$ from a \gls{GP} posterior. The central idea of decoupled sampling relies on Matheron\textquoteright s rule for  Gaussian random variables \citep{journel1978mining,chiles2009geostatistics,doucet2010note}, using variational Fourier features for an approximate prior and a deterministic data-dependent update term. The key benefit of this sampling is that the complexity scales linearly in the number of test points. This  is particularly useful in optimizing the acquisition function in \gls{BO} where we need to evaluate at many test points to select a next query. 

\subsection{Optimum values prior in Bayesian optimization}


\gls{BO} aims at maximizing an expensive black-box function $f(\cdot)$ using few function evaluations, i.e. finding $\bx^* := \arg \max_{\bx \in \mathcal{X} } f(\bx)$. In \gls{BO}, we typically construct a surrogate model of $f$ using a \gls{GP}. The surrogate model mimics the behavior of $f$ while it is cheaper to evaluate. Then, we sequentially select points $\bx_t$ at which to evaluate $f$.

Prior knowledge about the optimum value of the black-box contains useful information dictating the upper bound of $f(\cdot)$.
A recent approach has considered the \gls{BO} setting where the true optimum value $f^{+}=\max_{\bx\in\mathcal{X}}f(\bx)$
of the black-box function is available \citep{nguyen2020knowing}. In particular, they transform the \gls{GP} surrogate to satisfy that $f^+$ is an upper bound, and propose \gls{ERM} and \gls{CBM} acquisition functions.  However, this approach can be ineffective when the value of the true optimum $\max f(\cdot)$ is not known. If the optimum value is mis-specified, the performance degrades. In addition, this approach can only cope with knowledge of $f^+$, but it is unable to account for  knowledge of both $f^+$ and $f^-$ at the same time. The method we propose in this paper will overcome these limitations.


\vspace{-0pt}

\section{Gaussian process posterior sampling with approximate bounds $f^+$ and $f^-$}
\label{sec:GP-Posterior-Sampling_with_fmax_fmin}





We consider settings where we know  approximately the maximum and/or minimum value of $f$. We express the prior knowledge about the upper bound by specifying the value of $f^+$ and how loosely such bound via $\eta_+$. Similar design is for the lower bound via $f^-$ and $\eta_-$.

\begin{definition} \label{def:bound_random_variable}
Define the approximate bounds as random variables: $f^+ \sim \mathcal{N}(\max f ,\eta^2_+ )$ and $f^-  \sim \mathcal{N}(\min f,\eta^2_-)$. 
\end{definition}
 We may observe both $f^{+}$ and $f^{-}$ or either of them. The useful information  $f^+$ and $f^-$ give us is that (i) $f(\cdot)$ should not exceed these thresholds and (ii) $f(\cdot)$ should attain closely $f^+,f^-$. In particular, for maximization problems, the location at which  $f^+$ is attained, is \textit{potentially} the global maximizer of $f$ we are looking for in \gls{BO} settings.  

\begin{figure*}
\vspace{0pt}
\centering
\includegraphics[width=0.52\textwidth]{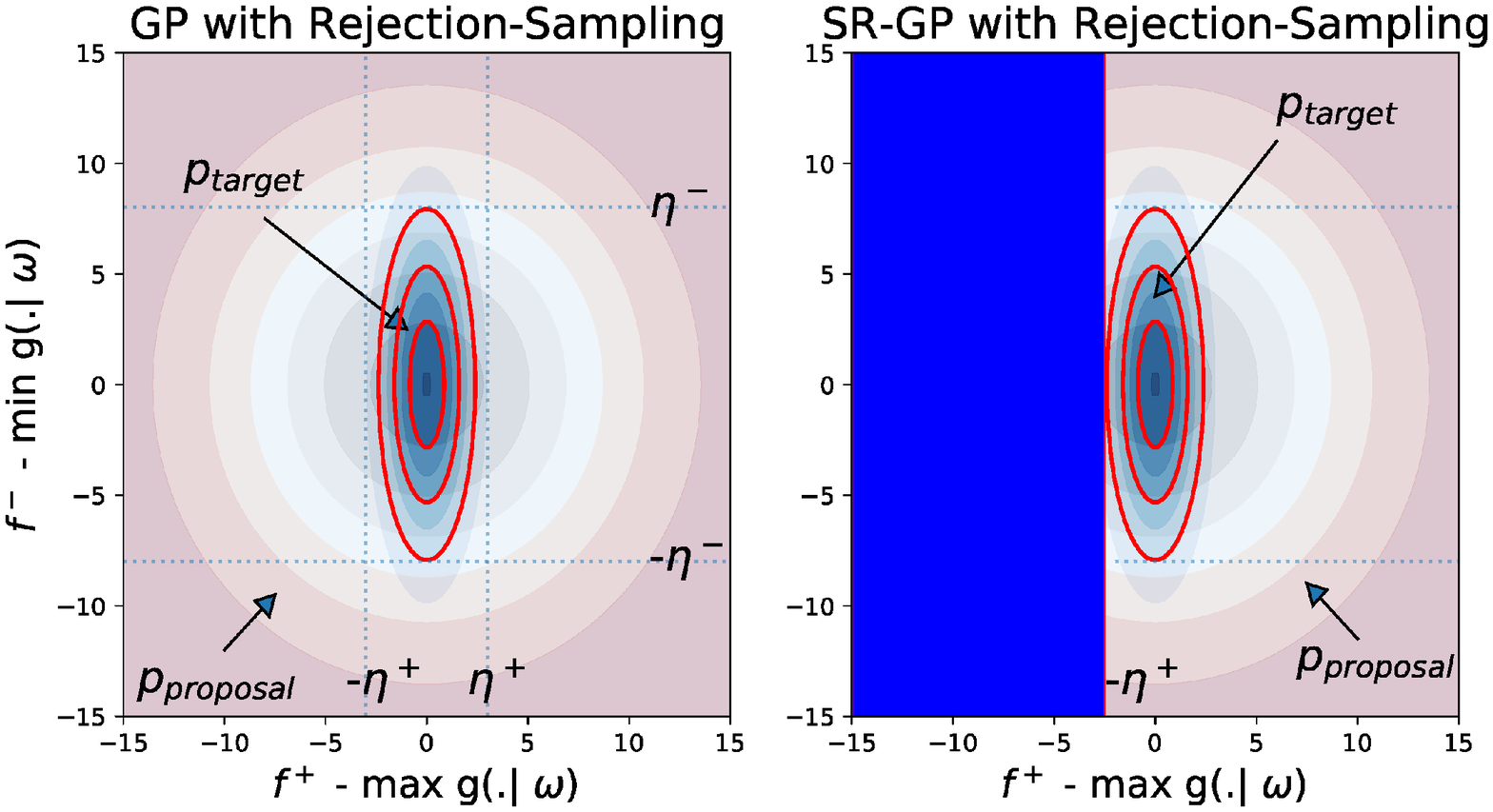} \hspace{1pt} \includegraphics[width=0.46\textwidth]{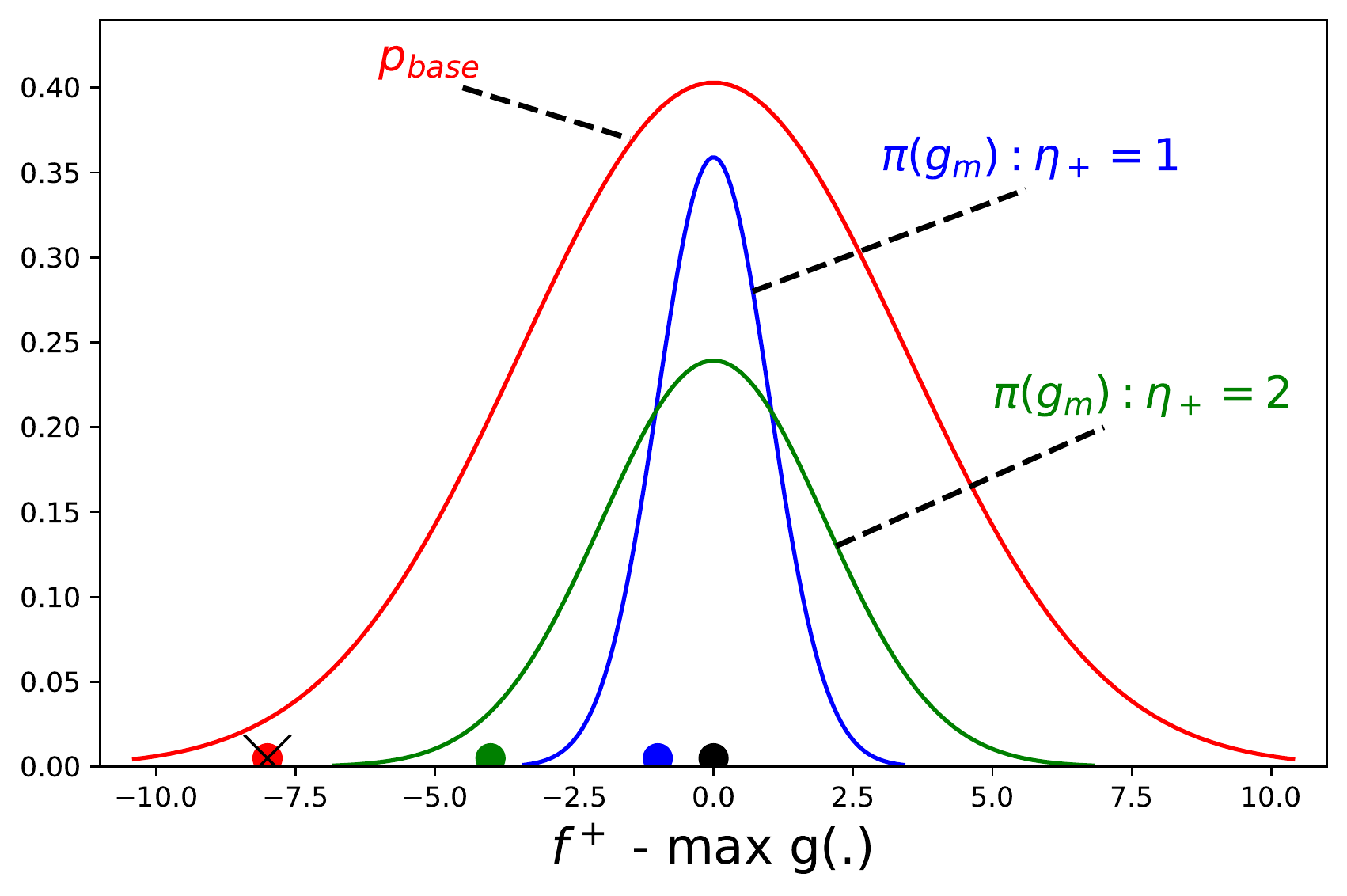}
\vspace{-3pt}
\caption{ Illustration of the sampling and weighting step. \textit{Left}: We can use  a \gls{GP} or a \gls{SRGP} for $p_{\textrm{base}}$. Given the same weighting probability $\pi(g_m)$, using \gls{SRGP} has a smaller sample variance by ignoring the redundant space (\textcolor{orange}{orange} area). 
\textit{Right}: For ease of representation, we use a univariate Gaussian distribution to visualize $\pi(g_m)$. The \gls{GP} sample $g_m(\cdot)$ (represented by the maximum value of $g_m(\cdot)$) is drawn from  $p_{\textrm{base}}$. We assign high weight to a sample depending on the density of $\pi(g_m)$. The \textcolor{red}{red} point will be rejected. The \textcolor{green}{green} point falls within the density of $\pi(g_m)$ with $\eta_+=2$, but not for $\eta_+=1$. The \textcolor{blue}{blue} point receives positive weights by both $\pi(g_m)$ including  $\eta_+=1$ and $\eta_+=2$. The \textbf{black} point receives the highest weight. 
}
\label{fig:Illustration_2d_prob_space}
\vspace{-2pt}
\end{figure*}

\subsection{Sampling and weighting \gls{GP} posterior samples}
\label{subsec:Sampling-from-GP_rejection_sampling}

Although the existing \gls{GP} posterior sampling approach \citep{wilson2020efficiently} has demonstrated great success, the resulting \gls{GP} samples usually do not obey the bounds $[f^-, f^+]$. Intuitively, assume we are given a few (noisy) observations from the black-box $f$. We can be confident in deciding which \gls{GP} samples are not admissible by looking at the bounds $f^{+},f^{-}$ and thus not making use of these samples. Therefore, we propose to assign different credits to the \gls{GP} samples based on the bound knowledge. 

We denote $p_{\textrm{base}}$ the base distribution to generate \gls{GP} posterior samples, e.g., using decoupling approach \citep{wilson2020efficiently}. Given the constraints $f^+,f^-$, we can perform weighted-averaging  across all \gls{GP} samples as:
\vspace{-7pt}
\small
\begin{align} \label{eq:weighted_averaging}
\hat{g}& := 
\sum_{g_m \sim p_{base}} g_m \frac{ p ( f^+, f^- \mid g_m ) }{ \sum p ( f^+, f^- \mid g_m ) } = \sum_{m=1}^M g_m \frac{ \pi( g_m) }{ \sum_{m=1}^M \pi( g_m) } 
\end{align}
\normalsize
where $g_m$ is a \gls{GP} posterior sample, see Eq. (\ref{eq:GP_sample_decoupled})  in the appendix for details. We denote $\pi( g_m):=p \Bigl( f^+, f^- \mid g_m \Bigr)$ which assigns high probability to a sample $g_m$ satisfying the approximate bounds while it assigns zero or very low probability otherwise. The fraction of $\frac{ \pi( g_m) }{ \sum_{m=1}^M \pi( g_m) }$ indicates the normalized weight of a sample $m$. We define $\pi(g_m)$ in the next section.

\subsection{The weighting probability $\pi( g_m)$ given $f^{+}$ and $f^{-}$ }\label{subsec:defining_p_target}
Given the data $\bX,\by$, we draw $M$ samples from a \gls{GP} posterior and expect that a \gls{GP} sample $g_m(\cdot)$ naturally follows the
conditions: (i) $f^{-} \lesssim
 g_m(\cdot) \lesssim
 f^{+}$ and
(ii) $\exists \bx_m^+$ such that  $g_m(\bx_m^+) \approx f^{+}$ and $\exists \bx_m^-$ such that $ g_m(\bx_m^-) \approx f^{-}$. Let us denote the maximum and minimum values of the sample by $g^+_m=\max g_{m}(\cdot)$ and $g^-_m=\min g_{m}(\cdot)$ which can be computed efficiently using the off-the-shelf toolboxes, such as multi-start gradient descent \cite{gyorgy2011efficient} because we have the analytical form available for $g_m(\cdot)$, presented in Appendix \ref{appendix_subsec_decouple_gp_posterior}. To satisfy the above conditions, we define the weighting probability $\pi(g_m)$ using an isotropic bivariate Gaussian distribution:
\begin{align} \label{eq:likelihood_gp_sample}
\pi(g_m) & :=\mathcal{N} \Big( \left[g^-_m, g^+_m \right] \mid \left[ f^{-}, f^+ \right],\text{diag}[\eta^2_-,\eta^2_+] \Big).
\end{align}
The mean of $\pi(g_m)$ is located at $[f^{-},f^{+}]$ and the $M$ observations $\left[g^-_m, g^+_m \right]$ are the min and max values of the \gls{GP} posterior samples,  see Fig. \ref{fig:Illustration_2d_prob_space} (left). This bivariate will flexibly collapse into a univariate Gaussian when either $f^{+}$ or $f^{-}$ is observed.  

\begin{table}
\vspace{-0pt} 
\caption{Acceptance ratio over $M=200$ samples from \gls{GP} and \gls{SRGP} respectively. \gls{SRGP} can generate better samples by incorporating $f^{+}$ and $\textrm{Var}[p_{\textrm{base-\gls{SRGP}}}]\le\textrm{Var}[p_{\textrm{base-\gls{GP}}}]$. See Section \ref{sec:experiments} for detail. \label{tab:Acceptance-ratio}}
\vspace{-4pt}

\begin{tabular}{lllll}
\toprule
\multirow{2}{*}{} & \multicolumn{2}{c|}{\small{}$\eta_+=\eta_-=0.5d $} & \multicolumn{2}{c}{\small{}$\eta_+=\eta_-=1d$} \tabularnewline
\cmidrule{2-5} 
 & \multicolumn{1}{c|}{{\small{}\gls{GP}}} & {\small{}\textsc{srgp}} & \multicolumn{1}{c|}{{\small{}\gls{GP}}} & {\small{}\textsc{srgp}}\tabularnewline
\midrule
{\small{}branin d=2} & {\small{}$.39(.2)$} & {\small{}$\mathbf{.89(.2)}$} & {\small{}$.62(.1)$} & {\small{}$\mathbf{.96(.2)}$}\tabularnewline
{\small{}rosenbrock d=2} & {\small{}$.52(.2)$} & \textbf{\small{}$\mathbf{.88(.2)}$} & {\small{}$.62(.2)$} & {\small{}$\mathbf{.95(.2)}$}\tabularnewline
{\small{}mccormick d=2} & {\small{}$.65(.1)$} & {\small{}$\mathbf{.91(.2)}$} & {\small{}$.69(.1)$} & {\small{}$\mathbf{1.0(.1)}$}\tabularnewline
{\small{}hartmann d=3} & {\small{}$.08(.1)$} & {\small{}$\mathbf{.40(.3)}$} & {\small{}$.31(.2)$} & {\small{}$\mathbf{.76(.2)}$}\tabularnewline
{\small{}alpine1 d=5} & {\small{}$.0(.0)$} & {\small{}$\mathbf{.13(.1)}$} & {\small{}$.10(.1)$} & {\small{}$\mathbf{.30(.2)}$}\tabularnewline
{\small{}gSobol d=5} & {\small{}$.01(.1)$} & {\small{}$\mathbf{.46(.3)}$} & {\small{}$.14(.2)$} & \textbf{\small{}$\mathbf{.74(.4)}$}\tabularnewline
\bottomrule
\end{tabular}
\vspace{-5pt}
\end{table}

\paragraph{The view of rejection sampling.}
The weighting scheme above can be seen from the rejection-sampling perspective.
We accept a \gls{GP} sample $g_{m}(\cdot)$ if it lies within a two standard deviation band around the mean of $\pi(g_m)$ as illustrated in Fig. \ref{fig:Illustration_2d_prob_space} (right). Specifying a smaller value for $\eta_+$ will make $\pi(g_m)$ strict and only accept a \gls{GP} sample which has the maximum value $\max g_m(\cdot)$ closer to $f^+$. On the other hand, increasing $\eta_+,\eta_-$ will lead to more accepted samples, but in lower quality.  
We visualize $\pi(g_m)$ using different values of $\eta=\{1,2\}$ in Fig. \ref{fig:Illustration_2d_prob_space} (right) where we have simplified the \textit{2d} space into one dimension for better representation.



\subsection{Square-root \gls{GP} for sampling $g_m$ \label{subsec:Square-root-Transformation-and} }

\begin{minipage}{0.5\textwidth}
\vspace{-15pt}
\begin{algorithm}[H]
	    \caption{\gls{SRGP}-weighting with $f^{+},f^{-}$\label{alg:TS_kov}}
	    
	\begin{algorithmic}[1]
		\STATE {\bfseries Input:}  \#GP sample $M$, data $D$, $f^{+}$, $f^{-}$
		
        	\STATE Learn a \gls{SRGP} posterior model $\mathcal{G}$ from ${D}$ and $f^{+}$
        	
        	\STATE Draw \gls{GP} samples $g_{1},....g_{M}$  from $p_{\textrm{base}}$ by (i) drawing from $\mu_h(\cdot)$ and (ii) computing $\mu_f(\cdot)$ in Eq. (\ref{eq:tgp_mu}) 
        	
        	\STATE Compute $\pi(g_m),\forall m=1,...,M$  by Eq. (\ref{eq:likelihood_gp_sample})

		\STATE {\bfseries Output:} $\hat{g}(\bx) = \sum_{ m=1}^M g_m(\bx)  \pi(g_m)$
	\end{algorithmic}
\end{algorithm}
\vspace{-5pt}

\end{minipage}

We  utilize  Matheron's rule to efficiently draw samples from a \gls{GP} through  $p_{\textrm{base}}$. However, taking a standard \gls{GP} model for $p_{\textrm{base}}$ can lead to unwanted samples being taken, as demonstrated empirically in Table \ref{tab:Acceptance-ratio} especially for high-dimensional functions such as \textit{alpine1} and \textit{gSobol}. Here, we consider a sample is accepted when the density of $\pi(g_m)$ is within two standard deviation of the Gaussian distribution defined in Eq. (\ref{eq:likelihood_gp_sample}) and rejected otherwise. 


Therefore, we propose to use a transformed \gls{GP} surrogate model given $f^{+}$ for $p_{\textrm{base}}$, using a technique presented by \cite{gunter2014sampling}.
Based on Definition \ref{def:bound_random_variable}, we have that $f^+ \in \big[\max f(\bx)-2\eta_+, \max f(\bx)+2\eta_+ \big]$\footnote{$2\eta_+$ is chosen to cover $95\%$ of the probability density} with high probability. Therefore, we can write $ \forall \bx, f(\bx) \le f^+ + 2\eta_+ $. Thus, there exists $h(\cdot)$ satisfying $\forall \bx, f(\bx) = f^+ + 2\eta_+ - ½ h^2(\bx)$. Formally, we define $f(\cdot)$ via $h(\cdot)\sim \gls{GP}$ as follows:
\begin{align}
f(\cdot) & =f^{+} + 2 \eta_+ -\tfrac{1}{2} h^{2}(\cdot), \quad \text{where } h \sim \gls{GP}(m_h,K_h). \label{eq:transformed_GP}
\end{align}
This transformation ensures the resulting
function $f(\cdot)$ to satisfy the bounded constraints $f^+ 	\gtrsim f(\cdot)$ and $\exists\bx^+,f(\bx^+)\rightarrow f^{+}$
when $h(\bx^+)\rightarrow0$. We refer to Fig. \ref{fig:1d_gp_sampling} (bottom right) for an illustration of this transformation.

We estimate a posterior for drawing samples of $h(\cdot)$ as follows. Denote the data in the original space by $D_{f}=\{\bx_{i},y_{i}\}_{i=1}^{N}$, we compute $D_{h}=\{\bx_{i},h_{i}\}_{i=1}^{N}$ where $h_{i}=\sqrt{2(f^{+} + 2 \eta_+  -y_{i})}$. Next we write the posterior of $p\left( h(\bx_*) \mid D_{h} \right) \approx \mathcal{N}\left( \mu_{h}(\bx_*),\sigma^2_{h}(\bx_*) \right)$ with $\mu_{h}(\bx_*)=m_h(\bx_*)+\mathbf{k}^T_{*}\left[\mathbf{K}+\sigma_h^{2}\idenmat\right]^{-1}\left(\bh-m_h\right)$ and $\sigma^{2}\left(\bx_*\right)=k_{**}-\mathbf{k}^T_{*}\left[\mathbf{K}+\sigma_h^{2}\idenmat\right]^{-1}\mathbf{k}_{*}$ where $\sigma_h^2$ is the measurement noise variance of $h(\cdot)$ and other variables are defined analogously to Section \ref{subsec:Sampling-from-GP}. 

\begin{figure*}
\vspace{-3pt}
\begin{centering}
\includegraphics[width=0.48\textwidth]{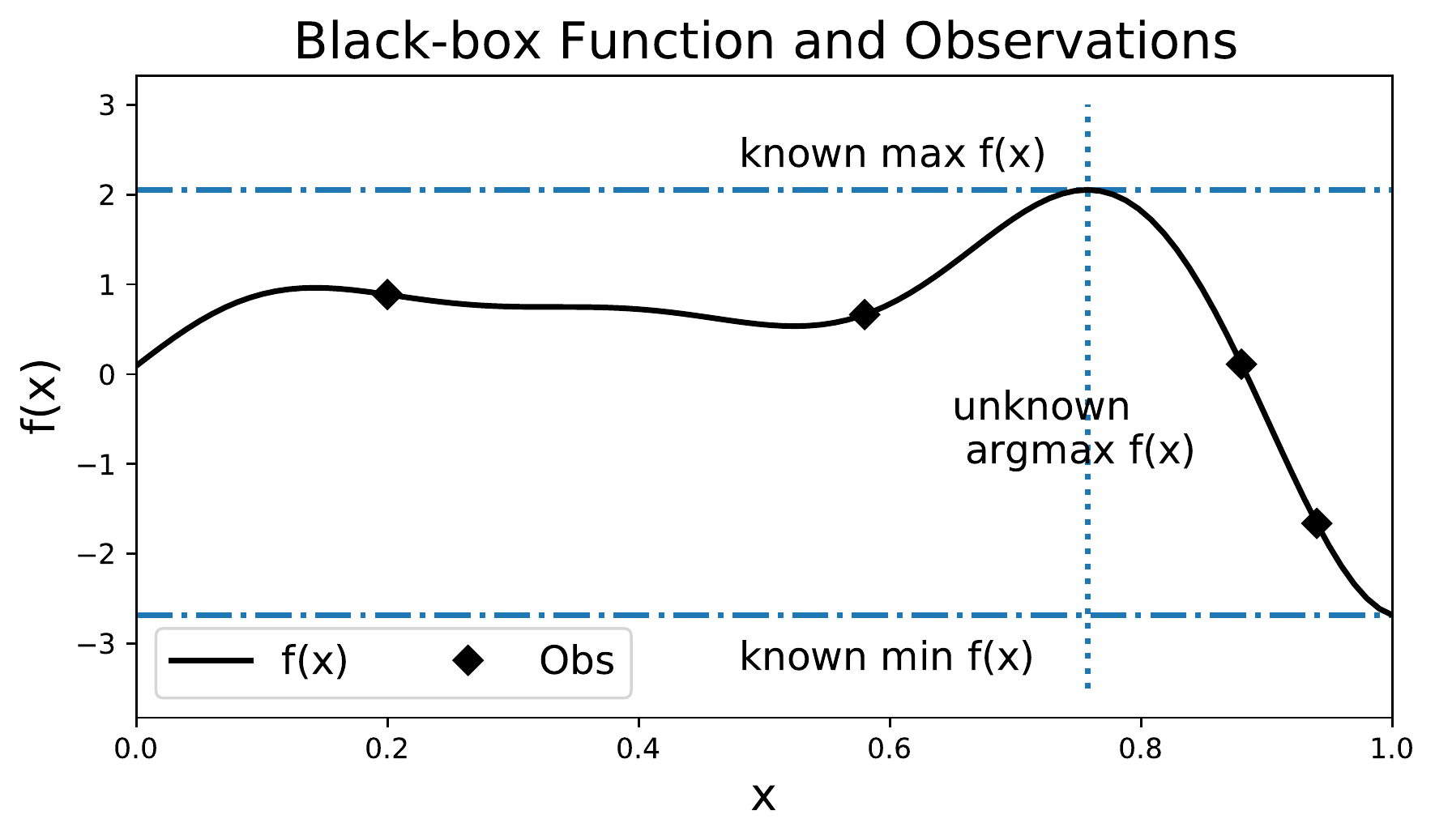}\includegraphics[width=0.48\textwidth]{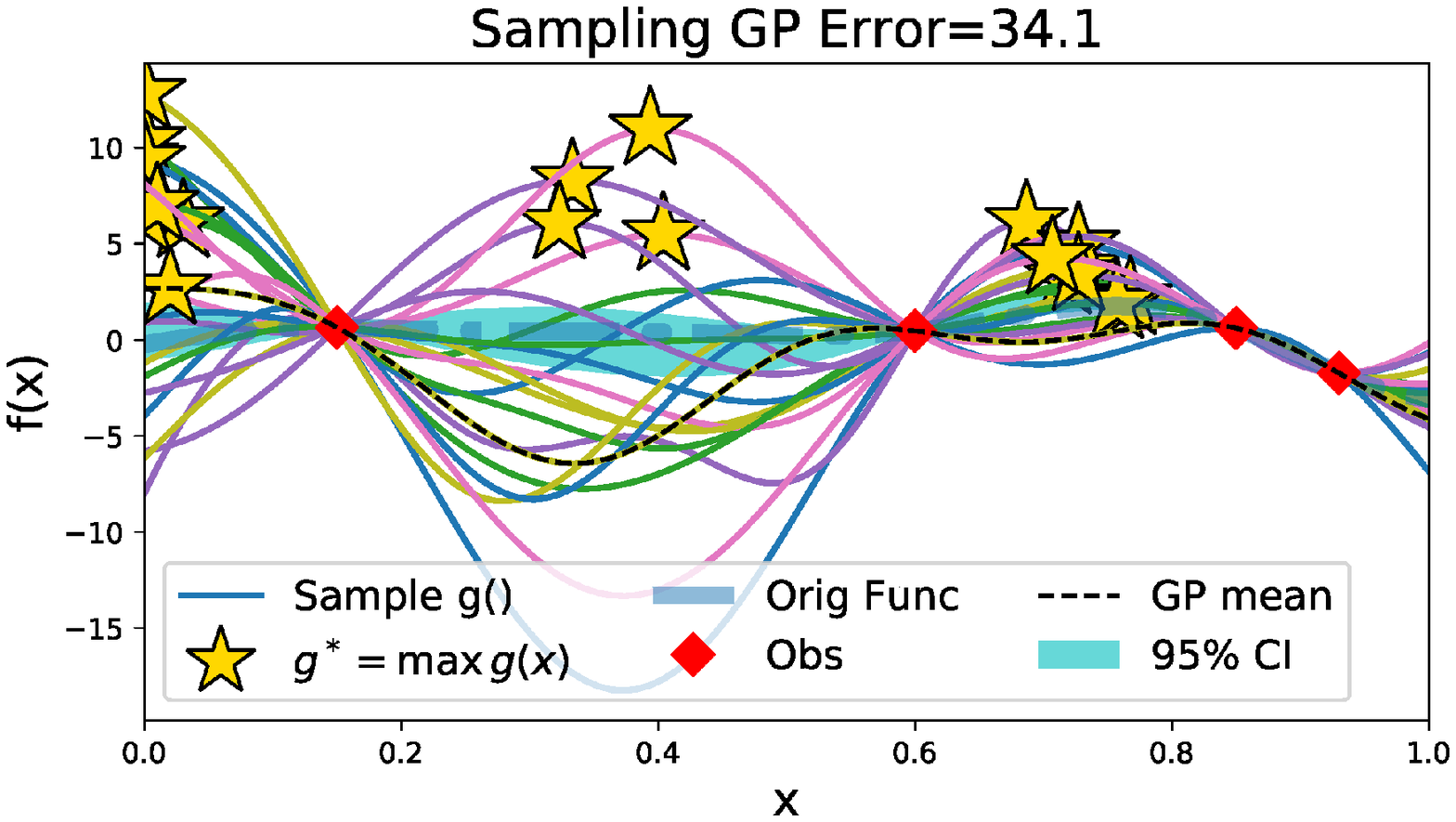}
\par\end{centering}
\vspace{-5pt}
\begin{centering}
\includegraphics[width=0.48\textwidth]{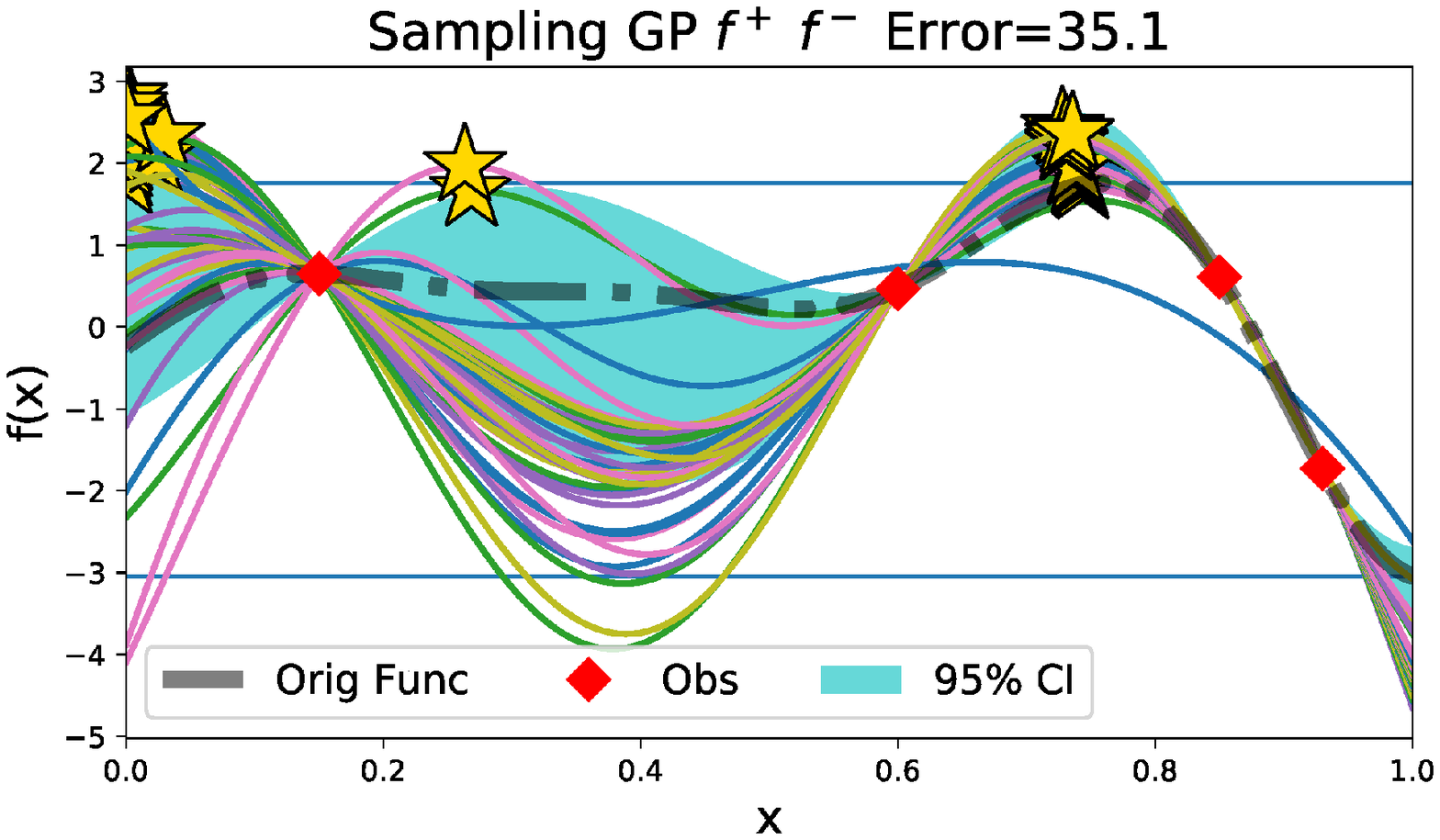}\includegraphics[width=0.48\textwidth]{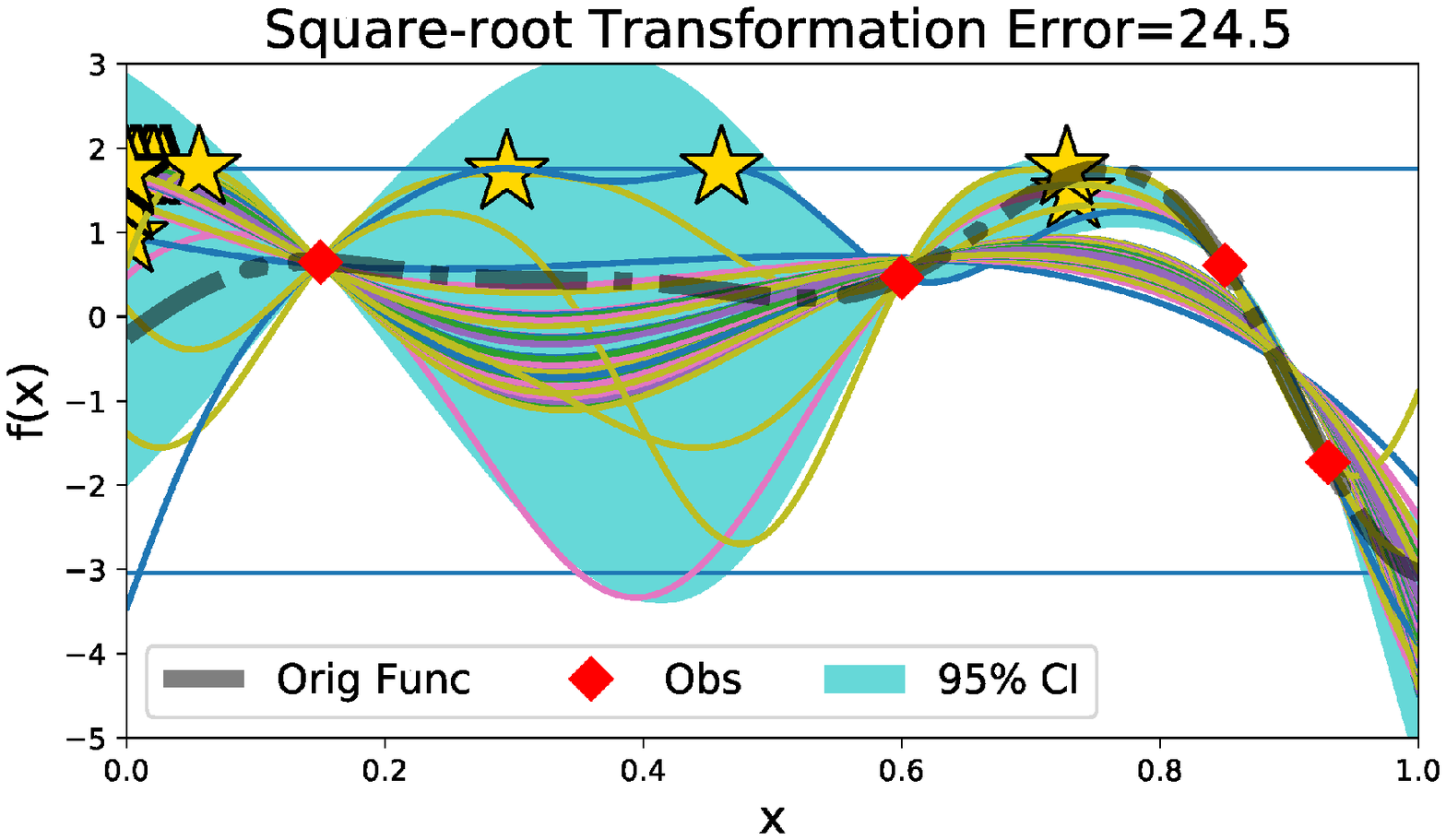}
\par\end{centering}
\vspace{-15pt}
\caption{\textit{Top Right}: Decoupled \gls{GP} sampling \citep{wilson2020efficiently} given
four observations from $f(\protect\bx)$. \textit{Bottom Left}: \gls{GP} sampling and weighting with both $f^{+}$ and $f^{-}$. \textit{Bottom Right}:
\gls{GP} sampling and weighting with square-root transformation (or \gls{SRGP}. Given limited observations,
the existing \gls{GP} sampling can be misleading to fluctuate widely beyond the bounds (see the \textit{y-axis}) while exploiting  $f^{+}$ and $f^{-}$ can result in better \gls{GP} sample.\label{fig:1d_gp_sampling}}
\vspace{-10pt}
\end{figure*}

We draw \gls{GP} samples from $h(\cdot)$ to get $f(\cdot)$ using Eq. (\ref{eq:transformed_GP}). The resulting samples of $f(\cdot)$ follow the upper bound conditions as desired---staying below $f^+$ and reaching (closely) the optimum value. After using $f^{+}$ to sample $f(\cdot)$, we further exploit $f^{-}$ via the weighting scheme using Eq. (\ref{eq:likelihood_gp_sample}).

To estimate the posterior predictive mean and variance in the original space, we use a Taylor expansion to approximate the posterior predictive distribution \citep{gunter2014sampling} $p \big(f(\bx_*) \mid \bX, \by, \eta_+, f^+ \big) \approx \mathcal{N}(\mu_{f},\sigma^2_{f})$ where

\begin{minipage}[t]{.49\textwidth}
\vspace{-20pt}
\begin{align}
\mu_{f}(\bx_*)=f^{+} + 2 \eta_+ - \tfrac{1}{2} \mu_{h}^{2}(\bx_*) \label{eq:tgp_mu}
\end{align}
\end{minipage}
\begin{minipage}[t]{.49\textwidth}
\vspace{-15pt}
\begin{align}
\sigma^2_{f}(\bx_*)=\mu_{h}(\bx_*)\sigma^2_{h}(\bx_*)\mu_{h}(\bx_*).
\label{eq:tgp_sigma}
\end{align}
\end{minipage}

Taylor expansion gives the best possible approximation at the mode which is a multivariate function. This mode is set to $\mu_g$ which is well approximated and thus for the resulting $\mu_f$ and $\sigma_f$. 
We derive the above equations in Appendix \ref{app_sec:posterior_approximation} and summarize the above process in Algorithm \ref{alg:TS_kov}. 

\subsection{Tightening the variances of \gls{GP} posterior samples\label{subsec:Improving-the-error}}




The base distribution $p_\textrm{base}$ for drawing \gls{GP} posterior samples in Eq. (\ref{eq:weighted_averaging}) can be used as (i) a standard \gls{GP} \citep{wilson2020efficiently}, (ii) weighting using \gls{GP} or (iii) weighting using \gls{SRGP} (presented in Section \ref{subsec:Square-root-Transformation-and}), respectively. For brevity, we denote these approaches as  $p_{\gls{GP}}$, $p_{\textrm{w-}\gls{GP}}$ and $p_{\textrm{w-}\gls{SRGP}}$.

We show that our approaches  $p_{\textrm{w-}\gls{GP}}$ and  $p_{\textrm{w-}\gls{SRGP}}$ tighten the sample variance relative to $p_{\gls{GP}}$ ignoring the bounds $f^+$, $f^-$. We state the main theoretical results and refer to Appendix \ref{app_subsec:proof_variance} for the proofs.

\begin{lem}
\label{lem:var_target_better_var_proposal}
In the 2d
space formed by $\left[ f^{+},f^{-} \right]$ and the \gls{GP} posterior samples $\{g^+_m, g^-_m \}_{m=1}^M$,
the sample variances by weighting \gls{SRGP} and \gls{GP} are smaller than the
\gls{GP} posterior sampling \citep{wilson2020efficiently} making no use of the bounds: ${\displaystyle \var\left[p_{\textrm{w-}\gls{GP}}\right]} \le{\displaystyle \var\left[p_{\gls{GP}}\right]}$ and ${\displaystyle \var\left[p_{\textrm{w-}\gls{SRGP}}\right]}\le{\displaystyle \var\left[p_{\gls{GP}}\right]}$.
\end{lem}
To intuitively understand our analysis, we visualize these variances in the 2d space in Fig. \ref{fig:Illustration_2d_prob_space} (left), created 
by $\left[ f^{+},f^{-} \right]$ and  $\{g^+_m, g^-_m \}_{m=1}^M$ taken from \gls{GP} posterior samples. In Lemma \ref{lem:var_target_better_var_proposal},
the sample variances  using our approaches will be smaller than the approach in \cite{wilson2020efficiently}, which makes no use of the bounds. This highlights the benefit of using the bounds to eliminate bad samples, which violate the bounds, to obtain a smaller sample variance.

Next, we show in Lemma \ref{lem:var_SRGP_better_var_GP} that using $p_{\textrm{w-}\gls{SRGP}}$
will be more sample-efficient than using $p_{\textrm{w-}\gls{GP}}$. In other words, we get more accepted samples by using the weighting \gls{SRGP}, given the same number of posterior samples $M$.

\begin{lem}
\label{lem:var_SRGP_better_var_GP}
In the 2d space formed
by $\left[ f^{+},f^{-} \right]$ and the \gls{GP} posterior samples $\{g^+_m, g^-_m \}_{m=1}^M$,
the sample variance of weighting \gls{SRGP} is smaller than weighting \gls{GP}: $\var\left[p_{\textrm{w-\gls{SRGP}}}\right]\le{\displaystyle \var\left[p_{\textrm{w-\gls{GP}}}\right]}$.
\end{lem}







\section{Bounded entropy search for Bayesian optimization given $f^{+}$ and $f^{-}$} \label{sec:BES}



The knowledge of $f^+$ and $f^-$ is also applicable and useful to inform and speed up \gls{BO}. Given a collection of \gls{GP} samples representing the underlying function $f(\cdot)$, each satisfies the bounds at different degrees. For \gls{BO} setting, we can not take all of them, but to select a single point for evaluation among these $M$ candidates.  We derive an acquisition function to find a single point maximizing the information gain about `good' \gls{GP} posterior samples satisfying the bounds.

\begin{minipage}{0.5\textwidth}
\vspace{-10pt}
\begin{algorithm}[H]
	    \caption{\gls{BO} with \gls{BES} given $f^{+},f^{-}$\label{alg:BO_KOV}}
	\begin{algorithmic}[1]
		\STATE {\bfseries Input:}  \#iter $T$, data $D_0$,    $f^{+}$,  $f^{-}$
	
        \FOR{$t=1, \dots, T$}
        	\STATE Draw \gls{GP} samples $g_{1},....,g_{M}$ given $D_{t-1}$
        	\STATE $\bx_{t}=\arg\max\alpha^{\text{BES}} (\bx\mid g_1,...,g_M)$ by Eq. (\ref{eq:final_acq_BES})
        	\STATE Evaluate $y_{t}=f(\bx_{t})$ and $D_{t}=D_{t-1}\cup\left(\bx_{t},y_{t}\right)$
    	\ENDFOR
    	
		\STATE {\bfseries Output:}  $\arg\max_{\bx\in\mathcal{X}}\mu_{T}\left(\bx\mid {D_{T}}\right)$

	\end{algorithmic}
\end{algorithm}
\end{minipage}

\paragraph{Bounded entropy search (\gls{BES}). \label{subsec:Decision-function}}
Let us write $\bx_{m}^{+}:=\arg\max_{\bx}g_m(\bx)$ as the maximum location of the $m$-th sample from a \gls{GP} posterior and 
$g_{m}^{+}=\max g_m(\bx) = g_m(\bx_{m}^{+}) , m=1,...,M$
be the corresponding maximum value. We propose to select a next point $\bx_{t} := \arg\max_{\bx \in\mathcal{X}} \alpha^{\gls{BES}}(\bx)$ by maximizing the  mutual information given the optimum samples $\{g_{m}^{+},\bx_{m}^{+}\}$ and $f^{+}$, $f^{-}$, i.e., 
\begin{align} 
\alpha^{\gls{BES}}(\bx) := \mathbb{I} \big(y_{\bx},\{g_{m}^{+}, \bx_m^+\}\mid \bX, \by,f^{+},f^{-} \big).
\end{align}
Our acquisition function aims at gaining the most information about the \gls{GP} posterior samples $g_m(\cdot)$. Equivalently it  learns the most about $f(\cdot)$, especially at the locations giving the outputs closer to the known optimum values $f^+$. Although the idea of gaining information about the optimums has been extensively used in recent works \citep{Hennig_2012Entropy,Hernandez_2014Predictive,Wang_2017Max}, we are \textit{the first} to extend it to situations when the knowledge about $f^+$ and $f^-$ is available. 

Denote $y_{\bx}:=\mathbb{E}[f_{\bx}\mid\bX,\by,\bx]$ and $\Lambda:=\{\bX,\by,f^{+},f^{-}\}$ for brevity, we can write the mutual information using the KL divergence as $\mathbb{I}(y_{\bx},\Upsilon_m \mid \Lambda)$
\vspace{-5pt}
\small
\begin{align} 
  \approx \: & \frac{1}{M} \sum_{ g_{m}^{+},\bx_{m}^{+} \sim p_{\textrm{base}}} \int p\left(g_{m}^{+},\bx_{m}^{+}, y_\bx\mid \Lambda \right) \log\frac{p\left(y_{\bx},g_{m}^{+},\bx_{m}^{+}\mid \Lambda\right)}{p\left(y_{\bx}\mid \Lambda \right) p\left(g_{m}^{+},\bx_{m}^{+}\mid \Lambda\right)} dy_\bx\nonumber \\ 
= \: & \frac{1}{M} \mathbb{E}_{p(y_{\bx}\mid \Lambda)} \left[\sum_{ g_{m}^{+},\bx_{m}^{+}}\hspace{-3mm} p\left(g_{m}^{+},\bx_{m}^{+}\mid y_{\bx},\Lambda \right)\log\frac{p\left(g_{m}^{+},\bx_{m}^{+}\mid y_{\bx}, \Lambda\right)}{ p\left(g_{m}^{+},\bx_{m}^{+}\mid \Lambda \right)} \right]   \label{eq:final_acq_BES}
\end{align}
\normalsize
where $\{g_{m}^{+},\bx_{m}^{+}\} \sim p_{\textrm{base}}$ is generated as follows: $g_m \sim p_{\textrm{base}}$, $\bx_{m}^{+}=\arg \max g_m(\cdot)$, $g^+_m = g_m(\bx^+_m)$ and $y_\bx \approx g_m(\bx)$. We have used the product rule $p(y_{\bx},g_m^+,\bx_m^+\mid \Lambda)=p(y_{\bx} \mid \Lambda)p(g_m^+,\bx_m^+ \mid y_\bx, \Lambda)$ to simplify the terms. Although $\bx_m^+$ and $g_m^+$ can be estimated efficiently given the analytical form of $g_m(\cdot)$, we are unable to compute the objective function in Eq. (\ref{eq:final_acq_BES}) analytically, but with approximations.  

\begin{figure*}
\vspace{-7pt}
\begin{centering}
\includegraphics[width=0.325\textwidth]{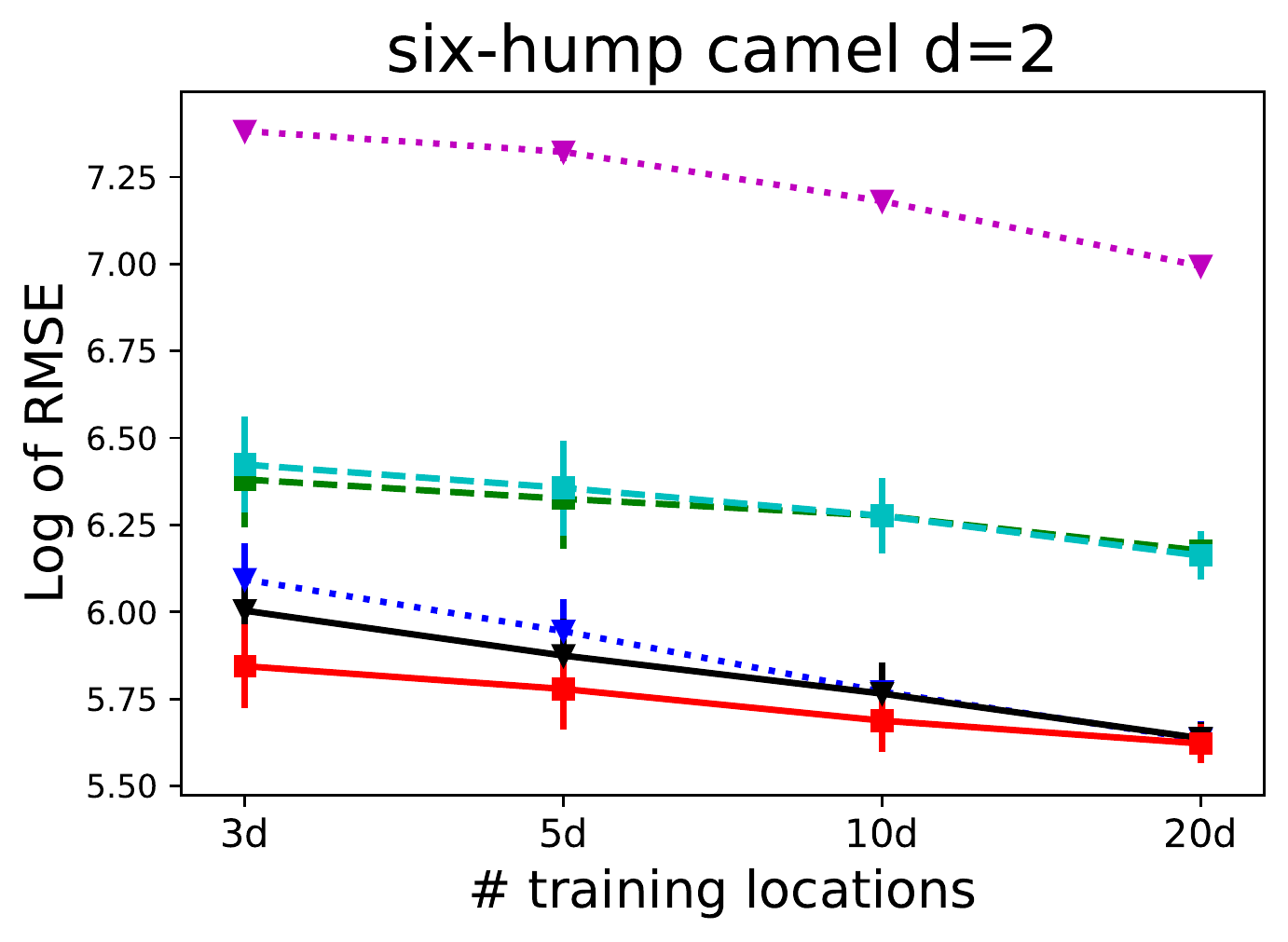}
\includegraphics[width=0.325\textwidth]{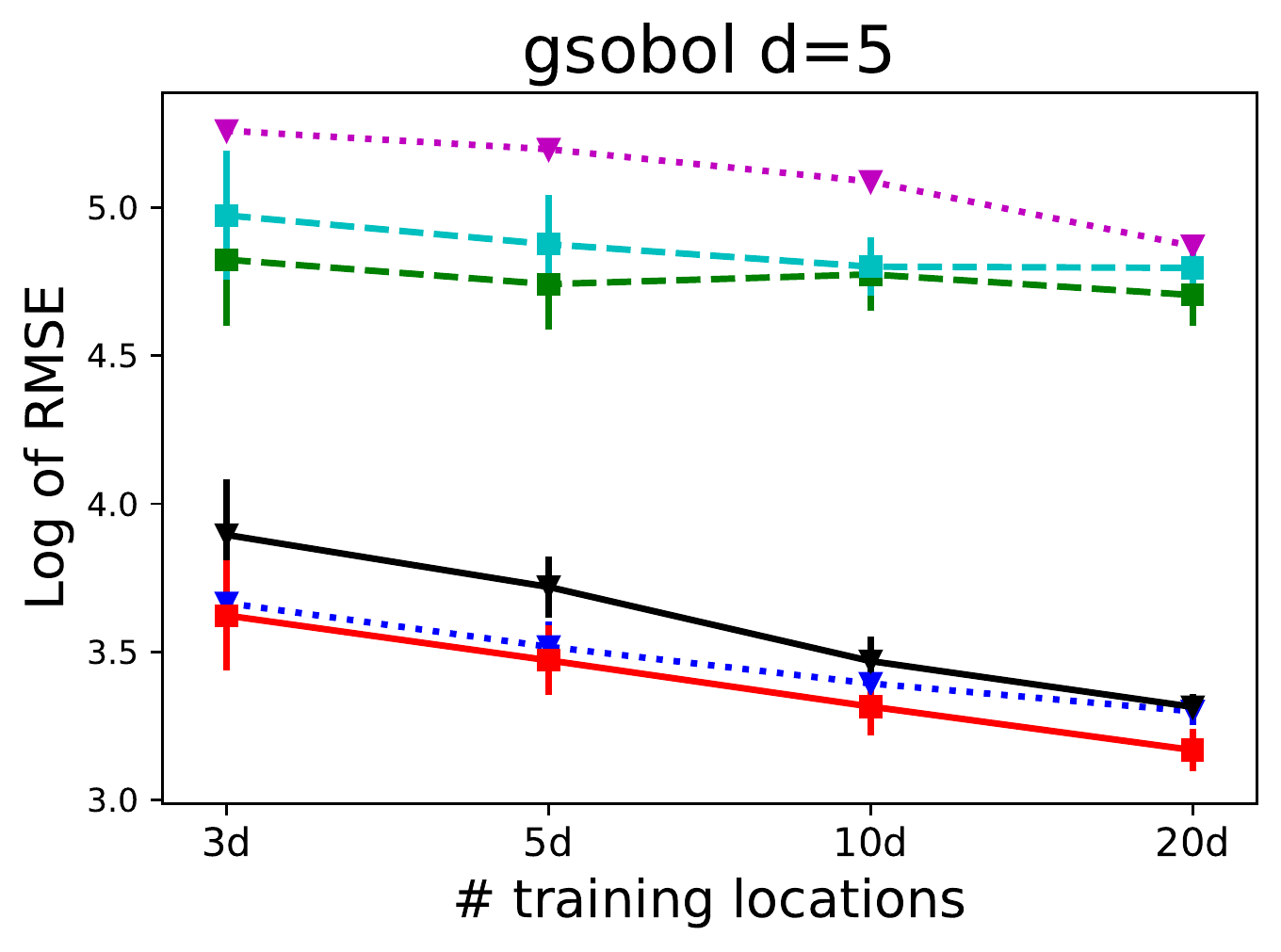}
\includegraphics[width=0.325\textwidth]{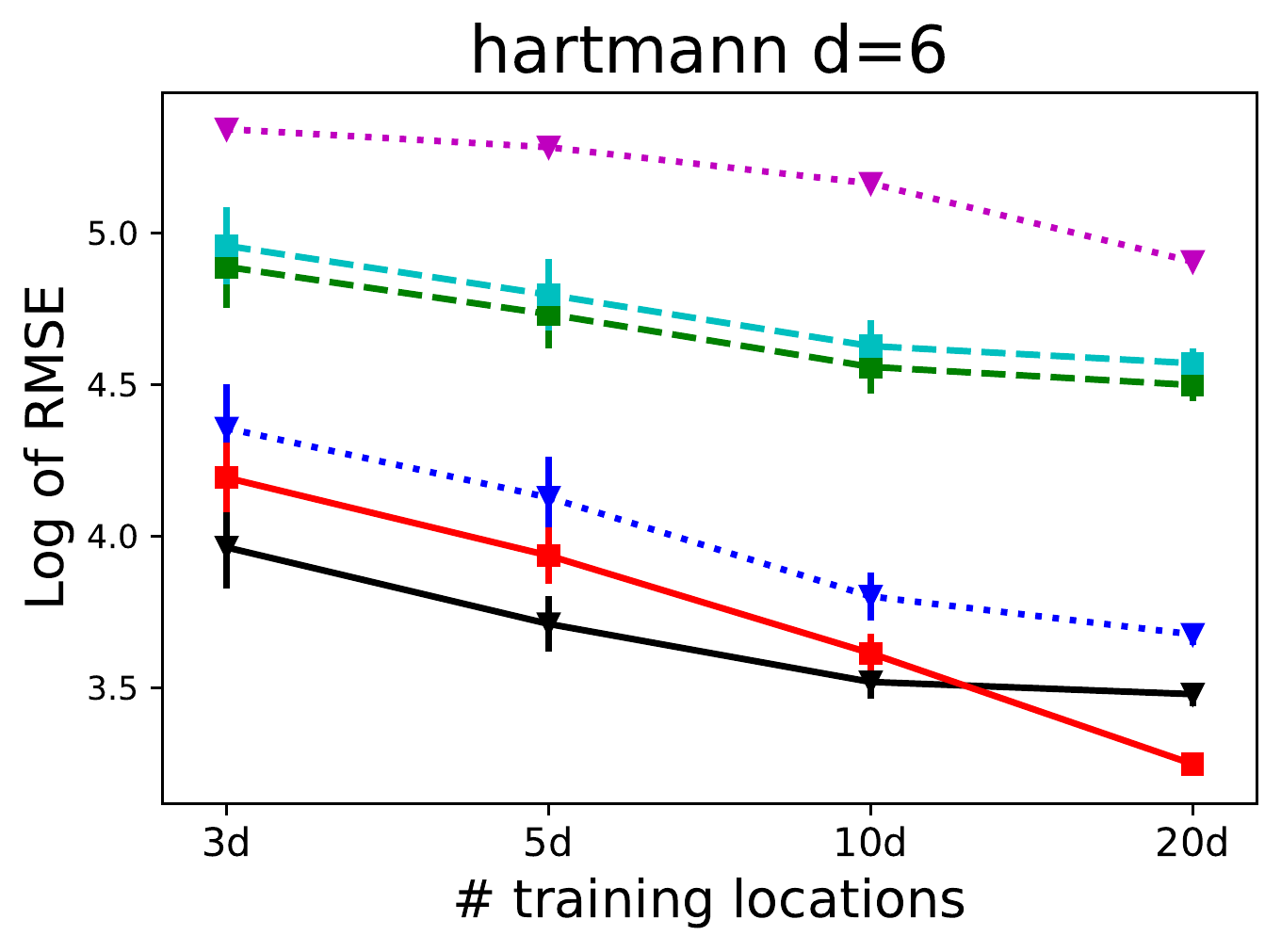}
\vspace{-5pt}
\includegraphics[width=0.82\textwidth]{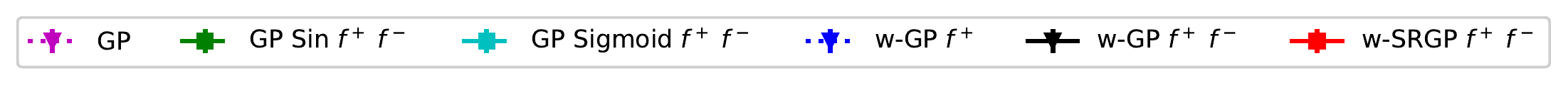}
\par\end{centering}
\vspace{-7pt}
\caption{Exploiting external knowledge about $f^{+}$ and $f^{-}$ leads to better \gls{GP} posterior sampling.  \label{fig:gp_sampling_comparison}}
\vspace{-9pt}
\end{figure*} 

Given $\pi(g_m)=\mathcal{N} \big( \left[g^-_m, g^+_m \right] \mid \left[ f^{-}, f^+ \right],\text{diag}[\eta^2_-,\eta^2_+] \big)$ in Eq. (\ref{eq:likelihood_gp_sample}), we want to quantify the influence of a new data point $\{\bx,y_\bx\}$ on a given \gls{GP} posterior sample. We write the likelihood with and without the presence of $\bx$. 
\begin{align} \label{eq:BES_density_gmax_given_yx}
p\left(  g_m^+,\bx_m^+  \mid \Lambda\right)  &=  \textcolor{blue}{ p\left(  g_{m}^{+} \mid \bx_{m}^{+}, \Lambda\right) }    \textcolor{orange}{ p \left( \bx_{m}^{+}\mid  \Lambda\right) }  \\ 
&\approx  \textcolor{blue}{ \pi(g_m) } \textcolor{orange}{ \mathcal{N} \left( g_{m}(\bx^+_m) \mid \mu(\bx^+_m), \sigma^2(\bx_{m}^{+}) \right) } \nonumber
\\ 
p\left( g_m^+,\bx_m^+ \mid y_{\textcolor{red}{\bx}}, \Lambda \right) &= p\left(  g_{m}^{+} \mid \bx_{m}^{+}, y_{\textcolor{red}{\bx}}, \Lambda\right)    p \left( \bx_{m}^{+}\mid y_{\textcolor{red}{\bx}}, \Lambda\right)  \label{eq:BES_density_gmax} \\
&\approx  \pi(g_m)  \mathcal{N} \left( g_{m}(\bx^+_m) \mid \mu_{\textcolor{red}{\bx}}(\bx^+_m), \sigma^2_{ \textcolor{red}{\bx}}(\bx_{m}^{+}) \right) \nonumber
\end{align}
where $\mu_{\bx}(\bx_{m}^{+})$,  $\mu(\bx_{m}^{+})$, $\sigma^2_{\bx}(\bx_{m}^{+})$, and $\sigma^2(\bx_{m}^{+})$\footnote{ we use a subscript $\sigma_{\bx}(\cdot)$ to indicate the inclusion of $\bx$.} are the \gls{GP} posterior predictive means and variances at $\bx_{m}^{+}$ with and without the inclusion of $\bx$. These predictive quantities are defined in Eqs. \eqref{eq:tgp_mu} and \eqref{eq:tgp_sigma}.  Plugging Eqs. (\ref{eq:BES_density_gmax_given_yx}) and \eqref{eq:BES_density_gmax} into Eq. (\ref{eq:final_acq_BES}), we obtain the final form of $\alpha^{\gls{BES}}(\bx)$.

In Eq. (\ref{eq:BES_density_gmax_given_yx}), the likelihood $\textcolor{orange}{p \left( \bx_{m}^{+}\mid  \Lambda\right)}$ can be defined to a prior knowledge about the input $\bx^+_m$ if available, such as \cite{li2020incorporating,souza2021bayesian}. Instead of using uniform prior, we can have better estimation by using the GP predictive mean $\mu(\bx^+_m)$ and variance $\sigma^2(\bx_{m}^{+})$  available to quantify the likelihood of the location of interest $\bx^+_m$ being the optimum (given the data). 

Without knowing the true output at a test location $\bx$, we can simplify the \gls{GP} predictive mean $\mu_{\bx}(\bx^+_m) \approx \mu(\bx^+_m)$.  This simplification saves computation cost for updating the \gls{GP} predictive mean at each considered location $\bx$.


The term $\pi(g_m)$ in Eq. (\ref{eq:BES_density_gmax}) will be retained when plugging into Eq. (\ref{eq:final_acq_BES}). Thus, $\pi(g_m)$ will remove any \gls{GP} sample which does not follow our bounded constraints. As a result, our decision only gains information about the `good' samples. In Appendix \ref{subsec:BES-without-rejection_ablation}, we provide an ablation study by accepting all samples (without weighting step), which drops the performance significantly.






\vspace{-4pt}

\paragraph{Explainable decision.}
The point selected by Eq. (\ref{eq:final_acq_BES}) will gain the maximum information
about the posterior samples. Our decision can be interpretable in each single term. The acquisition function value in Eq. (\ref{eq:final_acq_BES}) is high when the density $\frac{p\left(g_m^+,\bx_m^+\mid y_{\bx},\Lambda\right)}{p\left(g_m^+,\bx_m^+\mid \Lambda\right)}$ and $p\left(g_m^+,\bx_m^+\mid y_{\bx},\Lambda\right)$ are high, defined in Eqs. (\ref{eq:BES_density_gmax_given_yx}) and \eqref{eq:BES_density_gmax}. 
The first term $\frac{p\left(g_m^+,\bx_m^+\mid y_{\bx},\Lambda\right)}{p\left(g_m^+,\bx_m^+\mid \Lambda\right)} \propto \mathcal{N}(g_{m}^{+} \mid \mu(\bx^+_m),\left[\sigma^2_{\bx}(\bx_{m}^{+})^{-1}-\sigma^2(\bx_{m}^{+})^{-1}\right]^{-1})$ is high when $\sigma^2_{\bx}(\bx_{m}^{+}) \ll \sigma^2(\bx_{m}^{+})$ given that $\mathbb{E} [g^+_m] = \mu(\bx^+_m)$. To reduce this predictive uncertainty of $\sigma^2_{\bx}(\bx^+_m)$, our acquisition function encourages to place a point $\bx$ at the perceived optimum location $\bx^+_m$. Similarly, the second term $p\left(g_m^+,\bx_m^+\mid y_{\bx},\Lambda\right)$ takes high value when $\pi(g_m)$ is high and $\sigma^2_{ \bx}(\bx_{m}^{+})$ is small. This encourages (i) the sampled \gls{GP} has the maximum value $g_{m}^{+}$ consistent with $f^{+}$ and (ii) taking a location at $\bx_{m}^{+}$. 
 We refer to  Appendix  \ref{sec:Illustration-of-GES} for the illustration of the $\alpha^{\gls{BES}}$ and  \ref{app_sec:implementation_complexity} for the implementation discussion. We summarize all steps for computing \gls{BES} in Algorithm \ref{alg:BO_KOV}.

\section{Experiments} 
\label{sec:experiments}

We demonstrate the two claims in exploiting the knowledge of $f^{+}$and $f^{-}$ for improving \gls{GP} posterior sampling and \gls{BO}. We also perform ablation studies by varying the misspecified levels of such bounds for each setting. All experiments  are averaged over $30$ independent runs, the number of \gls{GP} samples $M=200$,\footnote{The ablation study with different choices of $M\in \{20,50,200,300,500\}$ is available in Appendix \ref{app_subsec:wrtM}} $\eta^2_+=0.02d$ and $\eta^2_-=0.5d$ where $d$ is the input dimension. The dimension of random Fourier features is set to default as $l=100$. We normalize the input $\bx\in [0,1]^d$ and standardize the output $y\sim \mathcal{N}(0,1)$ for robustness. We follow the common practice in optimizing \gls{GP} hyperparameter by maximizing the \gls{GP} log marginal likelihood \citep{Rasmussen_2006gaussian}. We refer to the appendix for additional experiments, illustrations, and ablation studies. We attach the Python source code in the submission and will release publicly in the final version.

\paragraph{\gls{SRGP} vs \gls{GP}.}
The key advantage of \gls{SRGP} is that we
can generate `good' samples more often than using a standard \gls{GP} in Section
\ref{subsec:Sampling-from-GP_rejection_sampling}. We demonstrate
this benefit in Fig. \ref{fig:Illustration_2d_prob_space} showing that \gls{SRGP} with weighting (or w-\gls{SRGP}) outperforms w-\gls{GP}. The results are more significant in higher dimensional functions where the w-\gls{GP} cannot get a single
sample accepted for \textit{alpine1} in $5d$ using $M=200$. 

\begin{figure*}
\vspace{-1pt}
\centering
\includegraphics[width=0.32\textwidth]{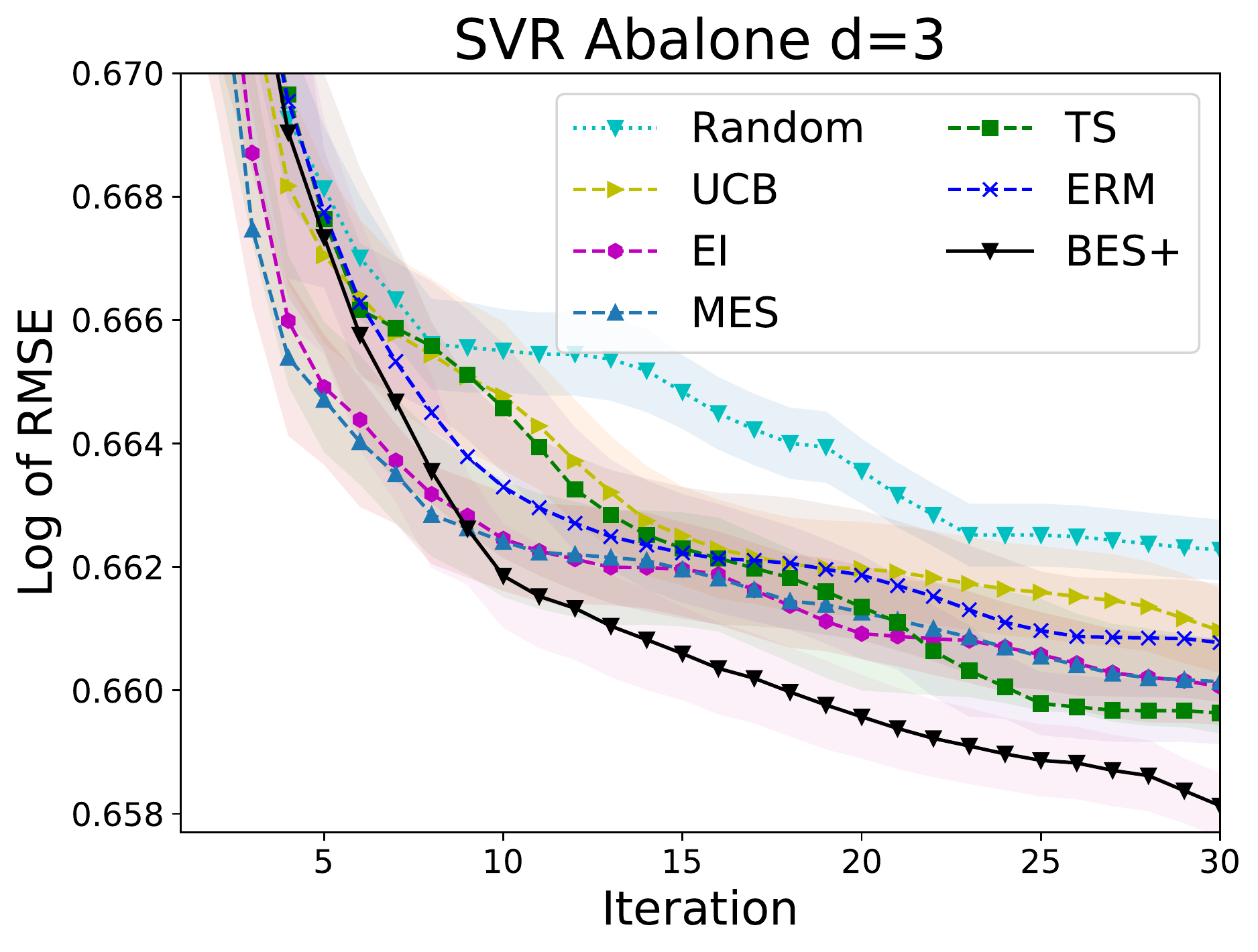}\includegraphics[width=0.32\textwidth]{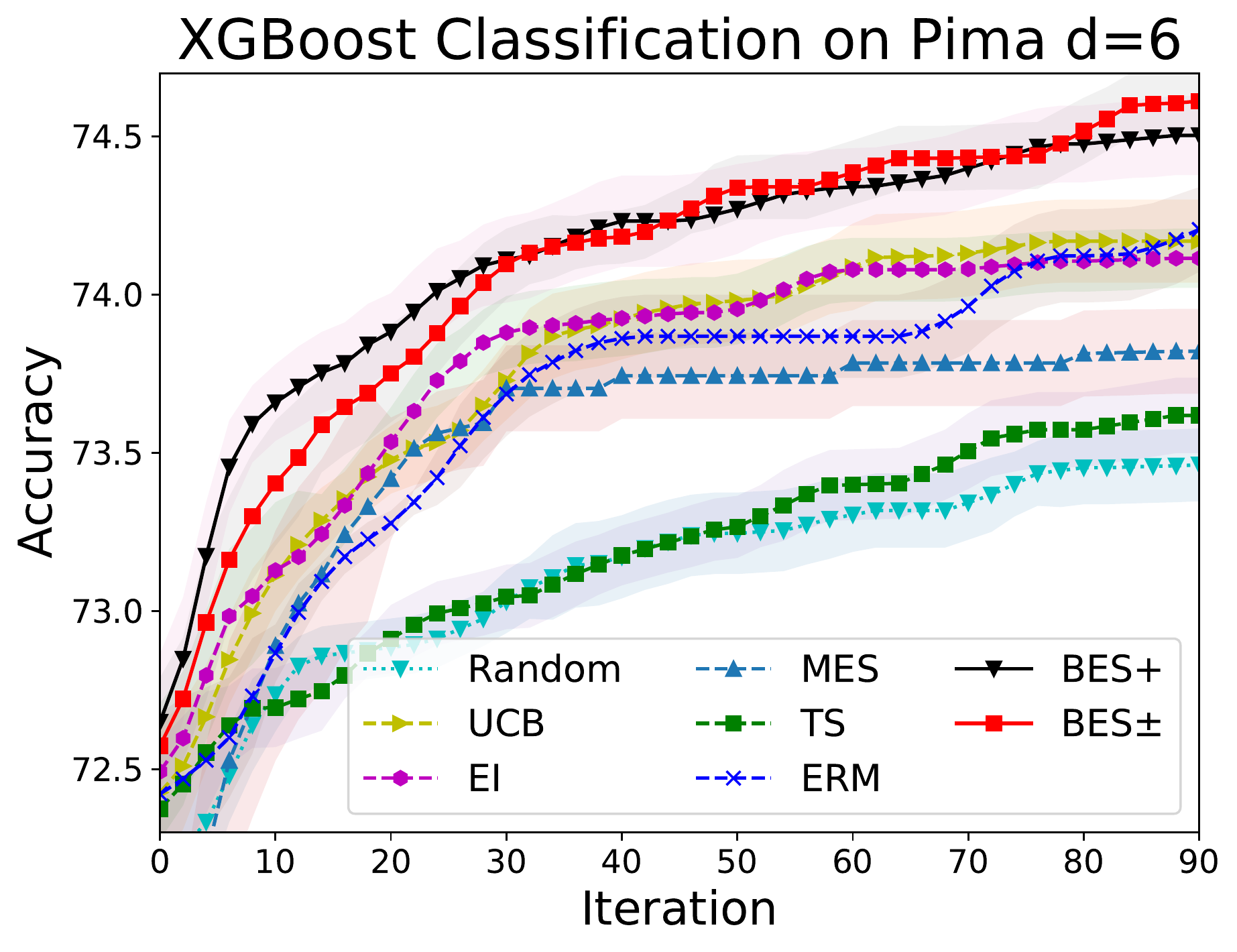}\includegraphics[width=0.32\textwidth]{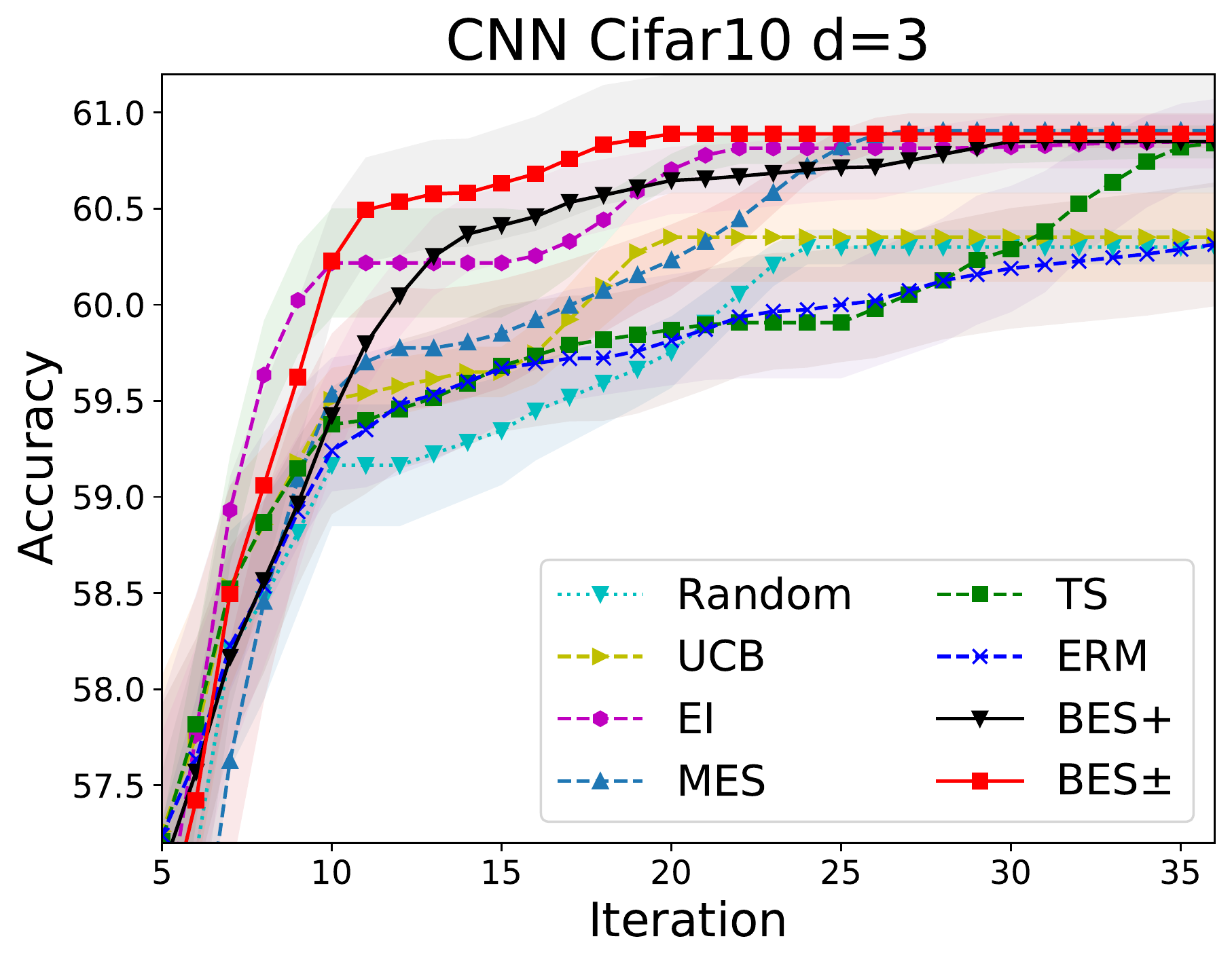}
\vspace{-3pt}
\caption{Machine learning hyperparameter tuning tasks. Exploiting the knowledge about optimum values will lead to better optimisation performance. \gls{BES}$\pm$ (with $f^{+}$ and $f^{-}$) performs the best. In \textsc{svr} setting, we only use $f^+$ while we have ultilized both $f^+$ and $f^-$ for \textsc{xgboost} and \textsc{cnn}.\label{fig:BO_result}}
\vspace{-4pt}
\end{figure*}

\begin{figure*}
\vspace{-2pt}
\begin{centering}
\includegraphics[width=0.93\textwidth]{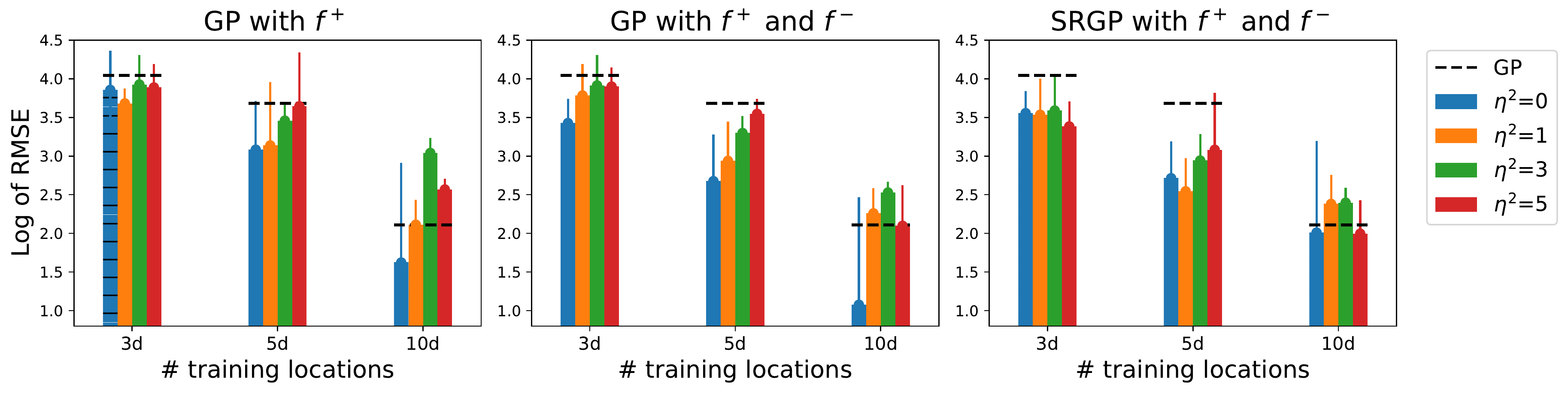}
\par\end{centering}
\begin{centering}
\par\end{centering}
\vspace{-6pt}
\caption{Our method can improve the performance of \gls{GP} posterior sampling without using $f^+$ or $f^-$ \citep{wilson2020efficiently}
(black dash line) under misspecification of $f^{+},f^{-}$. The gap is more significant given fewer training observations. The plot is using Forrester function. See the Appendix Fig. \ref{fig:ablation_gp_sampling_noise} for more results.  \label{fig:GP_sampling_misspecified}}
\vspace{-10pt}
\end{figure*}

\paragraph{Improving \gls{GP} posterior sampling. \label{subsec:exp_GP_sampling}}

A set of training locations $\bX_{\textrm{train}}\sim U[0,1]^{N_{\textrm{train}}\times d}$
is randomly generated and corresponding outputs $\by_{\textrm{train}}=f(\bX_{\textrm{train}})$
are subsequently observed. We construct a \gls{GP} given $\left\{ \bX_{\textrm{train}},\by_{\textrm{train}}\right\} $
and draw $M=200$ \gls{GP} posterior samples for each sampling scheme. Due to the trade-off between the number of accepted \gls{GP} samples versus the sampling quality for these methods with exploiting $f^+$ and $f-$, we rank and select the same ($M'=100$) number of posterior samples from the initial set of $M=200$ samples above. These samples are used to compute the \gls{RMSE} error against the true $f(\cdot)$. 

We vary the number of training locations $N_{\textrm{train}}=\{3d,5d,10d,20d\}$ and report the results in Fig. \ref{fig:gp_sampling_comparison} which
suggest that increasing the number of training observations will reduce
the error for all sampling schemes. More importantly, incorporating
the knowledge about $f^{+}$ and $f^{-}$ will significantly improve the \gls{GP} posterior
sampling performance. Furthermore, using both $f^{+}$ and $f^{-}$
results in better performance than using $f^{+}$ alone. Our approach surpasses the standard \gls{GP} sampling \citep{wilson2020efficiently}, which does not use the bound information, by a wide margin. The logistic and sigmoid transformations making use of the bounds are generally better than the standard GP, but are inferior to our proposed methods. We present further details and visualization of these transformations in Appendix \ref{app_subsec:Sampling-from-GP_sigmoid_sin}.


\paragraph{Improving Bayesian optimization. \label{subsec:exp_improving_BO}}
We next present a \gls{BO} task. We follow a popular evaluation criterion in information theoretic
approaches \citep{Hennig_2012Entropy,Hernandez_2014Predictive} that
selects the inferred argmax of the function for evaluation, i.e. $\widetilde{\bx}_{T}=\arg\max_{\bx\in\mathcal{X}}\mu_{T}\left(\bx\mid D_{T}\right)$.  The number of \gls{BO} iterations is $10d$, initialized
by $d$ observations randomly where $d$ is the number of input dimension.
We compare our model with standard \gls{BO} methods which do not use $f^{+}$ and
$f^{-}$, including \textsc{gp-ucb} \citep{Srinivas_2010Gaussian}, \gls{EI}
\citep{Mockus_1978Application}, max-value entropy search (\textsc{mes}) \citep{Wang_2017Max} and Thompson sampling (\gls{TS}). We compare with \gls{ERM} \citep{nguyen2020knowing}, which can make use of either $f^{+}$ or $f^{-}$, but not both of them. We compare the performance using popular benchmark functions and machine learning hyperparameter tunings including support vector regression \citep{Smola_1997Support}
on Abalone using root mean squared error (RMSE) as the main metric, \textsc{xgboost} \citep{chen2016xgboost} classification on Pima
Indians diabetes, and a \textsc{cnn} \citep{lecun1995convolutional} on \textsc{cifar10}.   We refer to Appendix \ref{subsec:BES-without-rejection_ablation} for studying a variant of \gls{BES} where $\pi(g_m)=\frac{1}{M}$ is set uniformly and Appendix \ref{subsec:Additional-experiments} for additional experiments. 
\begin{figure}
\centering
\vspace{-5pt}
\includegraphics[width=0.38\textwidth]{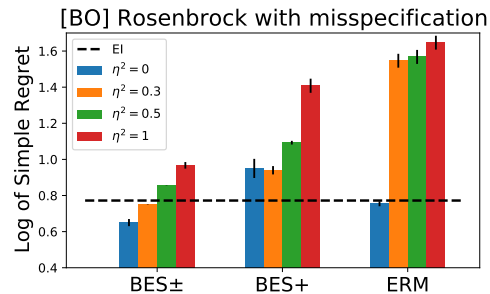} \hspace{5pt}
\vspace{-5pt}
\caption{Our \gls{BES} is more robust than \gls{ERM} in dealing with misspecification.\label{fig:BO_misspecified_Rosenbrock}}
\vspace{-10pt}
\end{figure}


\gls{BES}$\pm$ is an abbreviation for \gls{BES} using both $f^+$ and $f^-$ while \gls{BES}$+$ uses $f^+$ alone. We show in Fig. {\ref{fig:BO_result}} that exploiting the optimum knowledge in our \gls{BES} and \gls{ERM} should perform competitively. Notably, our \gls{BES} often achieves high performances at the \textit{earlier stage} of the optimization process where we do not have much information to infer the underlying function. Therefore, utilizing $f^+,f^-$ will bring significant benefits at the earlier stage that the existing approaches are unable to exploit. We note that when $f^{+}$ and $f^{-}$ are available, the \gls{ERM} is unable to exploit both.

\paragraph{Misspecifying $f^{+}$ and $f^{-}$.}
We perform ablation studies over different misspecified levels of the optimum values via $\eta^2:=\eta^2_{+}=\eta^2_{-}\in\{0,1,3,5\}$ for \gls{GP} sampling and $\eta^2\in\{0,0.3,0.5,1\}$ for \gls{BO}. Using the benchmark functions, we have access to the true values of $\max f(\cdot)$ and $\min f(\cdot)$.  Thus, the misspecified value is set as $f^+ = \max f(\cdot) \pm \eta^2_+$ where $\pm$ is randomly selected by a Bernoulli random variable. $f^-$ is specified analogously.

For a meaningful study, we consider these values after standardizing the output space $y\sim\mathcal{N}(0,1)$. This means that when $\eta^2=1$, the gap between the true $\max f(\cdot)$ and the specified $f^{+}$ is one standard deviation of $f(\cdot)$. We present further experimental results in the Appendix Fig. \ref{fig:ablation_gp_sampling_noise}.


When $\eta^2\le1$, the performance of \gls{GP} posterior
sampling in Fig. \ref{fig:GP_sampling_misspecified} is much better than \cite{wilson2020efficiently} in all
cases and when $\eta^2\le5$ our model is still better in a few cases.
Relaxing $\eta^2\rightarrow\infty$ will let the performance of w-\gls{GP}
reduce to the case of \cite{wilson2020efficiently} while the performance
of w-\gls{SRGP} will slightly suffer because the transformation depends
on $\eta^2_+$. Thus, we suggest to use our weighting \gls{SRGP} when the misspecification
is such that $\eta^2\le1$, one standard deviation of $f(\cdot)$, and
use w-\gls{GP} otherwise.


We study misspecifying $f^{+}$ and $f^{-}$ for \gls{BO}. In Fig. \ref{fig:BO_misspecified_Rosenbrock},
our \gls{BES} performs competitively the best,
especially when $\eta^2\le0.3$ in all cases. The information gain strategy in \gls{BES} is more noise-resilient in dealing with misspecification than \gls{ERM}. Especially if we underspecify $f^+$, the performance of \gls{ERM} will significantly drop as shown in \cite{nguyen2020knowing}, while our \gls{BES} can tolerate better. Utilizing both information about $f^{+}$ and $f^{-}$ will help to cope with misspecification better than using $f^{+}$ alone.


\begin{figure}
\vspace{-4pt}
\centering
\includegraphics[width=0.42\textwidth]{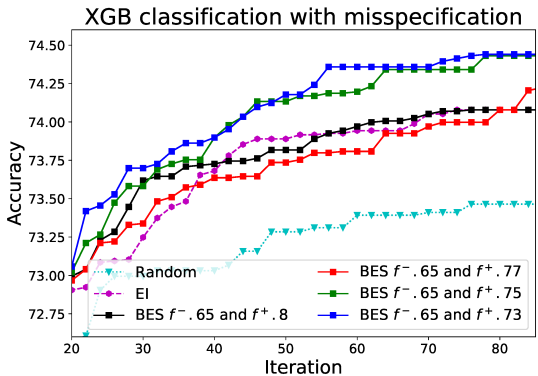}
\vspace{-5pt}
\caption{We specify $f^{+}$ to different values and set $f^{-}$ fixed in
\textsc{xgboost}. The performance is better when $f^{+}$ is set closer to
the true value $\approx0.745$.\label{fig:BO_misspecified}}
\vspace{-6pt}
\end{figure}

In Fig. \ref{fig:BO_misspecified}, we vary the values of $f^{+}$ within $[0.7, 0.9]$
and fix $f^{-}=0.65$ in the \textsc{xgboost} classification problem where
we do not know the true value for these bounds. The result indicates
that the best performance for \gls{BES} is achieved when $f^{+}$ is set
close to $\approx0.745$, which is likely the true (but unknown) upper bound. Our \gls{BES} with misspecifying $f^+ \in [0.73,0.75]$ still performs generally better than \gls{EI}.



\section{Conclusion, limitations and future work} \label{sec:conclusion}
\paragraph{Conclusion.} We have presented a new setting and approaches for exploiting the
upper and lower bounds available for some black-box functions. We exploit these bounds for improving the performance of \gls{GP} posterior sampling and \gls{BO}. The resulting \gls{GP} posterior sampling approach has tighter sample variance bounds than existing methods. Our approaches can be used as a plugin extension to existing distributed and asynchronous Thompson sampling as well as sparse \gls{GP} models.

\paragraph{Comparisons with \gls{ERM} \citep{nguyen2020knowing}.}
Our key advantages against \gls{ERM} is as follows.
First, our approach can handle the information about \textit{both} the max
and min of $f(\cdot)$ while the setting in \gls{ERM}
can only handle \textit{one} of them.
Second, our approach is more widely applicable in taking the optimum values which can be \textit{loosely} specified while \gls{ERM}
expects to observe the \textit{precise} value. Particularly, \gls{ERM} will perform
poorly if users underspecified the true value of $\max_{\bx}f(\bx)$.
Third, our approach is useful for both \gls{GP} posterior sampling and \gls{BO} settings while the one suggested by  \cite{nguyen2020knowing} is limited to \gls{BO}.

\paragraph{Limitations.} While we believe the results above are insightful, there are a number of limitations one needs to be aware of. First, our \gls{BES} relies on the \gls{GP} posterior samples to make a decision. There will be some iterations where none of the \gls{GP} samples stays within  the approximate bounds, such as when  $M$ is small or
the available data is biased to represent $f(\cdot)$. Whenever this happens, our \gls{BES} may be not applicable; instead we can  simply perform \gls{EI} \citep{Jones_1998Efficient}
for this iteration. We found this simple trick with \gls{EI} works well
empirically and illustrate it in  Appendix \ref{subsec:Switching-between-EI}.
Second, we acknowledge the challenge when the bounds are heavily mis-specified, where such knowledge about the bounds can be less beneficial. Thus, the results in Fig. \ref{fig:BO_misspecified} suggest that
if our loose bounds are defined more than $\eta^2>0.3$, it may be better
to use existing acquisition functions, such as \gls{EI}, without using the external knowledge.

\textbf{Future work} can extend the model to optimize multiple objectives simultaneously, each of which comes with loose upper and lower bounds. Using such bounds, other future works can be to improve parallel Thompson sampling \citep{Hernandez_2017Parallel,Kandasamy_2018Parallelised} or used jointly with other signals, such as monotonicity, for the best performance in the low data regime. 
Another direction for improving the sample efficiency is to consider important sampling on the space of random Fourier feature while we expect the high dimensional challenge.

\newpage

\bibliographystyle{apalike}

\bibliography{vunguyen}

\newpage

\begin{figure*}
\begin{centering}
\subfloat[We have a black-box function $f(\protect\bx)$ including four observations
(red dots), the maximum value $f^{-}$ and minimum value $f^{-}$.
The goal is to find the unknown maximum location $\arg\max f(\protect\bx)$.\label{fig:blackbox_4obs}]{\begin{centering}
\includegraphics[width=0.8\textwidth]{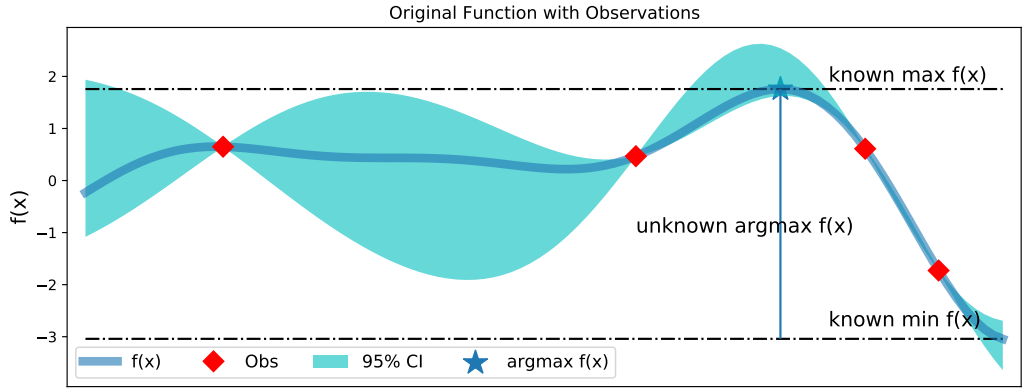}
\par\end{centering}
}
\par\end{centering}
\begin{centering}
\subfloat[\textit{Top}: \gls{GP} posterior sampling $p_{\textrm{base}}$ without the knowledge
about $f^{+}$ and $f^{-}$. \textit{Bottom}: the decision function will select
the left corner which is not the correct maximum location.\label{fig:illustration_BES_wo_rs}]{\begin{centering}
\includegraphics[width=0.8\textwidth]{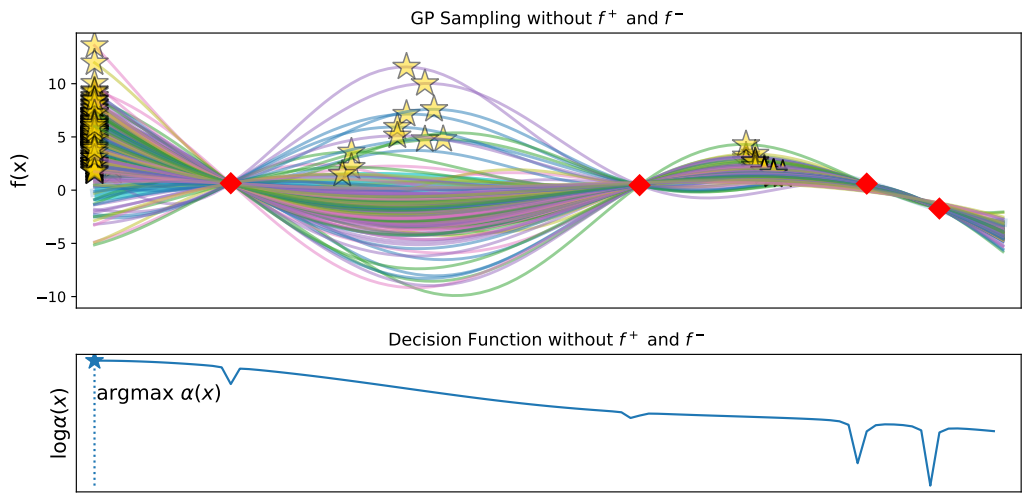}
\par\end{centering}

}
\par\end{centering}
\begin{centering}
\subfloat[\textit{Top}: \gls{GP} posterior sampling $\pi(g_m)$ guided by the $f^{+}$
and $f^{-}$. \textit{Bottom}: our decision selects the next point for evaluation
correctly as the unknown maximum location, i.e., $\arg\max\alpha^{\textrm{\gls{BES}}}(\protect\bx)=\arg\max f(\protect\bx)$.\label{fig:illustration_BES_rs}]{\begin{centering}
\includegraphics[width=0.8\textwidth]{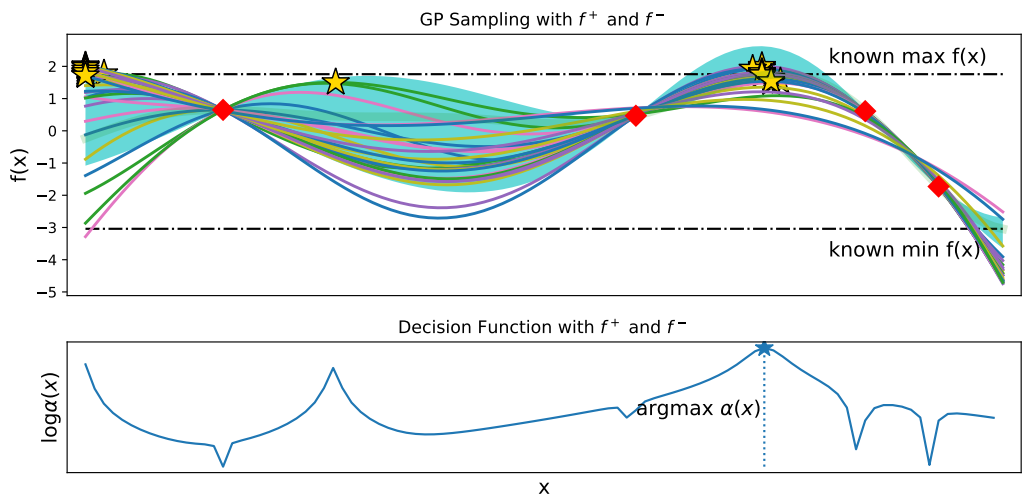}
\par\end{centering}
}
\par\end{centering}
\caption{Illustrating the \gls{BES}. \textit{Top}: a black-box function $f(\protect\bx)$
and four observations (red dots). Middle: $p_{\textrm{base}}$  as a \gls{GP} and the decision
without the weighting $\pi(g_m)$. \textit{Bottom}: $p_{\textrm{base}}$  as a \gls{SRGP} and the decision with the weighting $\pi(g_m)$ .\label{fig:decision_function_illustration}}
\end{figure*}

\appendix
\onecolumn

\newpage
\appendix

\newpage

\section{Illustration of Bounded Entropy Search (\gls{BES}) Decision Function \label{sec:Illustration-of-GES}}

We first illustrate the decision function of \gls{BES} given the knowledge
of $f^{+}=\max f(\bx)+\mathcal{N}(0,\eta^2_+)$ and $f^{-}=\min f(\bx)+\mathcal{N}(0,\eta^2_-)$
in Fig. \ref{fig:decision_function_illustration}. Given limited observations
in Fig. \ref{fig:blackbox_4obs}, it may be difficult to infer the
unknown optimum location $\arg\max f(\bx)$. However, we can utilize
the information about the upper bound and lower bound to better identify
this location of interest. In particular, we show that incorporating
$f^{+}$ and $f^{-}$ can help to identify correctly the unknown maximum
location in Fig. \ref{fig:illustration_BES_rs} while we may fail
without such extra knowledge as shown in Fig. \ref{fig:illustration_BES_wo_rs}.

\section{Further Technical Details}

\subsection{Transformation by sinusoidal and sigmoid \label{app_subsec:Sampling-from-GP_sigmoid_sin}}


\begin{figure*}
\begin{centering}
\includegraphics[width=0.49\textwidth]{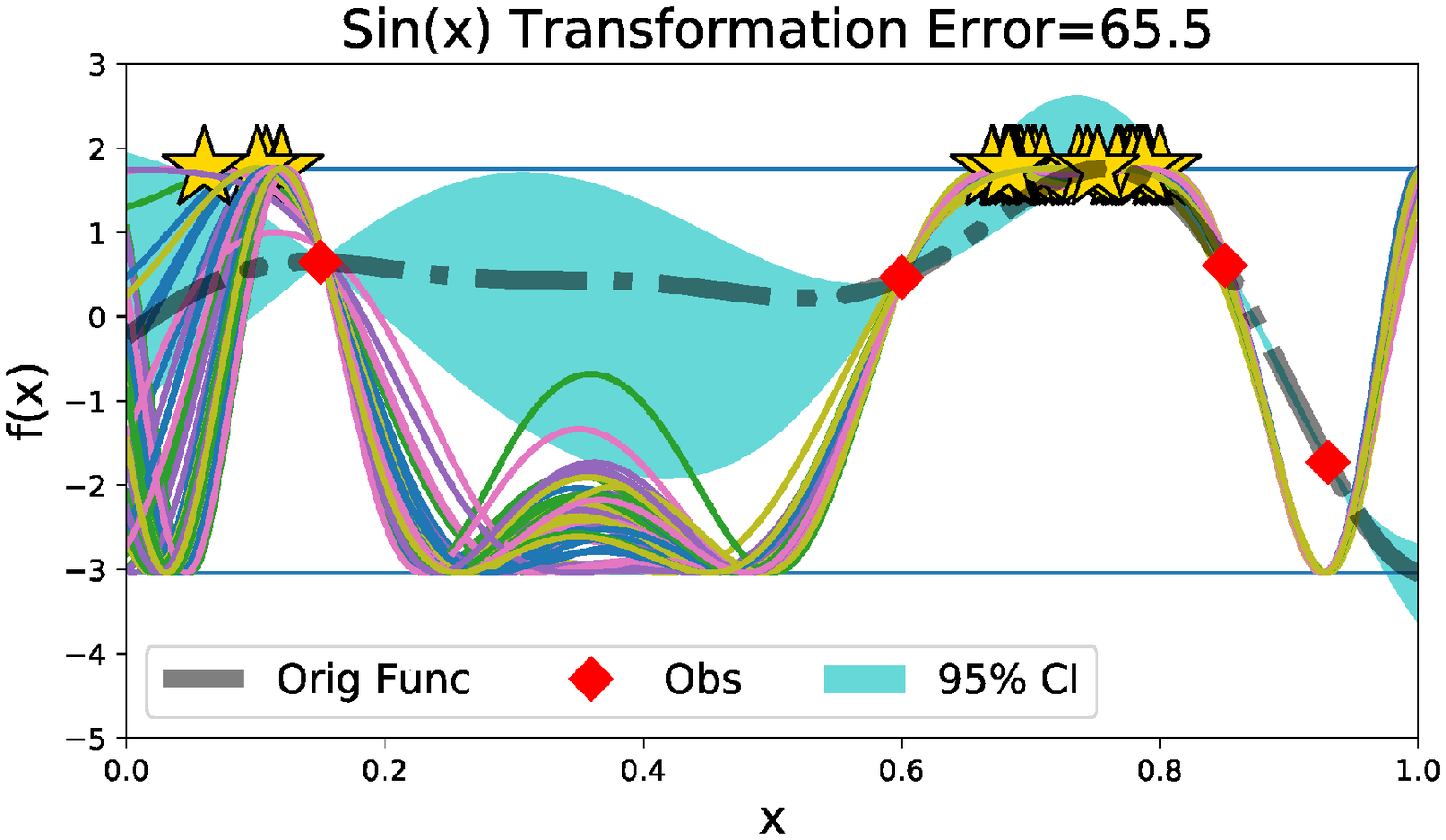}\includegraphics[width=0.49\textwidth]{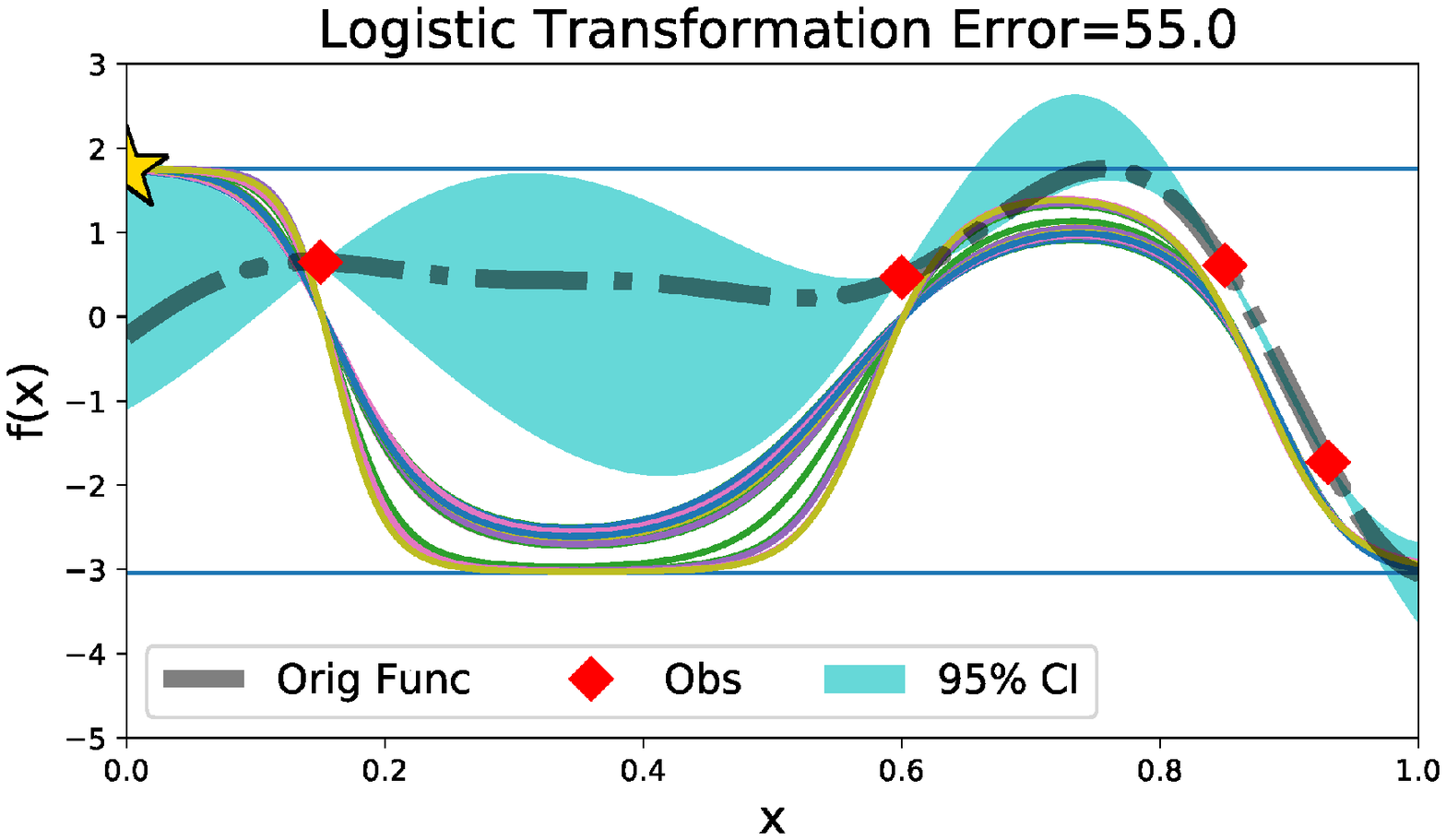}
\par\end{centering}
\vspace{-10pt}
\caption{\gls{GP} posterior sampling with $f^{+},f^{-}$ transformations by sinusoidal and
sigmoid functions. They both result in not good estimation although
they satisfy the bounded constraints. The major issue with these transformations is that the \gls{GP} samples are stretched out too much at the two tails $f^+$ and $f-$. \label{fig:GP-sampling-with_sin_sigmoid-1}}
\end{figure*}
We present two other ways to transform the \gls{GP} posterior samples by incorporating
the knowledge about $f^{+}$ and $f^{-}$. We utilize the sin wave
and sigmoid transformations. For such transformations, we define
a forward function to compute $f(\bx)$ given $h(\bx)$ which is constrained by either $\arcsin$ or $\log$. Then, we define a reverse function to compute $h(\bx)$
given $f(\bx)$. These forward and reverse functions are summarized
below:
\begin{itemize}
\item Using sinusoidal transformation: we define $\frac{f(\bx)-f^{-}}{f^{+}-f^{-}}=\frac{1}{2}\sin[h(\bx)]+\frac{1}{2}$
and $h(\bx)=\arcsin\left\{ 2\left[\frac{f(\bx)-f^{-}}{f^{+}-f^{-}}-0.5\right]\right\} $.
\item Using sigmoid transformation: we define $\frac{f(\bx)-f^{-}}{f^{+}-f^{-}}=\frac{1}{1+\exp \bigl(-h(\bx) \bigr)}$
and $h(\bx)=\log\frac{f(\bx)-f^{-}}{f^{+}-f^{-}}$.
\end{itemize}
Given $f^{+}$ and $f^{-}$, we transform the output $\by$ from the original space of $f(\cdot)$ to the space of $h(\cdot)$ as $h(\bx)=\arcsin\left\{ 2\left[\frac{f(\bx)-f^{-}}{f^{+}-f^{-}}-0.5\right]\right\} $
for sinusoidal and $h(\bx)=\log\frac{f(\bx)-f^{-}}{f^{+}-f^{-}}$
for sigmoid, respectively. We then draw samples from \gls{GP} posterior of
$h(\cdot)$ given $\left\{ \bx_i,h(\bx_i)\right\}_{i=1}^N$. Next, we transform these samples back to the original space of $f(\cdot)$.

We note that sinusoidal and sigmoid transformations \textit{require} the knowledge
about both $f^{+}$ and $f^{-}$ while square-root transformations and the proposed  weighting strategy
 are more \textit{flexible} to utilize either of $f^{+}$, $f^{-}$
or both of them.

We visualize the \gls{GP} samples by using sinusoidal and sigmoid functions in Fig. \ref{fig:GP-sampling-with_sin_sigmoid-1}.
However, we note that such transformations will not result in promising
samples as we have encoded the periodicity by sin wave or being 
stretched out by the sigmoid curve. These additional results justify our
choice of squared-root transformation and weighting using $f^+$ and $f^-$ in Section
\ref{subsec:Sampling-from-GP_rejection_sampling}. We  have also presented an empirical comparison in Fig. \ref{fig:gp_sampling_comparison}.




\subsection{Implementation and Computational Complexity} \label{app_sec:implementation_complexity}
We discuss the implementation aspect of our acquisition function, defined in Eq. (\ref{eq:final_acq_BES}) which involves sampling $M$ \gls{GP} posterior samples $g_1,...g_M$. The cost for this sampling is similar to that of \cite{wilson2020efficiently}. 

The Gaussian likelihood $\pi(g_m)$ is computed for each \gls{GP} sample $g_m(\cdot)$ once. We then optimize the acquisition function $\alpha^{\gls{BES}}(\bx)$ by iteratively evaluating it at test points $\bx$. The evaluation from the \gls{GP} posterior sample is cheap, thanks to the scalability of \gls{GP} posterior sampling using Matheron's rule that goes linearly with the number of test points \citep{wilson2020efficiently,wilson2021pathwise}.

In particular, the computation of Eq. (\ref{eq:final_acq_BES}) reduces to Eqs. (\ref{eq:BES_density_gmax_given_yx},\ref{eq:BES_density_gmax}) for which we need to compute the \gls{GP} posterior predictive mean $\mu(\bx^+_m)$ and variance $\sigma^2(\bx^+_m)$ in Eqs. (\ref{eq:tgp_mu}, \ref{eq:tgp_sigma})  with and without the presence of the considered location $\bx$.

Note that in computing Eq. (\ref{eq:BES_density_gmax_given_yx}), we need to efficiently estimate the variance at $\bx^+_m$ with the inclusion of $\bx$, i.e., $\sigma_{\bx}^2(\bx^+_m)$. This can be efficiently computed by using the Block-wise matrix inversion lemma
\begin{align} \label{eq:matrix_inversion}
K_{\bX\cup\bx,\bX\cup\bx}^{-1} & =\left[\begin{array}{cc}
K^{-1} & 0\\
0 & 0
\end{array}\right]+\frac{1}{\underbrace{k_{\bx,\bx}-\bk_{\bx,\bX}\bK_{\bX,\bX}^{-1}\bk_{\bx,\bX}^{T}}_{\sigma^2(\bx)}}\left[\begin{array}{cc}
\bK^{-1}\bk_{\bx,\bX}^{T}\bk_{\bx,\bX}\bK^{-1} & -\bK^{-1}\bk_{\bx,\bX}^{T}\\
-\bk_{\bx,\bX}\bK^{-1} & 1
\end{array}\right].
\end{align}
Particularly, we decompose the inverse covariance matrix of $K_{\bX\cup\bx,\bX\cup\bx}^{-1}$ given $K^{-1}$ which is precomputed from the previous iteration. We then compute the remaining terms in the right handside using matrix multiplication to update $K_{\bX\cup\bx,\bX\cup\bx}^{-1}$.

We report the time taken for computing \gls{BES} per iteration in Fig. \ref{fig:time_bes}. We can see that the computational cost for \gls{BES} increases with dimension. This is because the number of time required to evaluate the acquisition function will grow with dimension. Nevertheless, the overall cost for suggesting a point is still less than two minutes for $d=10$ while the black-box evaluation time is significantly higher, such as it takes a few hours to train a deep learning model.

\begin{figure}
\centering
\includegraphics[width=0.6\textwidth]{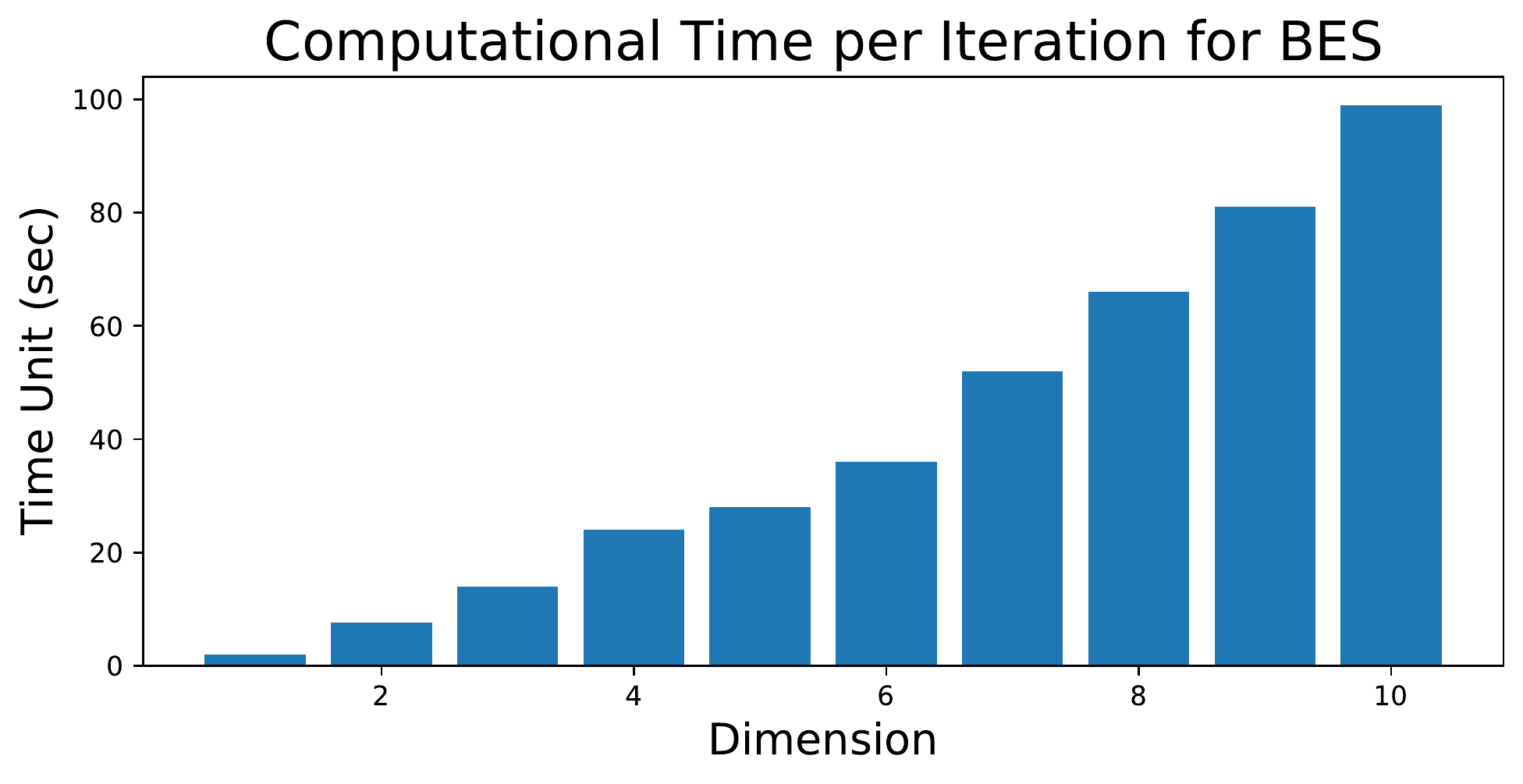}
\vspace{-5pt}
\caption{The computational cost for \gls{BES} increases with dimension. Given an expensive black-box function $f(\cdot)$ which can take hours for evaluation, our computational time is reasonable within less than two minutes for $10$ dimensions. The experiments are running on the Intel Core i7, 64GB Ram.\label{fig:time_bes}}
\end{figure}


\subsection{Decoupling Approach for \gls{GP} Posterior Sampling} \label{appendix_subsec_decouple_gp_posterior}

The central idea of decoupled sampling relies on the Matheron\textquoteright s rule for
Gaussian random variables \citep{journel1978mining,chiles2009geostatistics,doucet2010note} that uses variational Fourier features for an approximate prior and a deterministic data-dependent update term to efficiently draw samples.
Let $\phi(\bx_{i})=\sqrt{\frac{2}{l}}\cos\left(\theta_{i}^{T}\bx_{i}+\tau_{i}\right)$,
where $\theta_{i}$ is sampled proportional to the kernel\textquoteright s
spectral density, $\tau_{i}\sim \mathbb U(0,2\pi)$ and $l$ is the random feature dimension. We draw samples $f_{*}\mid\bX,\by,\bx_{*},\omega_m $ from a \gls{GP} posterior and evaluate them at any test input $\bx_{*}\in\mathcal{X}$ as
\begin{align}
\vspace{-10pt}
g_m(\bx_*) = \sum_{i=1}^{l} \underbrace{w_{i}\phi_{i}(\bx_{*})}_{\text{random}}+\sum_{j=1}^{N} \underbrace{v_{j}k(\bx_{*},\bx_{j})}_{\text{deterministic update}}\label{eq:GP_sample_decoupled}
\vspace{-10pt}
\end{align}
where $w_i \overset{\textrm{iid}}{\sim} \mathcal{N}(0,1)$, $\bw=[w_i,...,w_l]^T \in\mathbb{R}^{l}$, $\Phi=\phi(\bX)\in\mathbb{R}^{N\times l}$ and  $v=\big(K+\sigma_{f}^{2}\idenmat\big)^{-1}\left(\by-\Phi\bw\right)$. The sampling process above relies on the randomness of the Fourier features \citep{Rahimi_2007Random} via $\omega_m \sim p(\omega)$ 
\begin{align}
\mathbb{E}_{p(\omega)}\left[f_*\mid\bX,\by,\omega\right] & \approx \frac{1}{M}\sum_{\omega_{m} \sim p(\omega) }g(.\mid\bX,\by,\omega_{m})\label{eq:GPsampling}
\end{align}
where $\omega_{m}=\{\theta_{m},\tau_{m}\}$.  


\subsection{Tightening the variances\label{app_subsec:proof_variance}}

To put things in context, the base distribution $p_\textrm{base}$ for drawing \gls{GP} posterior samples in Eq. (\ref{eq:weighted_averaging}) can be used as (i) standard \gls{GP} \citep{wilson2020efficiently}, (ii) weighting using \gls{GP} (presented in Section \ref{subsec:defining_p_target}) or (iii) weighting using \gls{SRGP} (Section \ref{subsec:Square-root-Transformation-and}), respectively. We denote these approaches as  $p_{\gls{GP}}$, $p_{\textrm{w-}\gls{GP}}$ and $p_{\textrm{w-}\gls{SRGP}}$.   We characterize the sample variance of the proposed approach and show that our strategy in weighting the \gls{GP} posterior samples will tighten the variance of the \gls{GP} samples considering in the $2d$ space formed shown in Fig. \ref{fig:Illustration_2d_prob_space} (left). 
We adapt the Popoviciu's-inequality \citep{popoviciu1935equations}
to bound the variance in our two dimensional space considered in Fig.
\ref{fig:Illustration_2d_prob_space} (left).

\begin{lem}
\label{app_lem:(Popoviciu's-inequality-)} Let 
$l\in\mathbb{R}^{2}$ be a random variable  restricted to $[L_{1},U_{1}]$ in the 1st
dimension and $[L_{2},U_{2}]$ in the 2nd dimension, we bound the
variance of $l$ as \textup{$\textrm{Var}[l_{1}]\le\frac{\left(U_{1}-L_{1}\right)^{2}}{4}$
and $\textrm{Var}[l_{2}]\le\frac{\left(U_{2}-L_{2}\right)^{2}}{4}$.}
\end{lem}

\begin{proof}
Let us denote $[L_{1},L_{2}]=\inf \, l$ and $[U_{1},U_{2}]=\sup \, l$.
We define a function $h$ by $h(t)=\mathbb{E}\left[(l-t)^{2}\right]$
that $h''>0$.

Computing the derivative $h'$, and solving $h'(t)=-2\mathbb{E}[l]+2t=0$,
we find that $g$ achieves its minimum at $t=\mathbb{E}[l]$. We then
consider the value of the function $h$ at the specific point $t=\left[\frac{U_{1}+L_{1}}{2},\frac{U_{2}+L_{2}}{2}\right]$.
Next, we derive in the first dimension that
\begin{align*}
\textrm{Var}[l_{1}] & =h\left(\mathbb{E}[l_{1}]\right)\le h\left(\frac{U_{1}+L_{1}}{2}\right)\\
 & =\mathbb{E}\left[\left(l_{1}-\frac{U_{1}+L_{1}}{2}\right)^{2}\right] \quad \quad \quad \quad \quad \quad \quad \quad \textrm{by\thinspace the\,definition\thinspace of}\thinspace h()\\
 & =\frac{1}{4}\mathbb{E}\left[\left((l_{1}-L_{1})+(l_{1}-U_{1})\right)^{2}\right]\\
 & \le\frac{1}{4}\mathbb{E}\left[\left(U_{1}-L_{1}\right)^{2}\right]=\frac{\left(U_{1}-L_{1}\right)^{2}}{4}.
\end{align*}
In the last inequality, we have $\bigl((l_{1}-L_{1})+(l_{1}-U_{1})\bigr)^{2}\le\bigl((l_{1}-L_{1})-(l_{1}-U_{1})\bigr)^{2}=\bigl(U_{1}-L_{1}\bigr)^{2}$ due to $l_{1}-L_{1}\ge0$ and $l_{1}-U_{1}\le0$.
Similarly, we have the term in the second dimension bounded $\textrm{Var}[l_{2}]\le\frac{\left(U_{2}-L_{2}\right)^{2}}{4}$.
This concludes our proof.
\end{proof}
\begin{lem} (Lemma \ref{lem:var_target_better_var_proposal} in the main paper)
In the 2d
space formed by $\left[ f^{+},f^{-} \right]$ and the \gls{GP} posterior samples $\{g^+_m, g^-_m \}_{m=1}^M$,
the sample variances by weighting our \gls{SRGP} and \gls{GP} are smaller than the
\gls{GP} posterior sampling \citep{wilson2020efficiently} making no use of the bounds: ${\displaystyle \var\left[p_{\textrm{w-}\gls{GP}}\right]} \le{\displaystyle \var\left[p_{\gls{GP}}\right]}$ and ${\displaystyle \var\left[p_{\textrm{w-}\gls{SRGP}}\right]}\le{\displaystyle \var\left[p_{\gls{GP}}\right]}$.
\end{lem}

\begin{proof}
We first bound the sample variance in each dimension $d=\{1,2\}$ in Lemma \ref{app_lem:(Popoviciu's-inequality-)} as
\begin{minipage}[t]{0.48\columnwidth}%
\hspace{-20pt}
\begin{align*}
{\displaystyle \var\left[p_{\textrm{\gls{GP}}}^{(d)}\right]} \le\frac{\left(U_{\textrm{\gls{GP}}}^{(d)}-L_{\textrm{\gls{GP}}}^{(d)}\right)^{2}}{4},
\end{align*}
\end{minipage}%
\begin{minipage}[t]{0.48\columnwidth}%
\begin{align*}
{\displaystyle \var\left[p_{\textrm{w-\gls{GP}}}^{(d)}\right]} & \le\frac{\left(U_{\textrm{w-\gls{GP}}}^{(d)}-L_{\textrm{w-\gls{GP}}}^{(d)}\right)^{2}}{4}
\end{align*}
\end{minipage}

\hspace{-1pt}where the base distribution can be used as (i) a standard \gls{GP} , i.e., ${\displaystyle p_{\textrm{\gls{GP}}}}=p\bigl(f(\cdot) \mid\bX,\by \bigr)$
or (ii) a weighted \gls{GP} with the bounds information ${\displaystyle p_{\textrm{w-\gls{GP}}}}:=\pi(g_m)=p\bigl( f(\cdot) \mid\bX,\by,f^{+},f^{-}\bigr)$.

In our \gls{GP} sampling and weighting setting,
we have $U_{\textrm{w-\gls{GP}}}^{(d)}<U_{\textrm{\gls{GP}}}^{(d)},\forall d=1,2$
and $L_{\textrm{w-\gls{GP}}}^{(d)}>L_{\textrm{\gls{GP}}}^{(d)},\forall d$.
This means $U_{\textrm{w-\gls{GP}}}^{(d)}-L_{\textrm{w-\gls{GP}}}^{(d)}\le U_{\textrm{\gls{GP}}}^{(d)}-L_{\textrm{\gls{GP}}}^{(d)},\forall d$
and thus $\var\left[p_{\textrm{w-\gls{GP}}}\right]\le\var\left[p_{\textrm{\gls{GP}}}\right]$. Using similar derivations, we get ${\displaystyle \var\left[p_{\gls{GP}}\right]}\ge{\displaystyle \var\left[p_{\textrm{w-}\gls{SRGP}}\right]}$, thus conclude the proof.
\end{proof}

This can also be seen in Fig. \ref{fig:Illustration_2d_prob_space}
that indeed $\pi(g_m)$ is bounded by the density of $p_{\textrm{\gls{GP}}}$. 
Using \textsc{srgp}, we have flexibly
transformed the surrogate toward the desirable area. Under
such transformation, the $p_{\textrm{w-\gls{SRGP}}}$
has a tighter upper bound than $p_{\textrm{w-\gls{GP}}}$. We next
show that the sample variance of the base distribution using weighting \gls{SRGP}
is smaller than using weighting \gls{GP}, ${\displaystyle \var\left[p_{\textrm{w-\gls{SRGP}}}\right]\le\var\left[p_{\textrm{w-\gls{GP}}}\right]}$.
Thus, the \gls{SRGP} is more sample-efficient than \gls{GP} in generating accepted
samples. This variance property can be intuitively seen from Fig.
\ref{fig:Illustration_2d_prob_space} (left) in the main paper. In addition, we have numerically demonstrated this comparison in Table \ref{tab:Acceptance-ratio}.
\begin{lem} (Lemma \ref{lem:var_SRGP_better_var_GP} in the main paper)
In the 2d space formed
by $\left[ f^{+},f^{-} \right]$ and the \gls{GP} posterior samples $\{g^+_m, g^-_m \}_{m=1}^M$,
the sample variance of weighting \gls{SRGP} is smaller than weighting \gls{GP}: $\var\left[p_{\textrm{w-\gls{SRGP}}}\right]\le{\displaystyle \var\left[p_{\textrm{w-\gls{GP}}}\right]}$.
\end{lem}
\begin{proof}
To prove $\var\left[p_{\textrm{w-\gls{SRGP}}}\right]\le{\displaystyle \var\left[p_{\textrm{w-\gls{GP}}}\right]}$, it is equivalent to show that the inequality is true in all dimensions of the distributions considered. Let us denote the first dimension $(1)$ is formed by $f^{+}-\max g_m(\cdot)$ and the second dimension $(2)$ is by $f^{-}-\min g_m(\cdot)$, see Fig. \ref{fig:Illustration_2d_prob_space}. In the first dimension, we have the upper bound $U_{\textrm{w-\gls{SRGP}}}^{(1)}=f^{+}+2\eta_+\le U_{\textrm{w-\gls{GP}}}^{(1)}$,
the equality happens when $\eta_+\rightarrow\infty$ and the lower
bound $L_{\textrm{w-\gls{GP}}}^{(1)}=L_{\textrm{w-\gls{SRGP}}}^{(1)}$.
In the second dimension, we have both the upper and lower bounds unchanged $U_{\textrm{w-\gls{GP}}}^{(2)}=U_{\textrm{w-\gls{SRGP}}}^{(2)}$
and $L_{\textrm{w-\gls{GP}}}^{(2)}=L_{\textrm{w-\gls{SRGP}}}^{(2)}$.

Using Lemma \ref{app_lem:(Popoviciu's-inequality-)}, we conclude
that $\var\left[p_{\textrm{w-\gls{SRGP}}}^{(1)}\right]\le{\displaystyle \var\left[p_{\textrm{w-\gls{GP}}}^{(1)}\right]}$
and $\var\left[p_{\textrm{w-\gls{SRGP}}}^{(2)}\right]\le{\displaystyle \var\left[p_{\textrm{w-\gls{GP}}}^{(2)}\right]}$.
In words, the base distribution by \gls{SRGP} (Section \ref{subsec:Square-root-Transformation-and})
$p_{\textrm{w-\gls{SRGP}}}$ has a tighter sample variance than $p_{\textrm{w-\gls{GP}}}$
\citep{wilson2020efficiently}.
\end{proof}



\subsection{Derivations of the posterior approximation in Eq. (\ref{eq:tgp_mu}) and Eq. (\ref{eq:tgp_sigma})} \label{app_sec:posterior_approximation}

We follow existing works, such as \cite{gunter2014sampling,ru2018fast,nguyen2020knowing} to approximate $p \bigl( f(\bx_*) \mid \bX,\by, \eta_+, f^+ \bigr)$ with a Gaussian distribution by using a Taylor approximation.

Given $f(\bx)=f^+ + 2 \eta_+ - \frac{1}{2} h^2(\bx)$ from Eq. (\ref{eq:transformed_GP}), we perform a local linearisation of the parabolic transformation $q \bigl(h(\bx) \bigr) = f^+ + 2\eta_+ - \frac{1}{2} h^2(\bx)$ around $h_0$ to obtain $f(h) \approx q(h_0) + q'(h_0)(h-h_0) + ...$ where the gradient $q'(h_0)=-h$. 
Then we set $h_0 = m_h$ where $m_h$ is defined in Eq. (\ref{eq:transformed_GP}) to obtain an expression for
$f$ which is linear in $h$
\begin{align}
    f(\bx) &\approx q \bigl(h_0(\bx) \bigr) + q'\bigl(h_0(\bx) \bigr) \bigl(h(\bx)-h_0(\bx) \bigr) \\
       & = \bigl[f^+ + 2\eta_+ - \frac{1}{2} m_h^2(\bx) \bigr] - m_h(\bx) \bigl[h(\bx) - m_h(\bx) \bigr] \\
       & = f^+ + 2\eta_+ + \frac{1}{2} m_h^2(\bx) - m_h(\bx) h(\bx)
\end{align}

We make use an important property that the affine transformation of a Gaussian process remains Gaussian. We state the property as follows: let $y = A \bx + t$ , $\bx \sim \mathcal{N}(\mu_\bx, \sigma_\bx)$, then $y= \mathcal{N}(A \mu_\bx + t, A \sigma_\bx A)$.

Let apply this property into our setting. Given $ h \sim \mathcal{N}( \mu_h, \sigma^2_h)$ and $f \sim \mathcal{N}(\mu_f, \sigma_f) $, we obtain
\begin{align}
    \mu_f &= f^+ + 2\eta_+ + \frac{1}{2} \mu_h^2 - \mu_h^2 = f^+ + 2\eta_+ - \frac{1}{2} \mu_h^2  \\
    \sigma^2_f &= \mu_h \sigma^2_h \mu_h
\end{align}
as shown in Eq. (\ref{eq:tgp_mu}) and Eq. (\ref{eq:tgp_sigma}) in the main paper.

\section{Ablation Studies}

\subsection{Ablation studies with different numbers of \gls{GP} samples $M$}\label{app_subsec:wrtM}

We study the sensitivity with varying the number of \gls{GP} samples $M\in[20,50,200,300,500]$
in Fig. \ref{fig:GP_sampling_ablation_wrt_M}. We show that if the
number of $M$ is set too small, such as $M\le50$. The performance
can be negatively affected. On the other hand, if we set $M$ large enough to
be $[200,500]$, they tend to perform similarly well. This is the recommended
range for $M$ in our paper. We also note that increasing $M$ will
require more computation. However, such computation can be well handled
by using parallel infrastructure since drawing \gls{GP} samples are independent
of each other.

\begin{figure}
\includegraphics[width=0.33\textwidth]{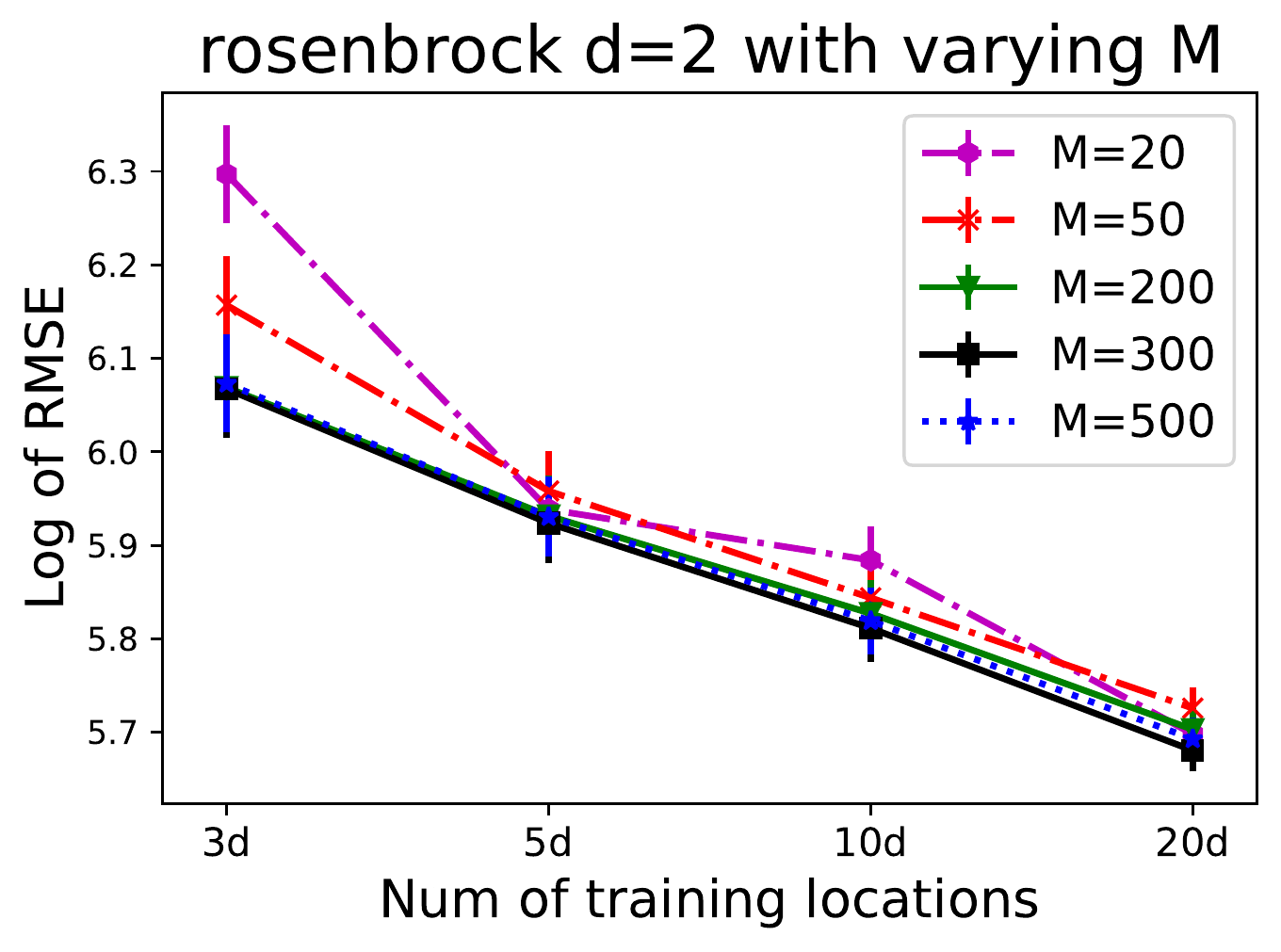}\includegraphics[width=0.33\textwidth]{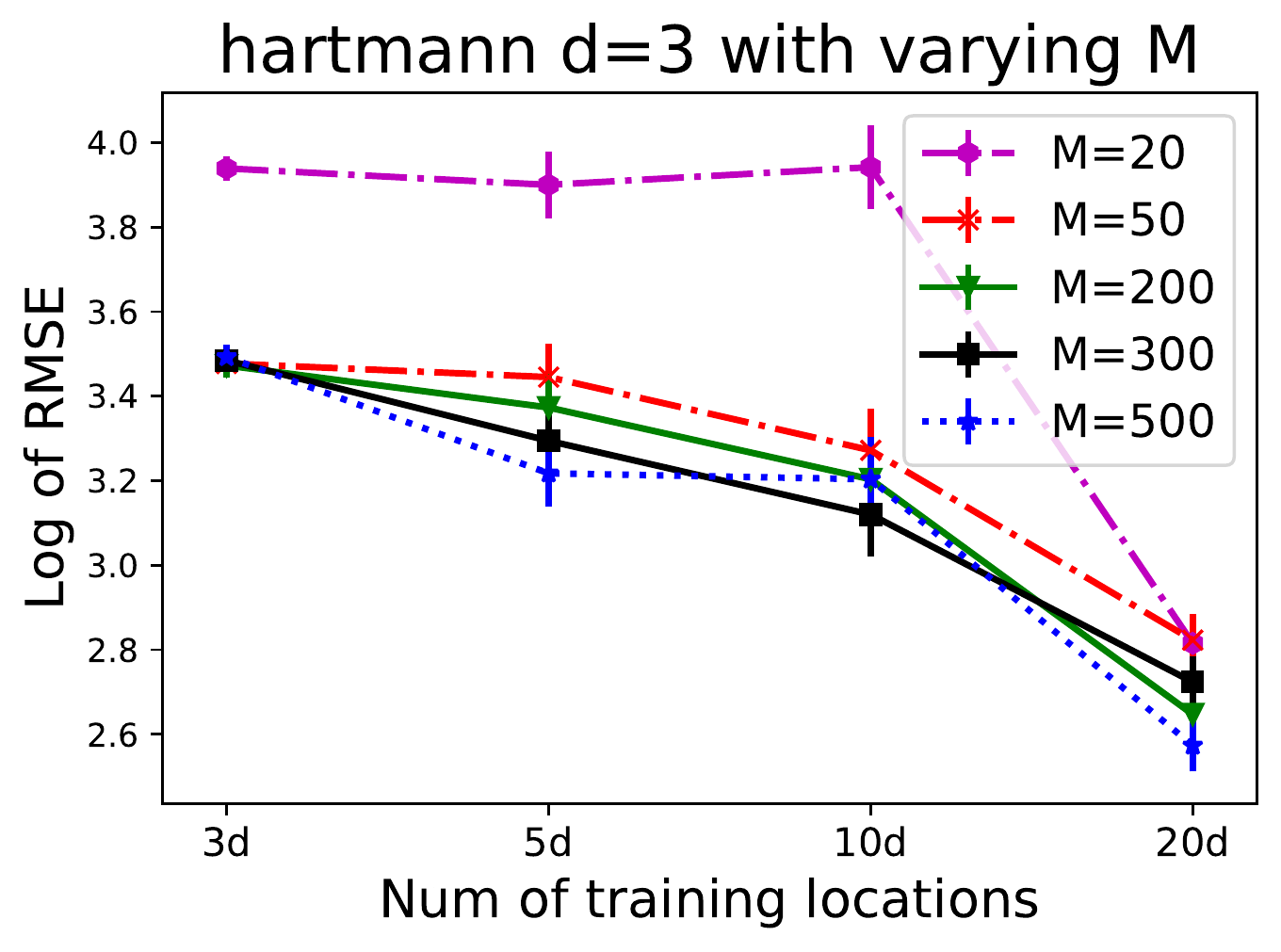}\includegraphics[width=0.33\textwidth]{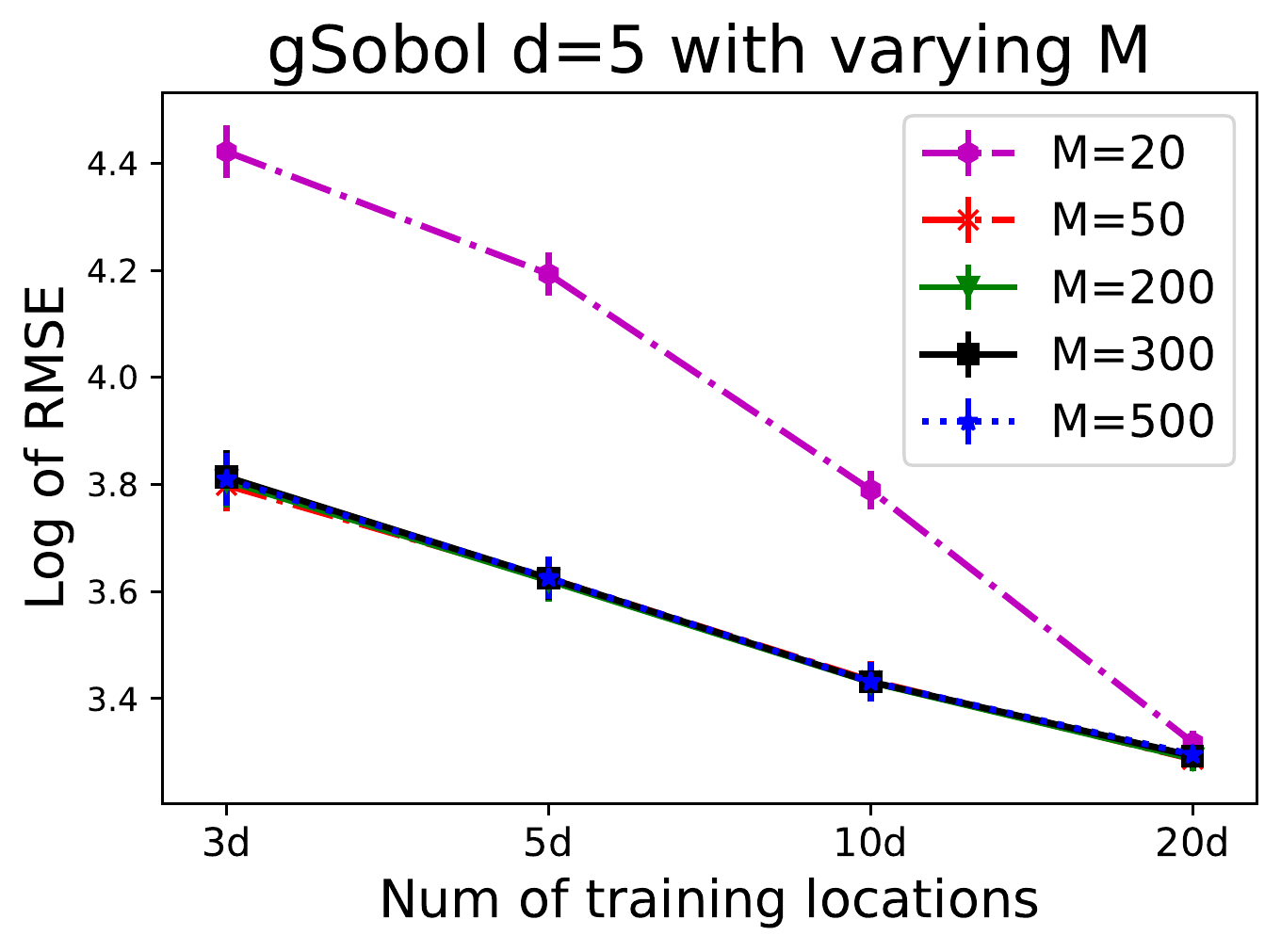}
\vspace{-7pt}
\caption{\gls{GP} posterior sampling with different number of \gls{GP} samples $M$. Our approach using $M\in[200,500]$ will result in the best performance while $M=20$ is the worst.\label{fig:GP_sampling_ablation_wrt_M}}
\vspace{-7pt}
\end{figure}

\begin{figure}
\begin{centering}
\includegraphics[width=0.6\textwidth]{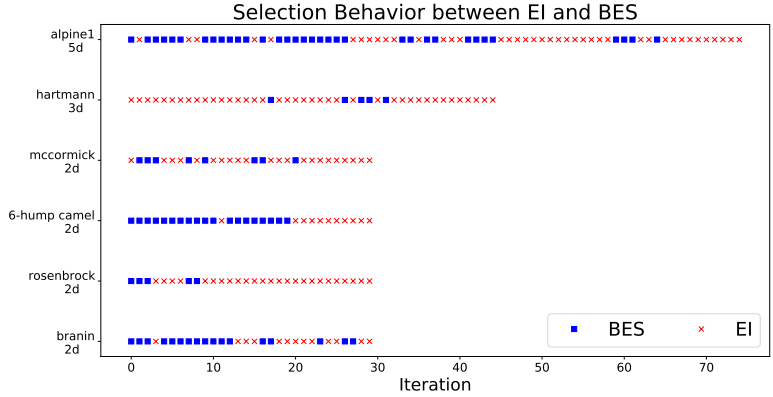}
\par\end{centering}
\caption{When none of the \gls{GP} sample is accepted (falling into the two standard deviation of $\pi(g_m)$), we propose to
use the \gls{EI} instead of \gls{BES}. This illustration indicates the switching
behavior between \gls{EI} vs \gls{BES}.\label{fig:Switching-between-EI_BES}}
\end{figure}

\subsection{Switching between \gls{EI} vs \gls{BES} \label{subsec:Switching-between-EI}}

We have discussed in Section \ref{sec:conclusion} that there
will be some iterations when we don't have any accepted sample for \gls{BES}. There
are several reasons, such as the observations are biased to reconstruct
the black-box function, the estimated \gls{GP} hyperparameters are not correct.
When the number of accepted samples is zero, we will perform Expected
improvement \citep{Jones_1998Efficient} instead of \gls{BES}. Our choice
is motivated by the strong and robust performance of \gls{EI}. We note that an alternative solution is to keep sampling the \gls{GP} until we get at least one accepted sample. 

We illustrate this behavior in Fig. \ref{fig:Switching-between-EI_BES}. We use the number of \gls{GP} samples $M=200$, $\eta^2_+=0.03d$ and $\eta^2_-=0.6d$
where $d$ is the input dimension. 
We use a blue square to indicate if \gls{BES} is used for this iteration
and red cross if \gls{EI} is taken.

\begin{figure*}
\begin{centering}
\includegraphics[width=0.48\textwidth]{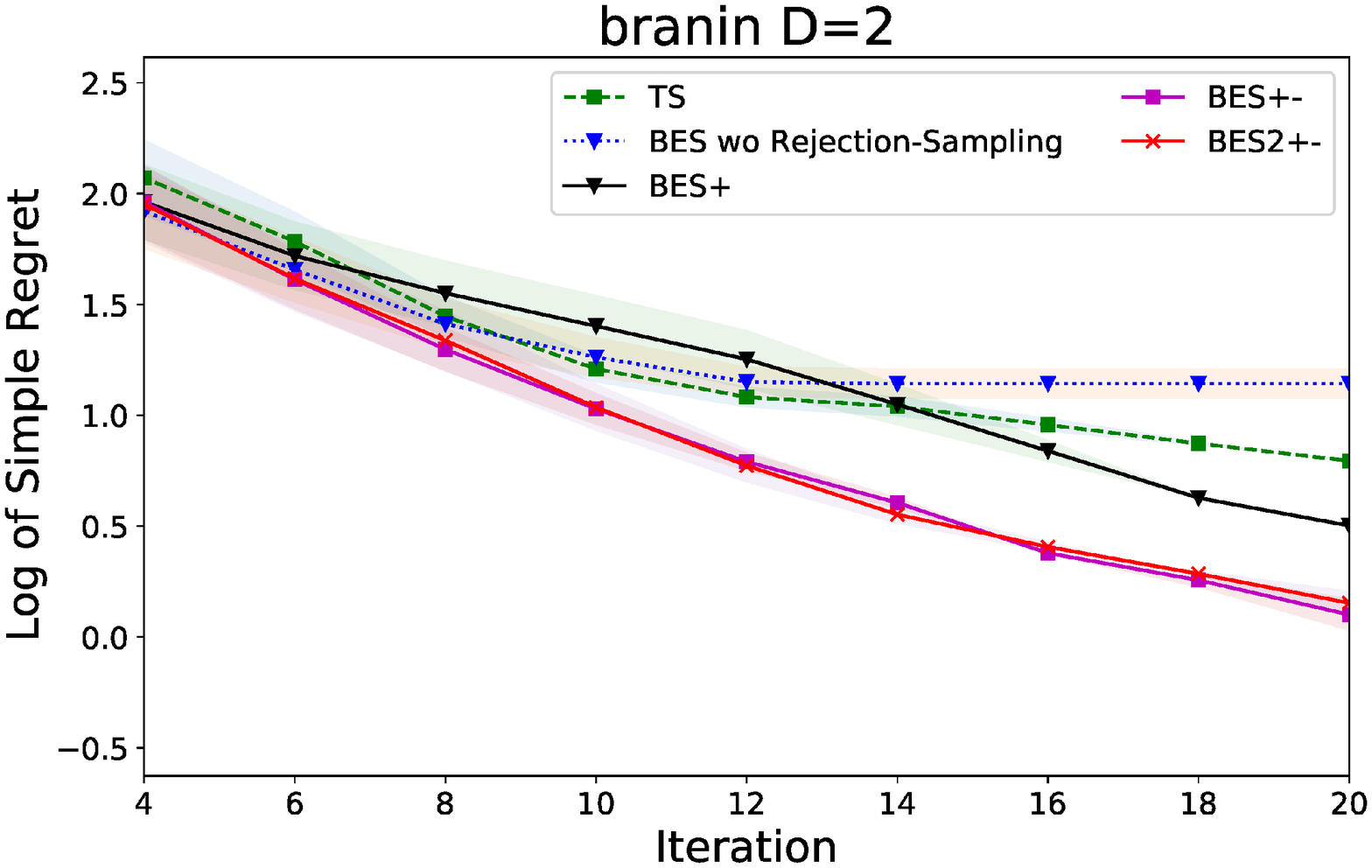}\includegraphics[width=0.48\textwidth]{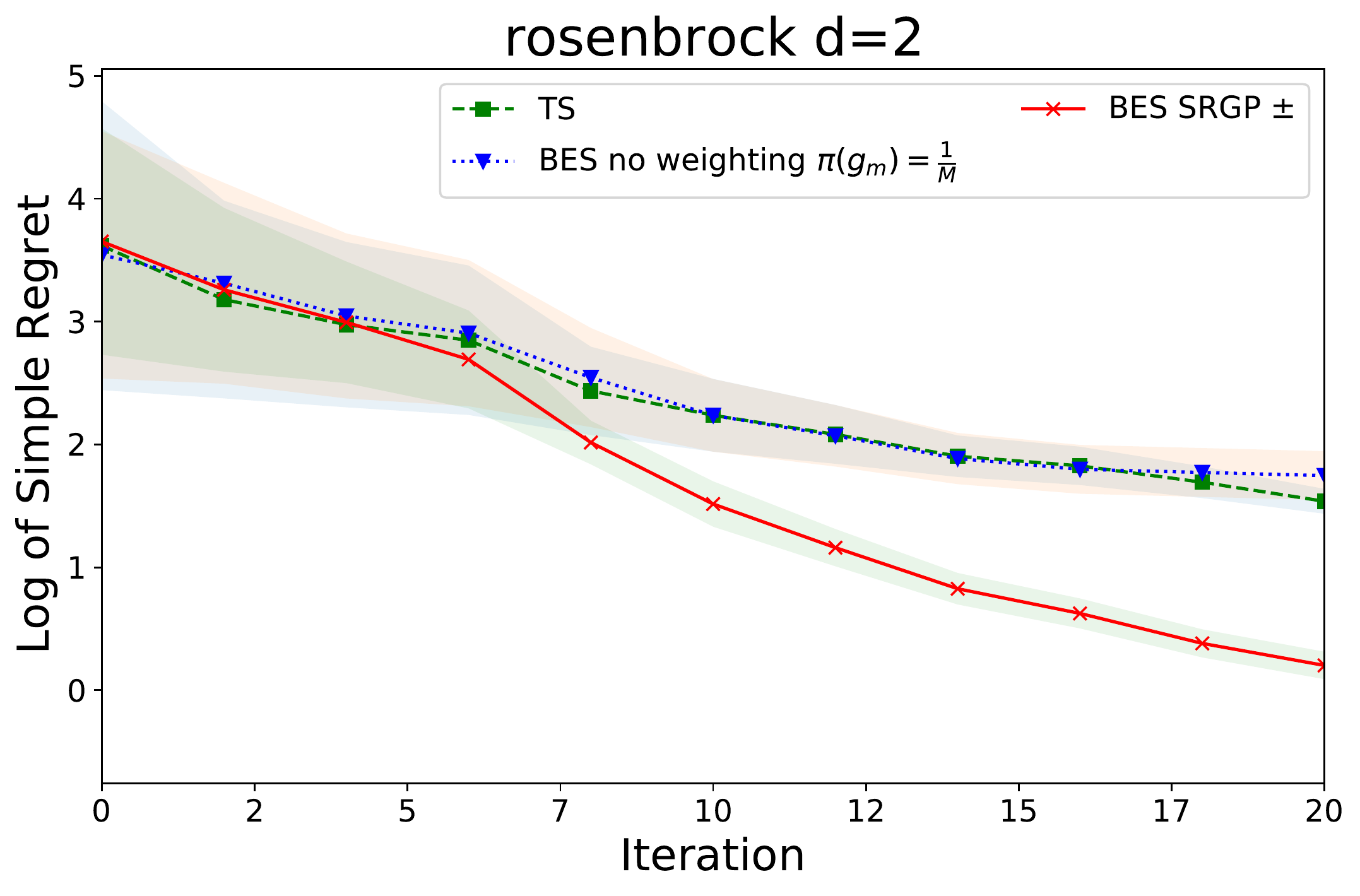}
\par\end{centering}
\begin{centering}
\includegraphics[width=0.48\textwidth]{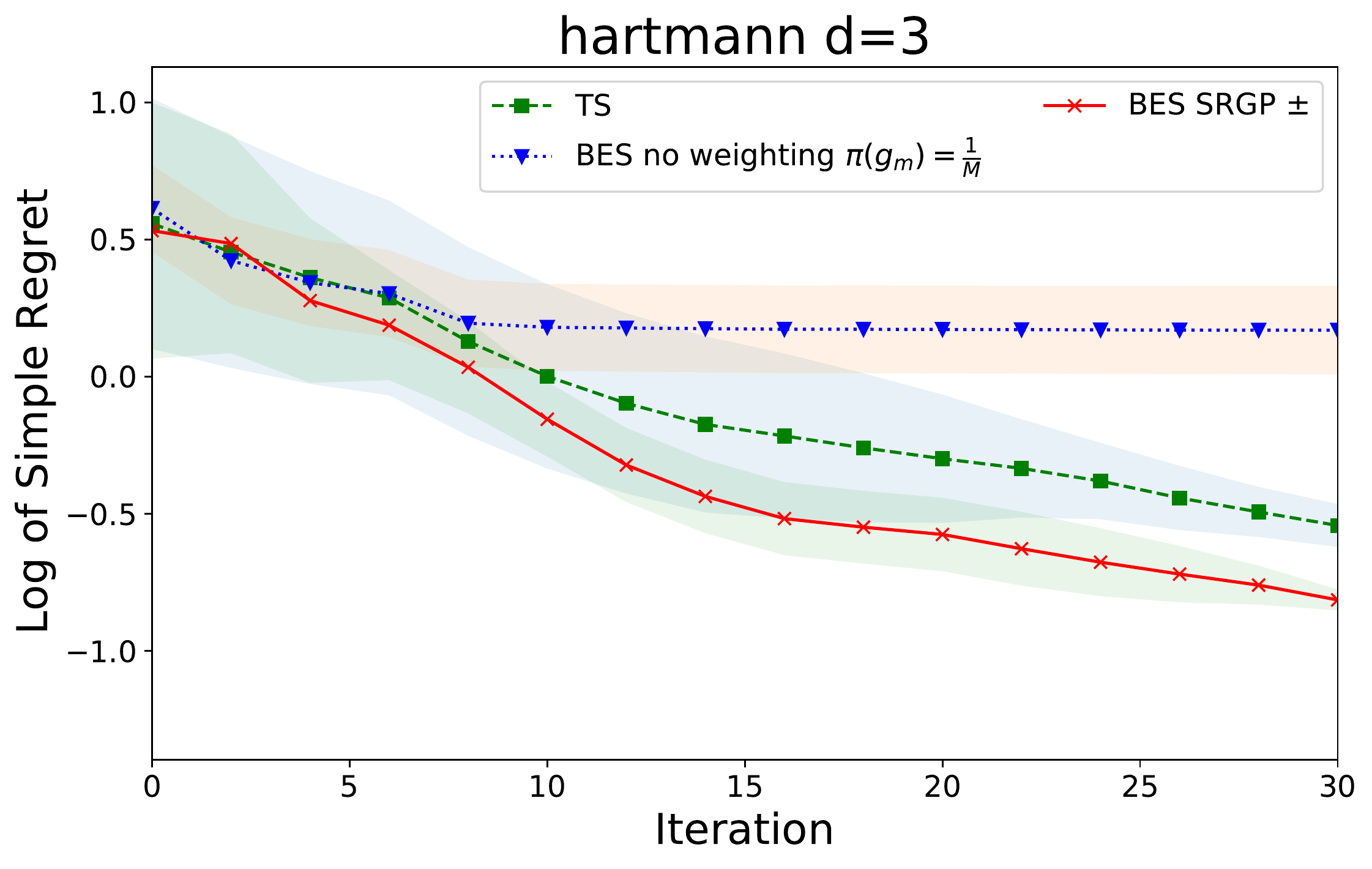}\includegraphics[width=0.48\textwidth]{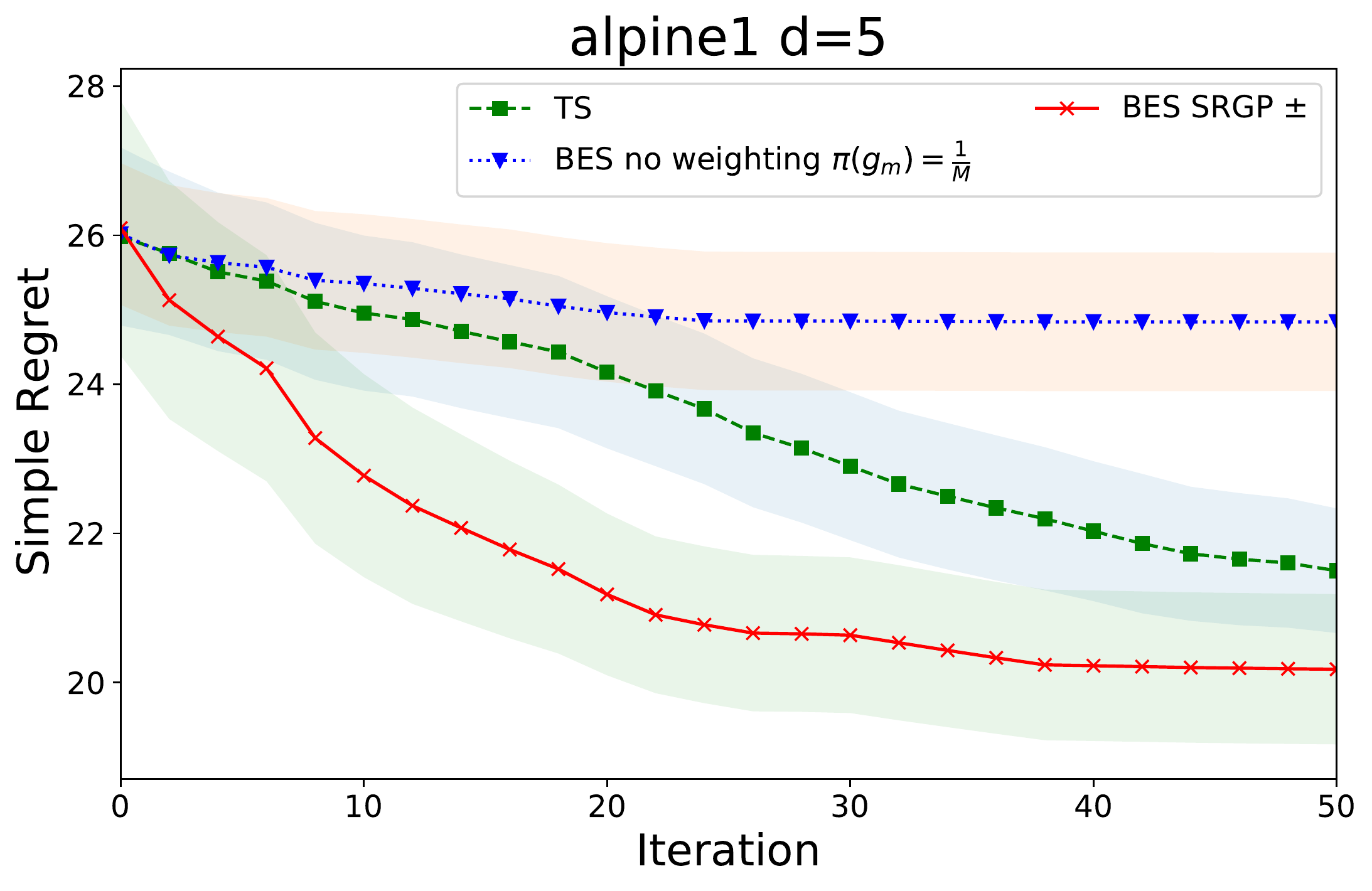}
\par\end{centering}
\vspace{-3pt}
\caption{Ablation studies of \gls{BES} without weighting or $\pi(g_m)=\frac{1}{M}$. The weighting step significantly
boosts the performance of \gls{BES} to achieve significantly better
than \gls{BES} without weighting (blue curve). \textit{Simple regret} is defined as the gap  between the optimum value and the best found value, i.e., $| \max_{\forall \bx \in \mathcal{X}} f(\bx) -  \max_{\forall \bx_i \in D_t} f(\bx_i)| $.\label{fig:Ablation-studies-of_wo_RS}}
\vspace{-7pt}
\end{figure*}

\subsection{BES without weighting \label{subsec:BES-without-rejection_ablation}}

We provide another set of experiments to compare the performance with a variant version of \gls{BES} without weighting and Thompson sampling (or \gls{BES} with $M=1$). The \gls{BES} setting without weighting is when we ignore the $f^+$, $f^-$ and set the weighting probability $\pi(g_m)=\frac{1}{M}$ uniformly.

Fig. \ref{fig:Ablation-studies-of_wo_RS} shows that
\gls{BES} without weighting (blue curve) will perform poorly. Particularly,
this version will accept all \gls{GP} posterior samples and gain information
about all of them, including the `bad' samples. Thompson sampling (\gls{TS})
is a special version of \gls{BES} in which we only draw a single \gls{GP} posterior
sample $g(\cdot)$. Then, we select a next point to evaluate as the mode
of this sample, i.e., $\bx_{t}=\arg\max_{\bx\in\mathcal{X}}g(\bx)$.
Because drawing a single \gls{GP} sample will have very high variance especially
in high-dimension, \gls{TS} will generally perform inferior in a sequential
\gls{BO} setting, as discussed in \cite{Hernandez_2014Predictive,Nguyen_ICDM2019}.

Exploiting the knowledge about both $f^{+},f^{-}$ will result in
better performance than exploiting a single value of $f^{+}$. Especially,
our \gls{BES} version using square-root transformation \gls{GP} will achieve the
best performance. This is because it can take into account the knowledge
of $f^{+}$ to tailor the \gls{GP} sample toward the desirable shape. This
will lead to better theoretical property as presented in Lemma \ref{lem:var_target_better_var_proposal}
and Lemma \ref{lem:var_SRGP_better_var_GP}.

The result in this experiment again highlights our choice of \gls{BES} in Section \ref{sec:BES} that gains information only about good \gls{GP} posterior samples.

\begin{figure*}
\begin{centering}
\includegraphics[width=0.98\textwidth]{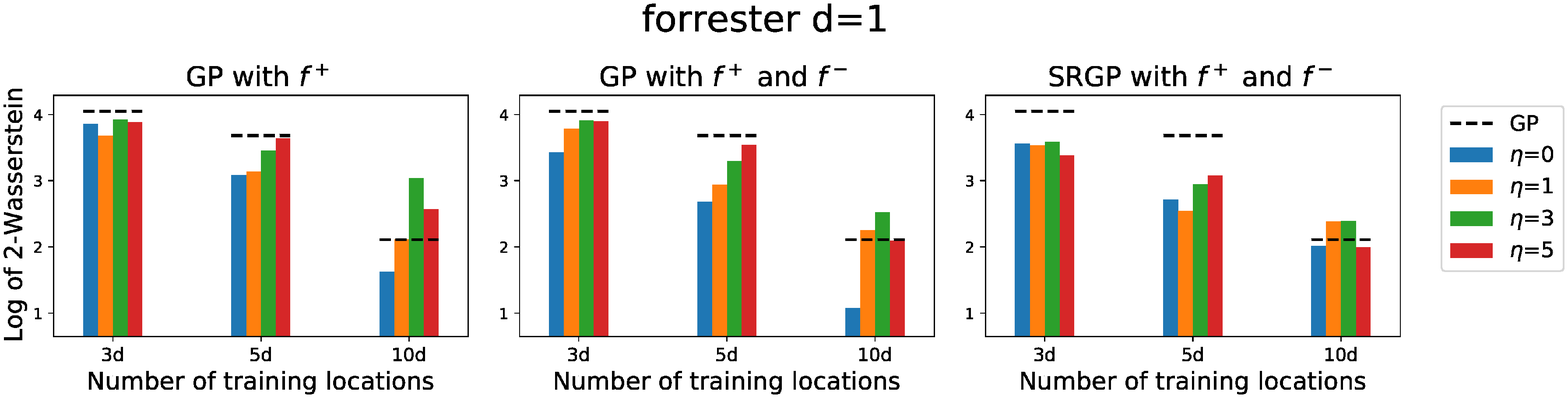}
\par\end{centering}
\begin{centering}
\includegraphics[width=0.98\textwidth]{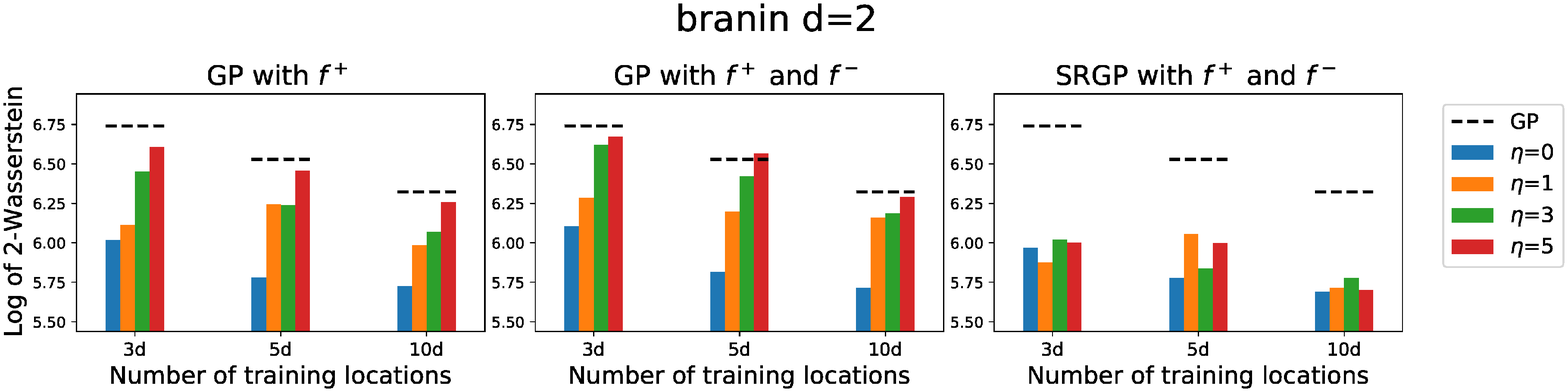}
\par\end{centering}
\begin{centering}
\includegraphics[width=0.98\textwidth]{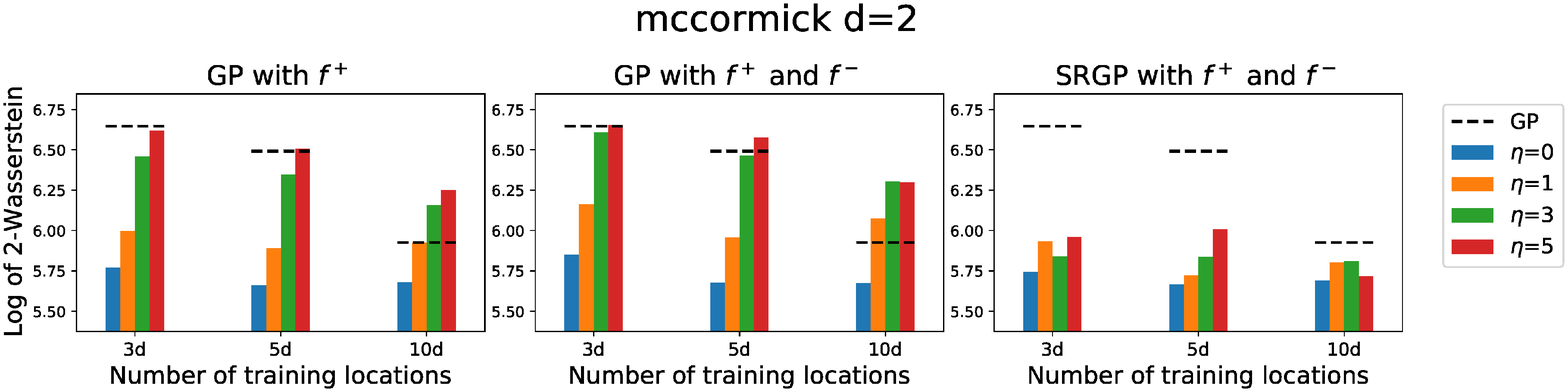}
\par\end{centering}
\begin{centering}
\includegraphics[width=0.98\textwidth]{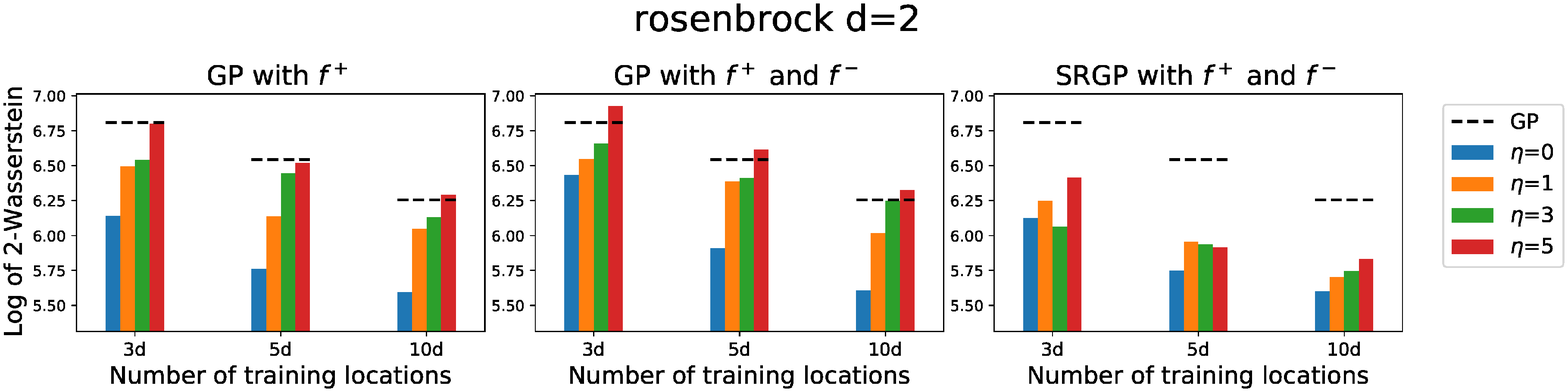}
\par\end{centering}
\begin{centering}
\includegraphics[width=0.98\textwidth]{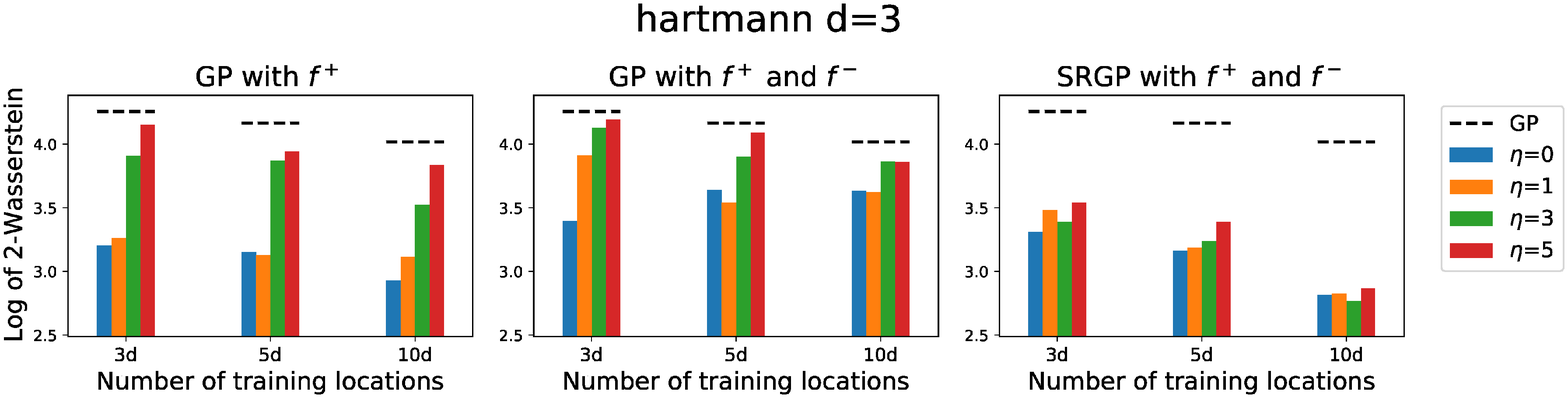}
\par\end{centering}
\caption{We study the sensitivity of the model by misspecifying the $f^{+}$
and $f^{-}$ with different levels $\eta^2:=\eta^2_+=\eta^2_-\in[0,1,3,5]$.
Increasing the misspecified levels will degrade the performance. It is important to note that our w-\gls{GP} with misspecification still performs better than the vanilla \gls{GP} sampling \citep{wilson2020efficiently} which does not use the bounds,
especially when the number of training locations is small, $3d$ or
$5d$. Asymptotically, when $\eta\rightarrow\infty$, our w-\gls{GP} will
accept all samples and become the case of \cite{wilson2020efficiently}.
\label{fig:ablation_gp_sampling_noise}}
\end{figure*}

\subsection{Sensitivity analysis of misspecifying $f^{+}$ and $f^{-}$ for
\gls{GP} posterior sampling}

We study the sensitivity  to
the \gls{GP} posterior sampling performance by misspecifying the upper and lower bounds. In the benchmark function considered, we know the true $\max f(\bx)$ and $\min f(\bx)$ in advance. We make use of these optimum values and vary the misspecified levels
$\eta^2_+=\eta^2_-\in\left[0,1,3,5\right]$ for which $f^{+}=\max f(\bx)+\mathcal{N}(0,\eta^2_+)$
and $f^{-}=\min f(\bx)+\mathcal{N}(0,\eta^2_-)$. This level is considered
in a standardized output space $y\sim\mathcal{N}(0,1)$. We present
the result in Fig. \ref{fig:ablation_gp_sampling_noise}. It is expected
that increasing the misspecified levels will lead to degrading the performance.
However, our model still performs  better than the vanilla \gls{GP} sampling
\citep{wilson2020efficiently} in most cases. Especially when we
have limited observations of less than $5d$ training locations, our
model outperforms \gls{GP} sampling when $\eta^2\le3$.

Relaxing this misspecify level $\eta\rightarrow\infty$ will let
the performance of w-\gls{GP} at least reduce to the case of \cite{wilson2020efficiently}
while the performance of w-\gls{SRGP} will be more sensitive as the transformation
depends on $\eta_+$.


\section{Experimental evaluation}

\subsection{Experimental setting \label{app_subsec:Experimental-setting}}

The benchmarks function and their ranges are available online.\footnote{https://www.sfu.ca/\textasciitilde ssurjano/optimization.html}
We summarize the hyperparameters and their min-max ranges used in
the machine learning tuning experiments in Table \ref{tab:Hyperparameters-setting}.

\begin{itemize}
\item Support vector regression \citep{Smola_1997Support} on Abalone dataset.
The Abalone dataset is available publicly.\footnote{https://archive.ics.uci.edu/ml/datasets/abalone}
We specify the upper bound as $f^{+}=-1.92$ and do not set the lower bound. 
\item \textsc{xgboost} classification \citep{chen2016xgboost}: we use Pima Indians
Diabetes database.\footnote{https://www.kaggle.com/uciml/pima-indians-diabetes-database}
We set the bounds $f^{+}=0.75$, $f^{-}=0.65$ for the result presented in Fig. \ref{fig:BO_result}. Then, we vary the upper bound $f^+ \in [0.7, 0.73,0.75,0.77,0.8,0.9]$ to study the sensitivity with different choices of $f^+$ presented in Fig. \ref{fig:BO_misspecified}.
\item \textsc{cnn} classification \citep{lecun1995convolutional} on a subset (10\%)
of \textsc{cifar10}. We use the bounds of $f^{+}=61.5$ and  $f^{-}=0$. 
\end{itemize}

\begin{table}
\centering{}\caption{The ranges for each hyperparameter used for tuning machine learning
algorithms. \label{tab:Hyperparameters-setting}}
\begin{tabular}{ccc}
\toprule 
\multicolumn{3}{c}{\textsc{xgboost}}\tabularnewline
\midrule 
Variables & Min & Max\tabularnewline
\midrule
min\_child\_weight & $1$ & $20$\tabularnewline
colsample\_bytree & $0.1$ & $1$\tabularnewline
max\_depth & $5$ & $15$\tabularnewline
subsample & $0.5$ & $1$\tabularnewline
gamma & $0$ & $10$\tabularnewline
alpha & $0$ & $10$\tabularnewline
\end{tabular}\hspace{20pt}%
\begin{tabular}{ccc}
\toprule 
\multicolumn{3}{c}{ \textsc{cnn}}\tabularnewline
\midrule 
Variables & Min & Max\tabularnewline
\midrule
batch size & $16$ & $100$\tabularnewline
learning rate & $1e^{-6}$ & $1e^{-2}$\tabularnewline
max iteration & $5$ & $50$\tabularnewline
\end{tabular}\hspace{20pt}%
\begin{tabular}{ccc}
\toprule 
\multicolumn{3}{c}{Support vector regression}\tabularnewline
\midrule 
Variables & Min & Max\tabularnewline
\midrule
$C$ & $0.1$ & $1000$\tabularnewline
epsilon & $1e^{-6}$ & $1$\tabularnewline
gamma & $1e^{-6}$ & $5$\tabularnewline
\bottomrule
\end{tabular}
\end{table}

\subsection{Additional experiments \label{subsec:Additional-experiments}}

We present additional experiments for \gls{GP} posterior
sampling in Fig. \ref{fig_app:gp_sampling_comparison} and Bayesian
optimization in Fig. \ref{fig_app:BO_result}. The additional results
are consistent with the results presented in the main paper. \gls{BES}$\pm$
(exploiting both $f^{+}$ and $f^{-}$) achieves the best performance.

\begin{figure*}
\begin{centering}
\includegraphics[width=0.33\textwidth]{figs/Six-hump-camel_2_RMSE}\includegraphics[width=0.33\textwidth]{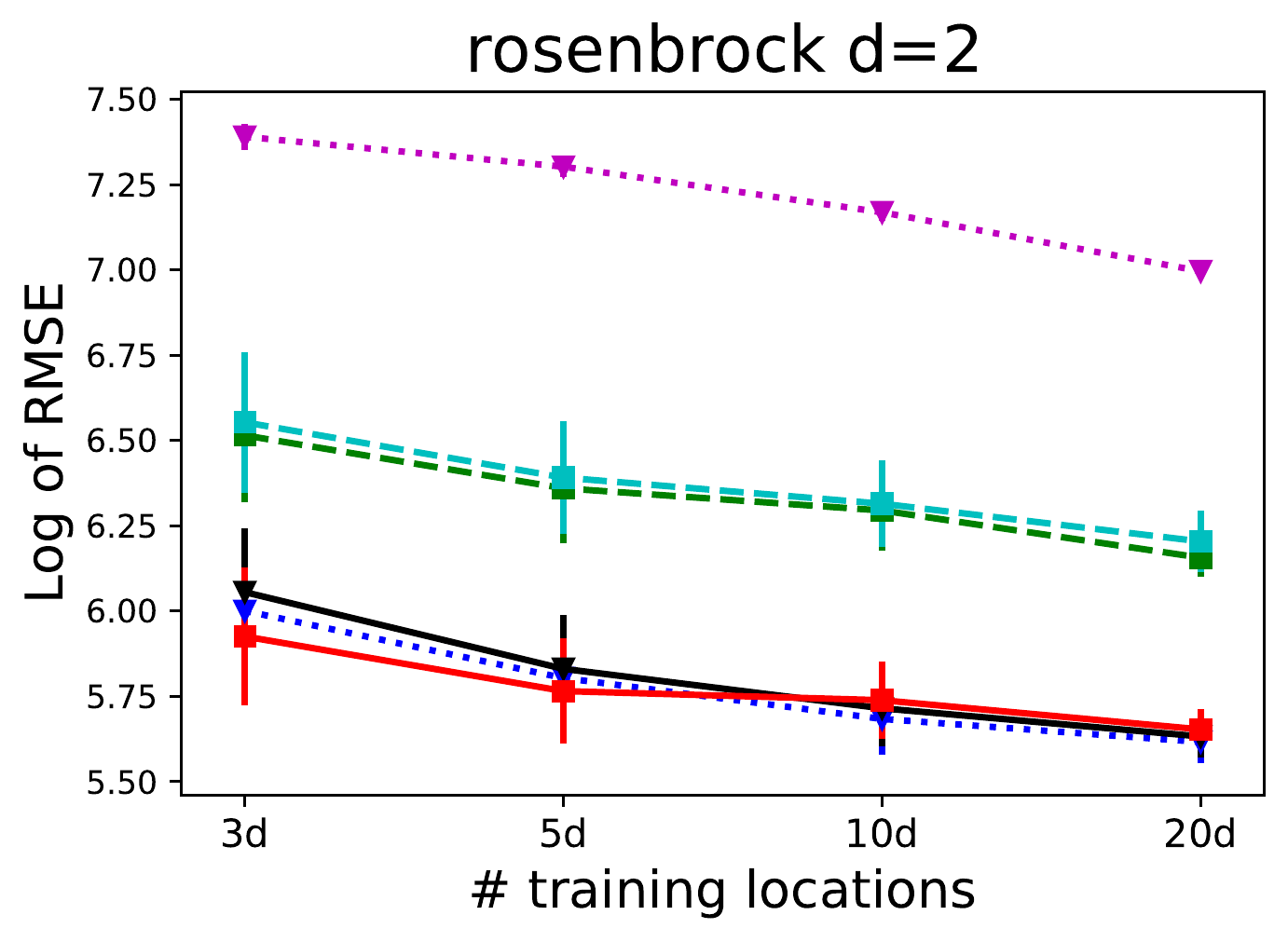}\includegraphics[width=0.33\textwidth]{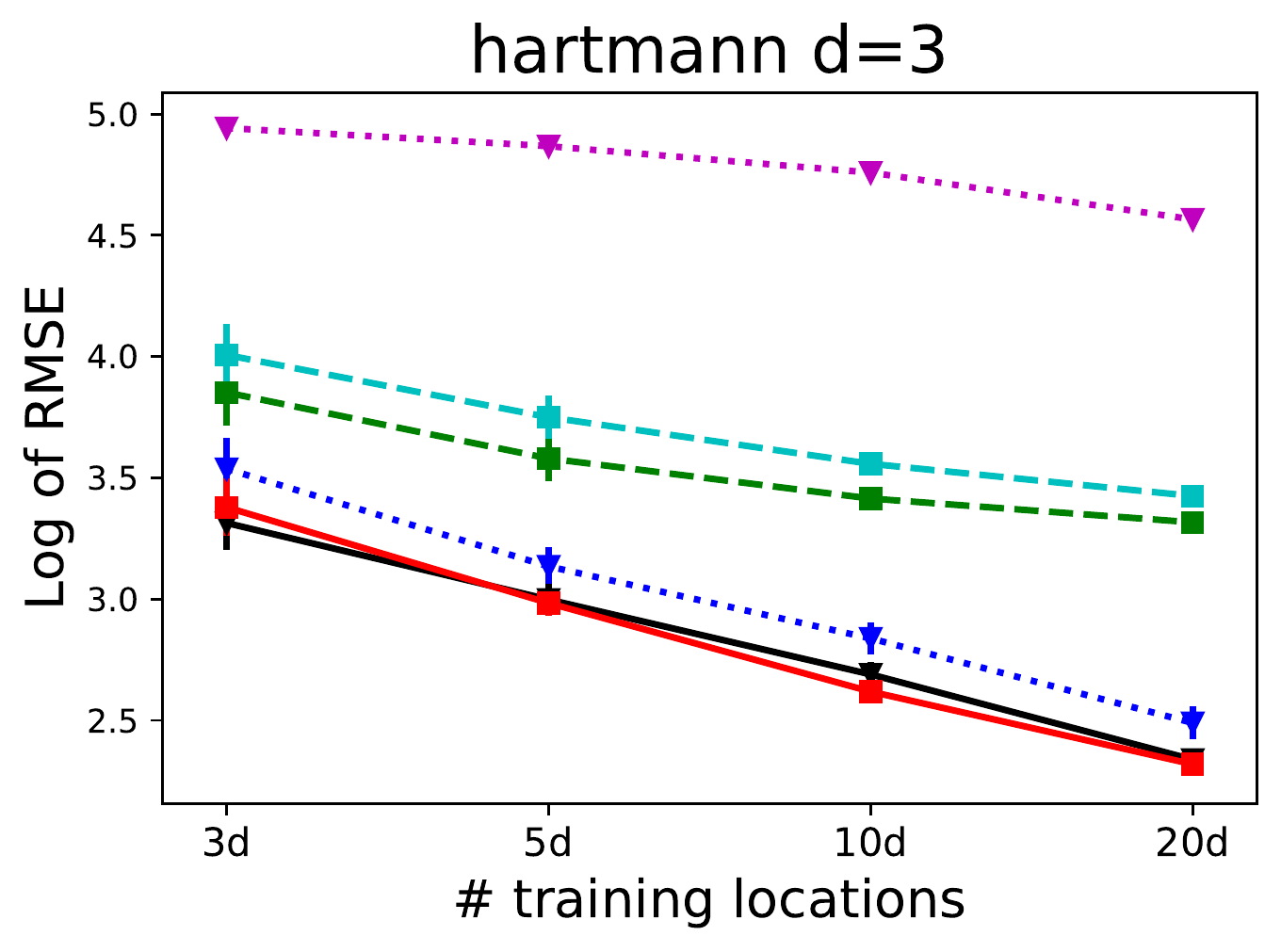}
\par\end{centering}
\begin{centering}
\includegraphics[width=0.33\textwidth]{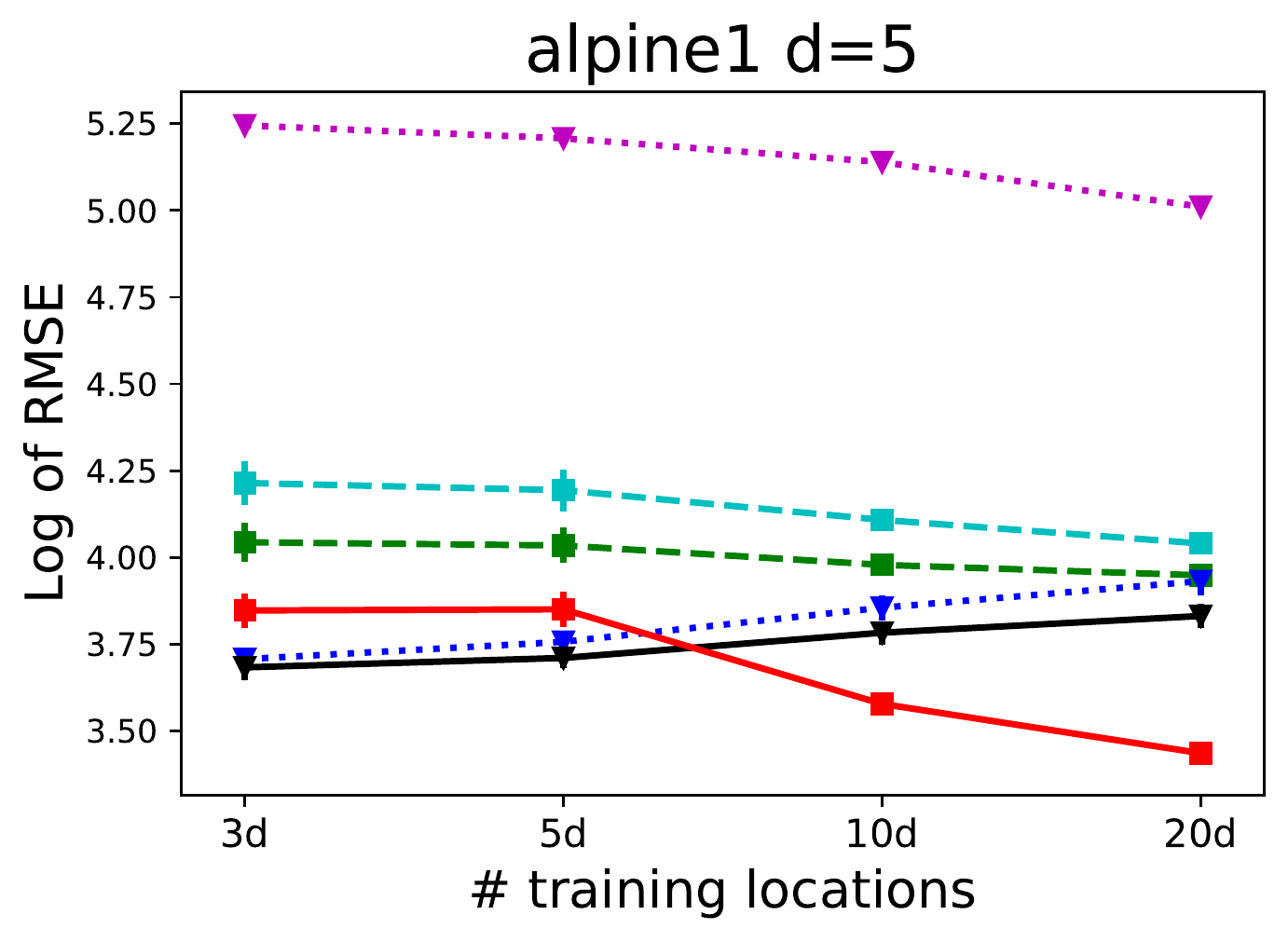}\includegraphics[width=0.33\textwidth]{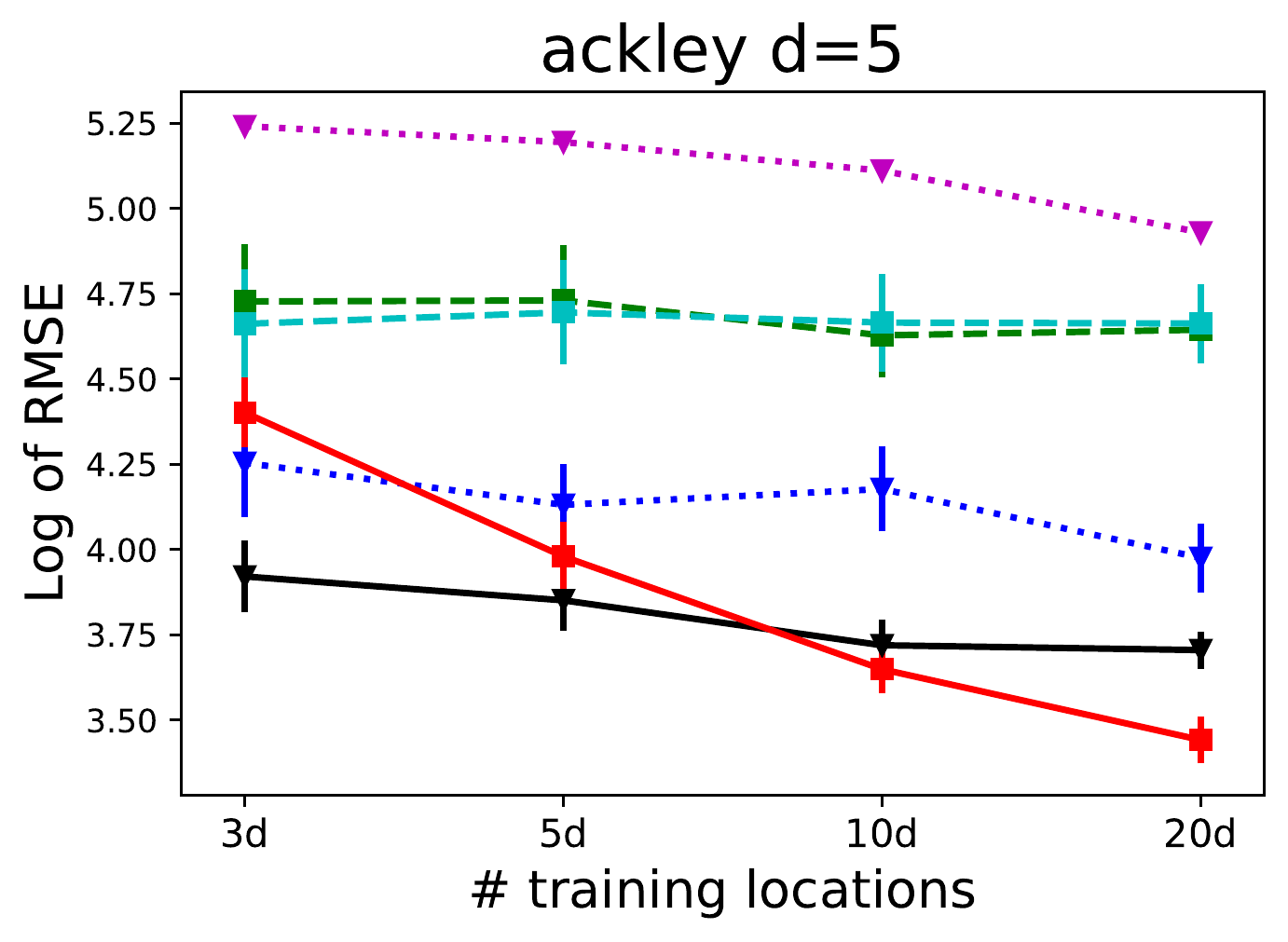}\includegraphics[width=0.33\textwidth]{figs/hartmann_6_RMSE.pdf}

\includegraphics[width=0.99\textwidth]{figs/gpsampling_legend.pdf}
\par\end{centering}
\caption{We measure the root mean squared error (\gls{RMSE}) loss  from
$M=200$ \gls{GP} samples $g(\cdot)$ against the true function $f(\cdot)$ with different
number of training points $[3d,5d,10d,20d]$ where $d$ is the number
of input dimensions. The $f^{+}$ and $f^{-}$ are known from the black-box
function. The overall trends indicate that (i) the loss is smaller
with increasing number of training points and (ii) exploiting the external
knowledge about $f^{+}$ and $f^{-}$ will lead to better estimation
than exploiting $f^{+}$ alone. When we get sufficient observations, e.g., the number of training locations is $20d$, the performance of w-\gls{GP} $f^{+},f^{-}$ resembles w-\gls{SRGP} $f^{+},f^{-}$. The error bar is estimated over
$30$ independent runs.\label{fig_app:gp_sampling_comparison}}
\end{figure*}

\begin{figure*}
\centering

\includegraphics[width=0.325\textwidth]{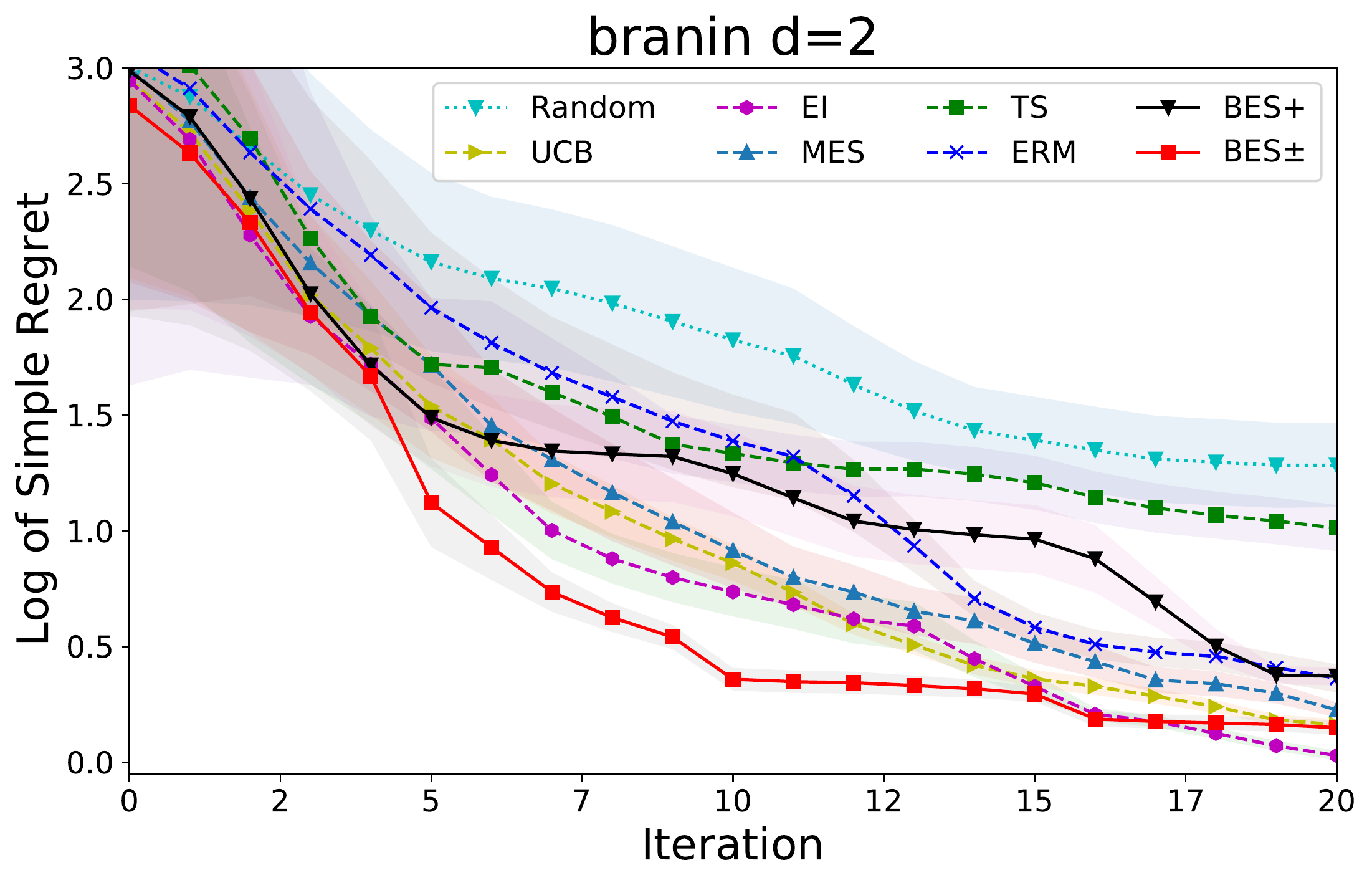}
\includegraphics[width=0.325\textwidth]{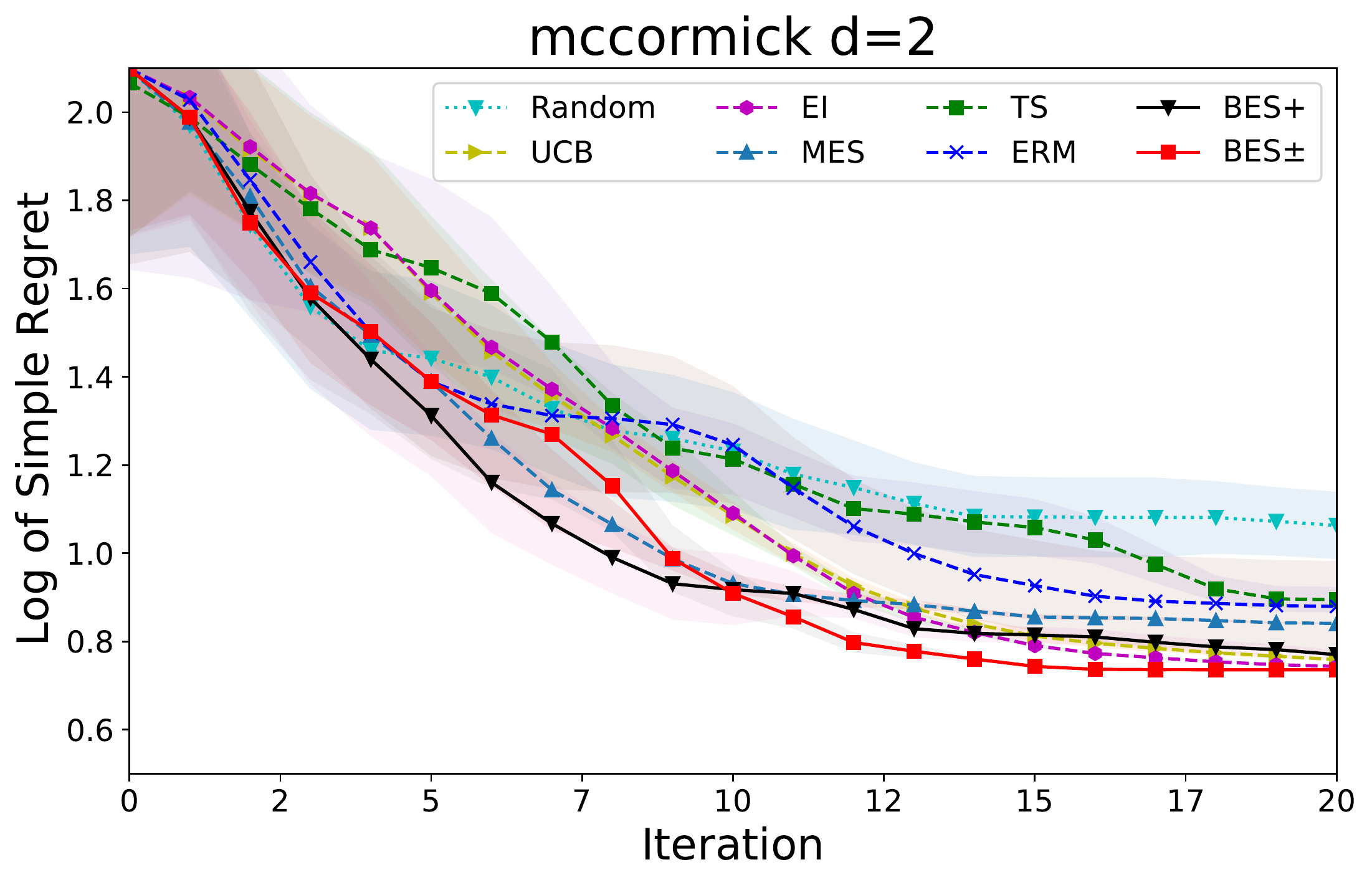}
\includegraphics[width=0.325\textwidth]{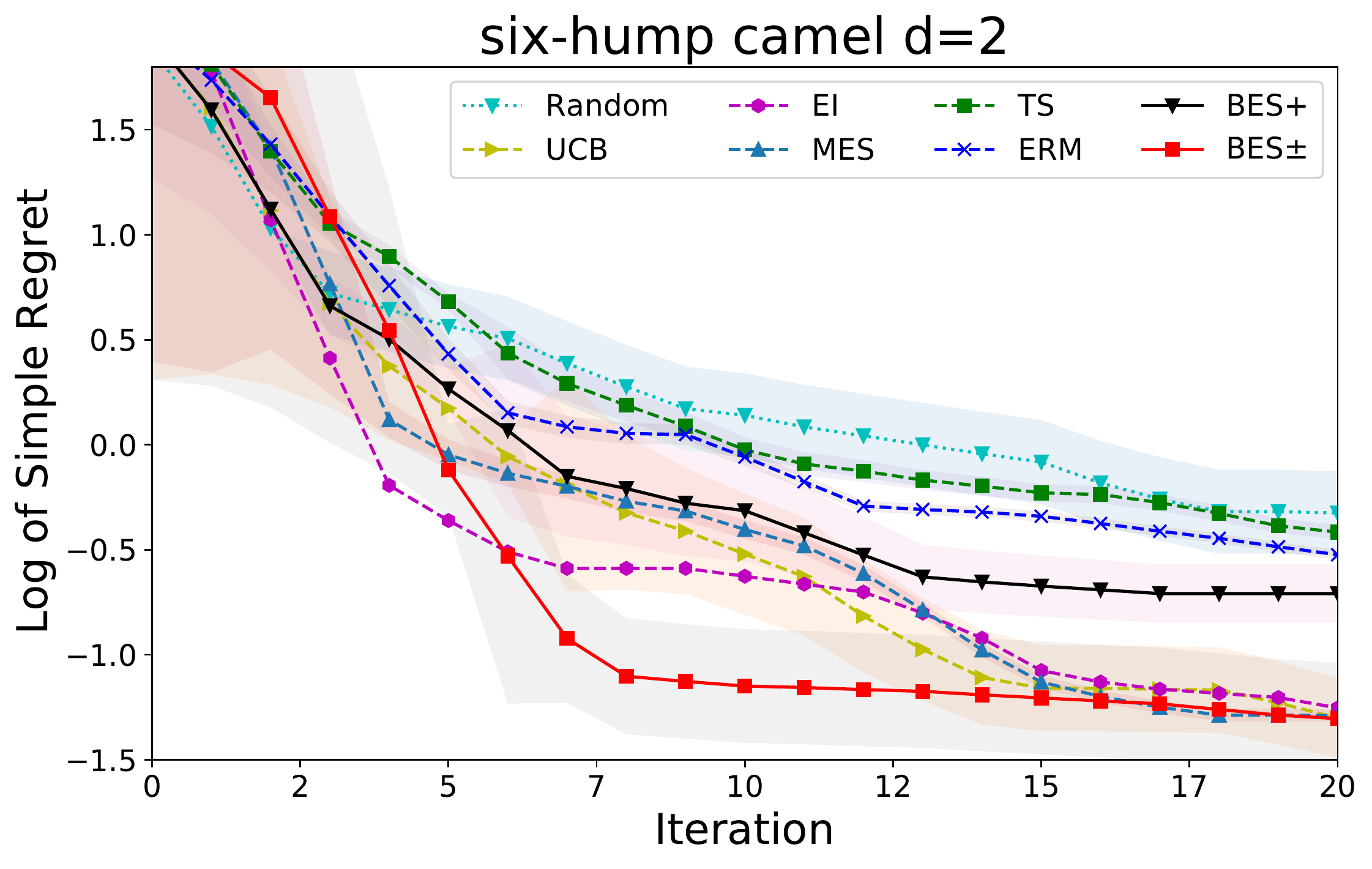}

\includegraphics[width=0.325\textwidth]{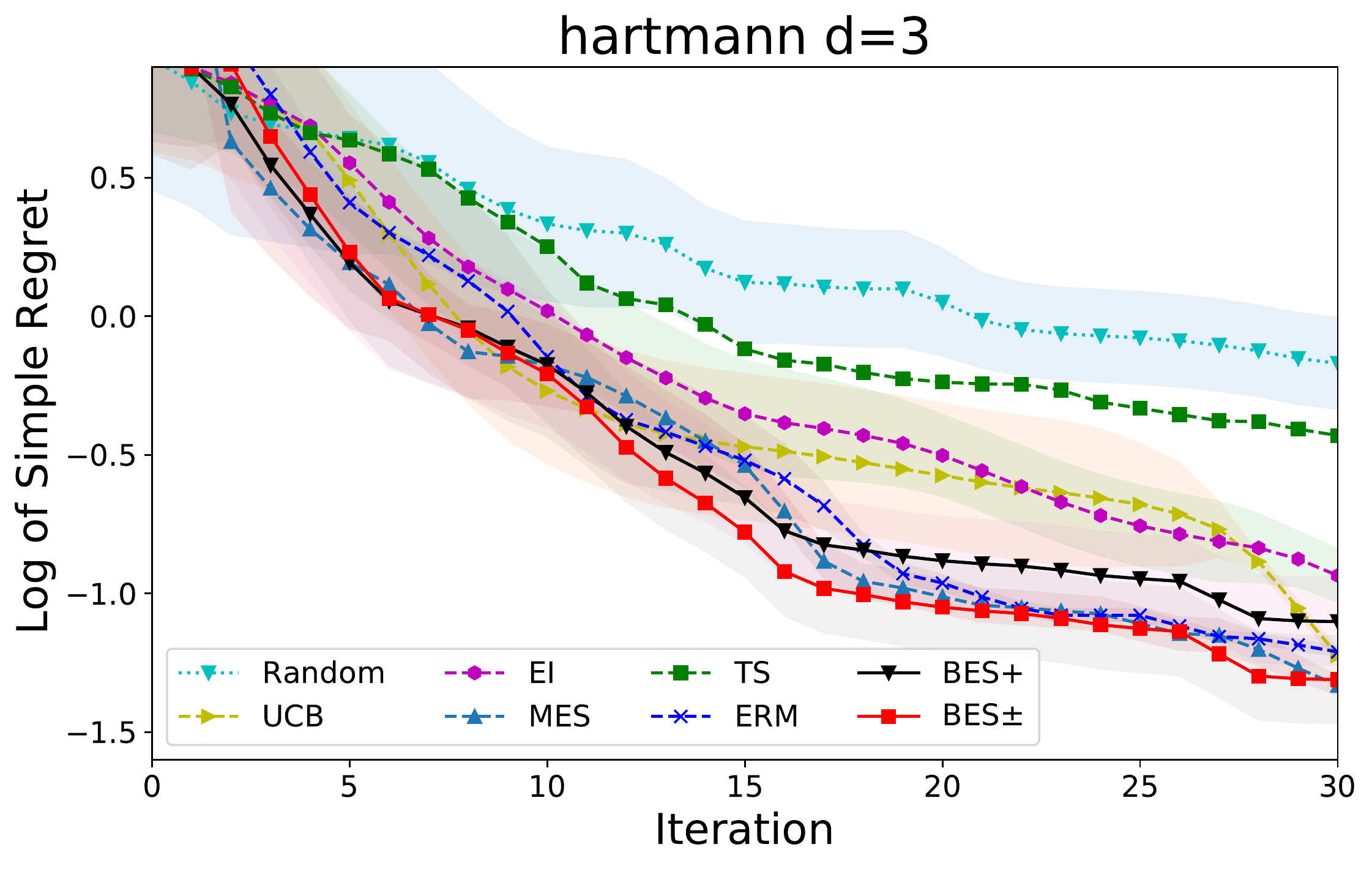}
\includegraphics[width=0.325\textwidth]{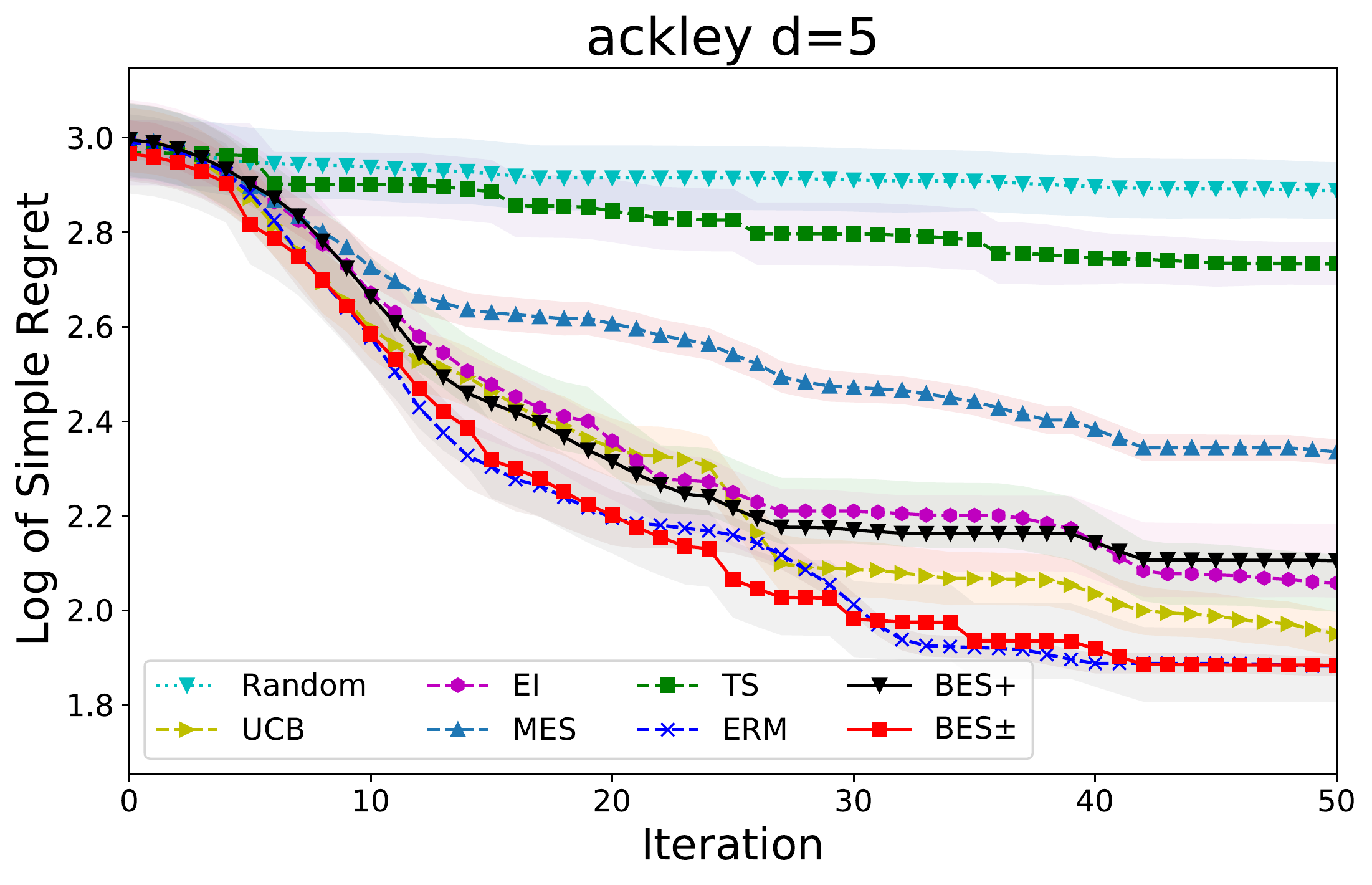}
\includegraphics[width=0.325\textwidth]{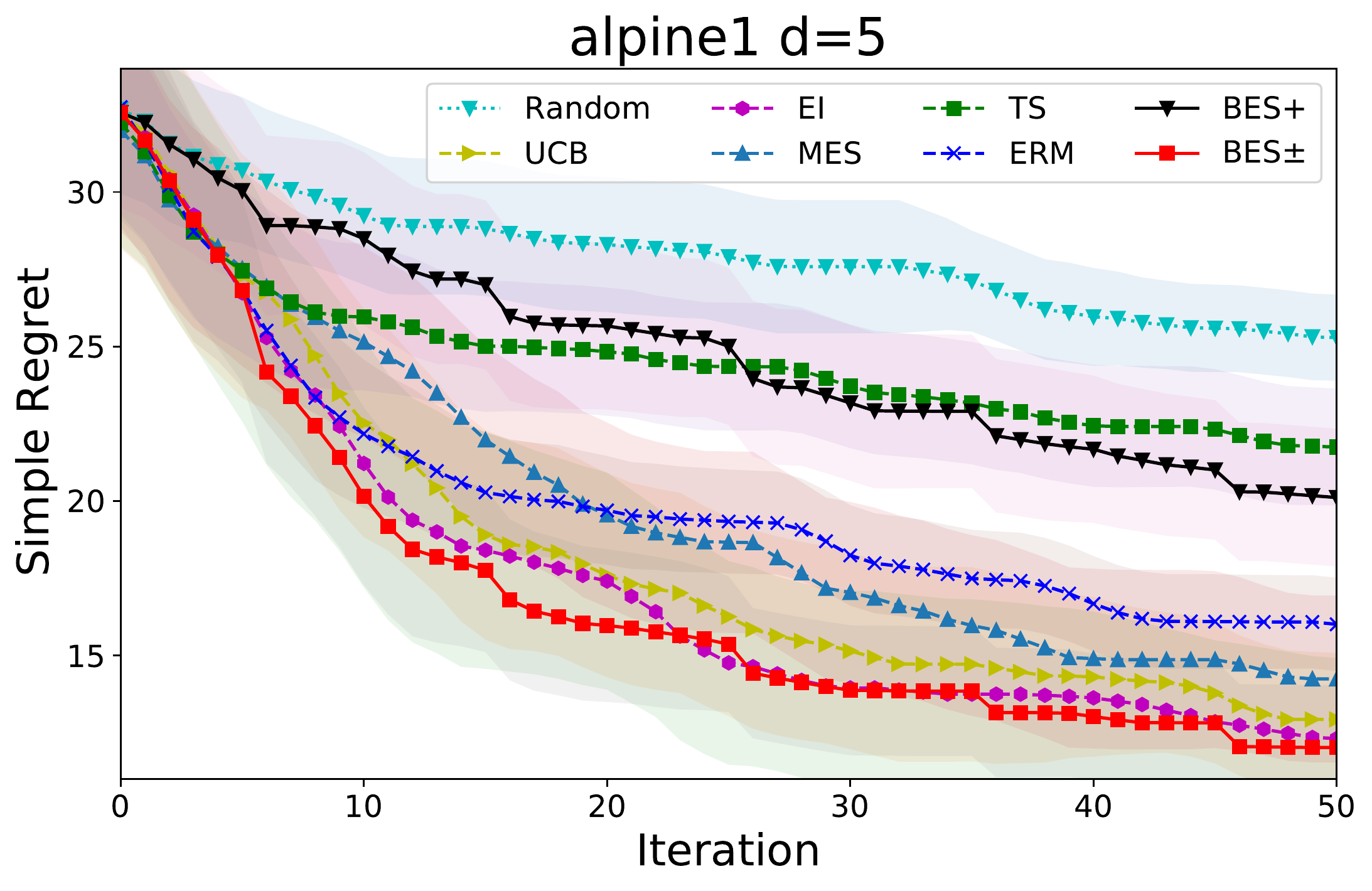}

\caption{We present additional results for Bayesian optimization using benchmark
functions. Our \gls{BES}$\pm$ (using both $f^{+},f^{-}$)
achieves better performance than \gls{BES}$+$ (using $f^{+}$ alone) and \gls{ERM}
(using $f^{+}$ alone) \citep{nguyen2020knowing}. \textit{Simple regret} is defined as $| \max_{\forall \bx \in \mathcal{X}} f(\bx) -  \max_{\forall \bx_i \in D_t} f(\bx_i)| $. \label{fig_app:BO_result}}
\end{figure*}

\end{document}